\documentclass[twoside,11pt]{article}
\usepackage{jmlr2e}

\usepackage{epsf}
\usepackage{epsfig}
\usepackage{amsmath}
\usepackage{subfigure}
\usepackage{amssymb}
\usepackage{algorithm}
\usepackage{multirow}
\usepackage{multicol}
\usepackage{graphicx}
\usepackage{algorithm}
\usepackage{algorithmic}
\usepackage{color}
\usepackage[dvipsnames]{xcolor}

\usepackage[pagebackref=true,breaklinks=true,letterpaper=true,colorlinks,bookmarks=yes]{hyperref}

\def\A{{\bf A}}
\def\a{{\bf a}}
\def\B{{\bf B}}

\def\C{{\bf C}}

\def\D{{\bf D}}
\def\d{{\bf d}}

\def\e{{\bf e}}

\def\G{{\bf G}}
\def\H{{\bf H}}
\def\I{{\bf I}}
\def\K{{\bf K}}
\def\k{{\bf k}}
\def\LL{{\bf L}}

\def\N{{\bf N}}

\def\PP{{\bf P}}
\def\Q{{\bf Q}}

\def\R{{\bf R}}

\def\S{{\bf S}}

\def\T{{\bf T}}
\def\U{{\bf U}}
\def\u{{\bf u}}
\def\V{{\bf V}}
\def\v{{\bf v}}
\def\W{{\bf W}}
\def\w{{\bf w}}
\def\X{{\bf X}}
\def\x{{\bf x}}
\def\Y{{\bf Y}}
\def\y{{\bf y}}
\def\Z{{\bf Z}}

\def\0{{\bf 0}}
\def\1{{\bf 1}}

\def\NM{{\mathcal N}}
\def\OM{{\mathcal O}}
\def\PM{{\mathcal P}}

\def\SM{{\mathcal S}}

\def\RB{{\mathbb R}}
\def\RBmn{{\RB^{m\times n}}}
\def\EB{{\mathbb E}}
\def\PB{{\mathbb P}}

\def\Ph{\mbox{\boldmath$\Phi$\unboldmath}}

\def\Si{\mbox{\boldmath$\Sigma$\unboldmath}}

\def\Lam{\mbox{\boldmath$\Lambda$\unboldmath}}

\def\Ph{\mbox{\boldmath$\Phi$\unboldmath}}

\def\argmin{\mathop{\rm argmin}}

\def\nnz{\mathrm{nnz}}

\def\range{\mathrm{range}}
\def\tr{\mathrm{tr}}
\def\rk{\mathrm{rank}}
\def\diag{\mathsf{diag}}

\def\poly{\mathrm{poly}}

\def\nystrom{{Nystr\"{o}m} }

\def\Knys{{\tilde{\K}_{c}^{\textrm{nys}}}}
\def\Kmod{{\tilde{\K}_{c}^{\textrm{proto}}}}

\def\Kgen{{\tilde{\K}_{c,s}^{\textrm{fast}}}}
\def\Unys{{\U^{\textrm{nys}}}}
\def\Umod{{\U^{\star}}}
\def\Ugen{{\U^{\textrm{fast}}}}

\newtheorem{question}{Question}
\newtheorem{assumption}{Assumption}

\jmlrheading{17}{2016}{1-49}{5/15; Revised 4/16}{12/16}{Shusen Wang, Zhihua Zhang, and Tong Zhang}
\ShortHeadings{Towards More Efficient SPSD Matrix Approximation and CUR Matrix Decomposition}{Wang, Zhang, and Zhang}
\firstpageno{1}

\begin{document}
\title{Towards More Efficient SPSD Matrix Approximation\\ and CUR Matrix Decomposition}

\author{\name  Shusen Wang  \email shusen@berkeley.edu \\
        \addr Department of Statistics \\
                University of California at Berkeley \\
                Berkeley, CA 94720, USA \\
        \AND
        \name Zhihua Zhang \email zhzhang@math.pku.edu.cn \\
        \addr School of Mathematical Sciences \\
		        Peking University \\
		        Beijing 100871, China
        \AND
        \name Tong Zhang \email tzhang@stat.rutgers.edu \\
        \addr Department of Statistics\\
		        Rutgers University\\
		        Piscataway, New Jersey 08854, USA
        }

\editor{Gert Lanckriet}

\maketitle


\begin{abstract}%
	Symmetric positive semi-definite (SPSD) matrix approximation methods have been extensively used to speed up large-scale eigenvalue computation and kernel learning methods.
	The standard sketch based method, which we call the prototype model, produces relatively accurate approximations, but is inefficient on large square matrices.
	The Nystr\"om method is highly efficient, but can only achieve low accuracy.
	In this paper we propose a novel model that we call the {\it fast SPSD matrix approximation model}.
	The fast model is nearly as efficient as the Nystr\"om method and as accurate as the prototype model.
	We show that the fast model can potentially solve eigenvalue problems and kernel learning problems in linear time with respect to the matrix size $n$ to achieve $1+\epsilon$ relative-error,
	whereas both the prototype model and the Nystr\"om method cost at least quadratic time to attain comparable error bound.
	Empirical comparisons among the prototype model, the Nystr\"om method, and our fast model demonstrate the superiority of the fast model.
	We also contribute new understandings of the Nystr\"om method.
	The Nystr\"om method is a special instance of our fast model
	and is approximation to the prototype model.
	Our technique can be straightforwardly applied to make the CUR matrix decomposition more efficiently computed without much affecting the accuracy.
\end{abstract}



\section{Introduction} \label{sec:introduction}

With limited computational and storage resource,
machine-precision inversion and decompositions of large and dense matrix are prohibitive.
In the past decade matrix approximation techniques have
been extensively studied by
the theoretical computer science community \citep{woodruff2014sketching}, the
machine learning community \citep{mahoney2011ramdomized},
and the numerical linear algebra community \citep{halko2011ramdom}.

In  machine learning, many graph analysis techniques and kernel methods require expensive matrix computations on symmetric matrices.
The truncated eigenvalue decomposition (that is to find a few eigenvectors corresponding to the greatest eigenvalues)
is widely used in graph analysis such as spectral clustering, link prediction in social networks \citep{shin2012multi},
graph matching \citep{patro2012global}, etc.
Kernel methods \citep{scholkopf2002learning} such as kernel PCA and manifold learning
require the truncated eigenvalue decomposition.
Some other kernel methods such as Gaussian process regression/classification require solving  $n\times n$ matrix inversion, where $n$ is the number of training samples. 
The rank $k$ ($k\ll n$) truncated eigenvalue decomposition ($k$-eigenvalue decomposition for short) of an $n\times n$ matrix costs time $\tilde\OM(n^2 k)$\footnote{The $\tilde\OM$
	notation hides the logarithm factors.};
the matrix inversion costs time $\OM (n^3)$.
Thus, the standard matrix computation approaches are infeasible when $n$ is large.

For kernel methods, we are typically given $n$ data samples of dimension $d$,
while the $n\times n$ kernel matrix $\K$ is unknown beforehand and should be computed.
This adds to the additional $\OM(n^2 d)$ time cost.
When $n$ and $d$ are both large, computing the kernel matrix is
prohibitively expensive.
Thus, a good kernel approximation method should avoid the computation of the entire kernel matrix.

Typical SPSD matrix approximation methods speed up matrix computation by
efficiently forming a low-rank decomposition $\K \approx \C \U \C^T$
where $\C\in \RB^{n\times c}$ is a sketch of $\K$ (e.g., randomly sampled $c$ columns of $\K$)
and $\U\in \RB^{c\times c}$ can be computed in different ways.
With such a low-rank approximation at hand,
it takes only $\OM (n c^2)$ additional time to approximately compute the rank $k$ ($k \leq c$) eigenvalue decomposition
or the matrix inversion.
Therefore, if $\C$ and $\U$ are obtained in linear time (w.r.t.\ $n$) and $c$ is independent of $n$,
then the aforementioned eigenvalue decomposition and matrix inversion can be approximately solved in linear time.

The \nystrom method 
is perhaps the most widely used kernel approximation method.
Let $\PP$ be an $n\times c$ sketching matrix
such as uniform sampling \citep{williams2001using,gittens2011spectral},
adaptive sampling \citep{kumar2012sampling},
leverage score sampling \citep{gittens2013revisiting}, etc.
The \nystrom method computes $\C$ by $\C = \K \PP \in \RB^{n\times c}$
and $\U$ by $\U = (\PP^T \C)^\dag \in \RB^{c\times c}$.
This way of computing $\U$ is very efficient, but it
incurs relatively large approximation error even if $\C$ is a good sketch of $\K$.
As a result, the \nystrom method is reported to have low
approximation accuracy in real-world applications \citep{dai2014scalable,hsieh2014fast,si2014multi}.
In fact, the \nystrom is impossible to attain $1+\epsilon$ bound relative to $\|\K - \K_k\|_F^2$
unless $c \geq \Omega \big(\sqrt{n k /\epsilon}\big)$
\citep{wang2013improving}.
Here $\K_k$ denotes the best rank-$k$ approximation of $\K$.
The requirement that $c$ grows at least linearly with $\sqrt{n}$ is a very pessimistic result.
It implies that in order to attain $1+\epsilon$ relative-error bound,
the time cost of the \nystrom method is of order $ n c^2 =
\Omega (n^2 k/\epsilon)$ for solving the $k$-eigenvalue
decomposition or matrix inversion, which is quadratic in $n$.
Therefore, under the $1+\epsilon$ relative-error requirement,
the \nystrom method is not a linear time method.

The main reason for the low accuracy of the \nystrom method is due to the way
that the $\U$ matrix is calculated. In fact, much higher accuracy can be obtained
if $\U$ is calculated by solving the minimization problem $\min_\U \|\K - \C \U \C^T\|_F^2$,
which is a standard way to approximate symmetric matrices \citep{halko2011ramdom,gittens2013revisiting,wang2013improving,wang2014modified}.
This is the randomized SVD for symmetric matrices \citep{halko2011ramdom}.
\citet{wang2014modified} called this approach the prototype model
and provided an algorithm that samples $c = \OM(k/\epsilon)$ columns of $\K$ to form $\C$
such that $\min_\U \|\K - \C \U \C^T\|_F^2 \leq (1+\epsilon) \|\K - \K_k\|_F^2$.
Unlike the Nystr\"om method, the prototype model does not require
$c$ to grow with $n$.
The downside of the prototype model is the high computational cost.
It requires the full observation of $\K$
and $\OM(n^2 c)$ time to compute $\U$.
Therefore when applied to kernel approximation, the time cost cannot be less than $\OM(n^2 d + n^2 c)$.
To reduce the computational cost, this paper considers the problem of efficient calculation of $\U$ with
fixed $\C$ while achieving an accuracy comparable to the prototype model.

More specifically,
the key question we try to answer in this paper can be described as follows.
\begin{question} \label{question:linear}
	For any fixed $n\times n$ symmetric matrix $\K$,
	target rank $k$, and parameter $\gamma$,
	assume that
	\begin{small}
		\begin{enumerate}
			\vspace{-2mm}
			\item[A1]
			We are given a sketch matrix $\C \in \RB^{n\times c}$ of $\K$, which
			is obtained in time $\textrm{Time}(\C)$;
			\vspace{-1mm}
			\item[A2] The matrix $\C$ is a good sketch of $\K$ in that
			$\min_{\U} \|\K - \C \U \C^T\|_F^2 \leq (1+\gamma) \|\K - \K_k\|_F^2$.
		\end{enumerate}
	\end{small}
	
	Then we would like to know whether for an arbitrary $\epsilon$,
	it is possible to compute $\C$ and $\tilde{\U}$ such that the following two requirements are satisfied:
	\begin{small}
		\begin{enumerate}
			\vspace{-2mm}
			\item[R1] The matrix
			$\tilde{\U}$ has the following error bound:
			\[
			\|\K - \C \tilde{\U} \C^T  \|_F^2 \leq (1+\epsilon) (1+\gamma) \|\K - \K_k\|_F^2.
			\vspace{-1mm}
			\]
			\item[R2]
			The procedure of computing $\C$ and $\tilde\U$ and approximately solving the aforementioned
			$k$-eigenvalue decomposition or the matrix inversion run in time
			${\OM} \big( n \cdot \poly (k , \gamma^{-1}, \epsilon^{-1}) \big) +  \textrm{Time}(\C) $.
		\end{enumerate}
	\end{small}
\end{question}

Unfortunately, the following theorem shows that neither the \nystrom method nor the prototype model enjoys such desirable properties.
We prove the theorem in Appendix~\ref{sec:thm_linear}.

\begin{theorem}\label{thm:linear}
	Neither the \nystrom method nor the prototype model satisfies the two requirements in Question~\ref{question:linear}.
	To make requirement R1 hold, both  the \nystrom method and the prototype model cost time no less than
	${\OM} \big( n^2 \cdot \poly (k , \gamma^{-1}, \epsilon^{-1}) \big) + \textrm{Time} (\C)$
	which is at least quadratic in $n$.
\end{theorem}

In this paper we give an affirmative answer to the above question. In particular, it has the following consequences.
First, the overall approximation has high accuracy in the sense that $\|\K - \C \tilde\U \C^T\|_F^2$
is comparable to $\min_\U \|\K - \C \U \C^T\|_F^2$,
and is thereby comparable to the  best rank $k$ approximation.
Second, with $\C$ at hand, the matrix $\tilde\U$ is
obtained efficiently (linear in $n$).
Third, with $\C$ and $\tilde\U$ at hand, it takes extra time which is also linear in $n$
to compute the aforementioned eigenvalue decomposition or linear
system. Therefore, with a good $\C$, we can use linear time to
obtain desired $\U$ matrix such that the accuracy is comparable to the best possible low-rank approximation.

The CUR matrix decomposition \citep{mahoney2009matrix} is
closely related to the prototype model and troubled by the same computational problem.
The CUR matrix decomposition is an extension of the prototype model from symmetric matrices to general matrices.
Given any $m\times n$ fixed matrix $\A$, the CUR matrix decomposition
selects $c$ columns of $\A$ to form $\C \in \RB^{m\times c}$
and $r$ rows of $\A$ to form $\R \in \RB^{r\times n}$,
and computes matrix $\U\in \RB^{c\times r}$ such that $\|\A - \C \U \R\|_F^2$ is small.
Traditionally, it costs time
\[
\OM (m n \cdot \min\{ c, r\})
\]
to compute the optimal $\U^\star = \C^\dag \A \R^\dag$ \citep{stewart1999four,wang2013improving,boutsidis2014optimal}.
How to efficiently compute a high-quality $\U$ matrix for CUR is unsolved.

\subsection{Main Results}

This work is motivated by an intrinsic connection between
the \nystrom method and the prototype model.
Based on a generalization of this observation, we propose the {\it fast SPSD matrix approximation model} for approximating any symmetric matrix.
We show that the fast model satisfies the requirements in Question~\ref{question:linear}.
Given $n$ data points of dimension $d$,
the fast model computes $\C$ and $\Ugen$ and approximately solves the truncated eigenvalue decomposition or matrix inversion in time
\vspace{-1mm}
\[
\OM\big( n c^3 /\epsilon + n c^2 d /\epsilon \big) + \textrm{Time} (\C).
\vspace{-1mm}
\]%
Here $ \textrm{Time} (\C)$ is defined in Question~\ref{question:linear}.

The fast SPSD matrix approximation model achieves the desired properties in
Question~\ref{question:linear} by solving $\min_\U \|\K - \C \U \C^T\|_F $ approximately rather than exactly while ensuring
\[
\|\K - \C \Ugen \C^T \|_F^2
\leq (1+\epsilon) \min_\U \|\K - \C \U \C^T \|_F^2.
\]%
The time complexity for computing $\Ugen$ is linear in $n$, which is far
less than the time complexity $\OM(n^2 c)$ of the prototype model.
Our method also avoids computing the entire kernel matrix $\K$;
instead, it computes a block of $\K$ of size $\frac{\sqrt{n}c}{\epsilon} \times \frac{\sqrt{n}c}{\epsilon}$,
which is substantially smaller than $n\times n$.
The lower bound in Theorem~\ref{thm:lower_bound} indicates that the $\sqrt{n}$ factor here is optimal,
but the dependence on $c$ and $\epsilon$ are suboptimal and can be potentially improved.

This paper provides a new perspective on the Nystr\"om method.
We show that, as well as our fast model,
the Nystr\"om method is approximate solution to the problem $\min_{\U} \|\C \U \C^T - \K\|_F^2$.
Unfortunately, the approximation is so rough that the quality of the \nystrom method is low.

Our method can also be applied to improve the CUR matrix decomposition of the general matrices which are not necessarily square.
Given any matrices $\A \in \RB^{m\times n}$, $\C\in \RB^{m\times c}$, and $\R \in \RB^{r\times n}$,
it costs time $\OM (m n \cdot \min\{c, r\})$ to compute the matrix $\U = \C^\dag \A \R^\dag$.
Applying our technique, the time cost drops to only
\[
\OM \big( c r \epsilon^{-1}\cdot \min \{ m ,n \} \cdot \min \{ c , r\} \big) ,
\]
while the approximation quality is nearly the same.

\subsection{Paper Organization}

The remainder of this paper is organized as follows.
Section~\ref{sec:notation} defines the notation used in this paper.
Section~\ref{sec:related_work} introduces the related work of matrix sketching and SPSD matrix approximation.
Section~\ref{sec:gennystrom} describes our fast model and
analyze the time complexity and error bound.
Section~\ref{sec:cur} applies the technique of the fast model to compute the CUR matrix decomposition more efficiently.
Section~\ref{sec:experiments} conducts empirical comparisons to show the effect of the $\U$ matrix.
The proofs of the theorems are in the appendix.


\section{Notation} \label{sec:notation}

The notation used in this paper are defined as follows.
Let $[n]=\{1, \ldots, n\}$, $\I_n$ be the $n{\times}n$ identity matrix,
and $\1_n$ be the $n\times 1$  vector of all ones.
We let $x \in y \pm z$  denote $y-z \leq x \leq y+ z$.
For an $m {\times} n$ matrix $\A=[A_{i j}]$, we let $\a_{i:}$ be its $i$-th row,
$\a_{:j}$ be its $j$-th column, $\nnz(\A)$ be the number of nonzero entries of $\A$,
$\|\A\|_F = (\sum_{i,j} A_{i j}^2)^{1/2}$ be its Frobenius norm, and
$\|\A\|_2 = \max_{\x\neq \0} \|\A \x\|_2 / \|\x\|_2$ be its spectral norm.

Let $\rho = \rk (\A ) $.
The condensed singular value decomposition (SVD) of $\A$ is defined as
\[
\A
\; = \; \U \Si \V^T
\; = \; \sum_{i=1}^\rho \sigma_i \u_i \v_i^T
\]
where $\sigma_1 , \cdots , \sigma_r$ are the positive singular values in the descending order.
We also use $\sigma_i (\A)$ to denote the $i$-th largest singular value of $\A$.
Unless otherwise specified, in this paper ``SVD'' means the condensed SVD.
Let $\A_k = \sum_{i=1}^k \sigma_i \u_i \v_i^T $ be the top $k$ principal components of $\A$ for any positive integer $k$ less than $\rho$.
In fact, $\A_k$ is the closest to $\A$ among all the rank $k$ matrices.
Let $\A^\dag = \V \Si^{-1} \U^T$ be the {\it Moore-Penrose inverse} of $\A$.

Assume that $\rho = \rk (\A) < n$.
The column leverage scores of $\A$ are $l_i = \|\v_{i:}\|_2^2 $ for $ i = 1$ to $n$.
Obviously, $l_1 + \cdots + l_n = \rho$.
The column coherence is defined by $\nu (\A) = \frac{n}{\rho } \max_{j\in [n]} \|\v_{j:}\|_2^2 $.
If $\rho = \rk (\A) < m$, the row leverage scores and coherence are similarly defined.
The row leverage scores are $\|\u_{1:}\|_2^2 , \cdots , \|\u_{m:}\|_2^2$
and the row coherence is $\mu (\A) = \frac{m}{\rho} \max_{i\in [m]} \|\u_{i:}\|_2^2 $.

\begin{table}[t]
	\caption{A summary of the notation.}
	\label{tab:notation}
	\begin{center}
		\begin{small}
			\begin{tabular}{l |l }
				\hline
				Notation & Description \\
				\hline
				$n$ & number of data points\\
				$d$ & dimension of the data point \\
				$\K$ &  $n\times n$ kernel matrix \\
				$\PP$, $\S$ & sketching matrices \\
				$\C$ & $n\times c$ sketch computed by $\C = \K \PP$\\
				$\Umod$ & $\C^\dag \K (\C^\dag)^T \in \RB^{c\times c}$---the $\U$ matrix of the prototype model\\
				$\Unys$ & $(\PP^T \K)^\dag \in \RB^{c\times c}$---the $\U$ matrix of the Nystr\"om method\\
				$\Ugen$ & $(\S^T \C)^\dag (\S^T \K \S) (\C^T \S)^\dag \in \RB^{c\times c}$---the $\U$ matrix of the fast model\\
				\hline
			\end{tabular}
		\end{small}
	\end{center}
	\vspace{-8mm}
\end{table}

We also list some frequently used notation  in Table~\ref{tab:notation}.
Given the decomposition $\tilde\K = \C \U \C^T \approx \K$ which has rank at most $c$,
it takes $\OM (n c^2)$ time to compute the eigenvalue decomposition of $\tilde{\K}$
and  $\OM (n c^2)$ time to solve the linear system $(\tilde\K + \alpha \I_n) \w = \y$ to obtain $\w$
(see Appendix~\ref{sec:approx_eigen_linear} for more discussions).
The truncated eigenvalue decomposition and  linear system are
the bottleneck of many kernel methods,
and thus an accurate and efficient low-rank approximation can help to accelerate the computation of
kernel learning.

\section{Related Work} \label{sec:related_work}

In Section~{\ref{sec:sketching}} we introduce matrix sketching.
In Section~{\ref{sec:related_nystrom}} we describe two SPSD matrix approximation methods.

\subsection{Matrix Sketching} \label{sec:sketching}

Popular matrix sketching methods include uniform sampling,
leverage score sampling \citep{drineas2006sampling,drineas2008cur,woodruff2014sketching},
Gaussian projection \citep{johnson1984extensions},
subsampled randomized Hadamard transform (SRHT) \citep{drineas2011faster,lu2013faster,tropp2011improved},
count sketch \citep{charikar2004finding,clarkson2013low,meng2013low,nelson2013osnap,pham2013fast,thorup2012tabulation,weinberger2009feature}, etc.

\subsubsection{Column Sampling} \label{sec:sketching_sampling}

Let $p_1 , \cdots , p_n \in (0, 1)$ with $\sum_{i=1}^n p_i = 1$ be the sampling probabilities.
Let each integer in $[n]$ be independently sampled with probabilities $s p_1, \cdots , s p_n$,
where $s \in [n]$ is integer.
Assume that $\tilde{s}$ integers are sampled from $[n]$.
Let $i_1 , \cdots , i_{\tilde{s}}$ denote the selected integers, and let $\EB [\tilde{s}]=s$.
We scale each selected column by $\frac{1}{\sqrt{s p_{i_1}}} , \cdots , \frac{1}{\sqrt{s p_{i_{\tilde{s}}}}}$, respectively.
Uniform sampling means that the sampling probabilities are $p_1  = \cdots = p_n = \frac{1}{n}$.
Leverage score sampling means that the sampling probabilities are proportional to the leverage scores $\l_1, \cdots , \l_n$ of a certain matrix.

We can equivalently characterize column selection by the matrix $\S \in \RB^{n\times \tilde{s}}$.
Each column of $\S$ has exactly one nonzero entry;
let $(i_j, j)$ be the position of the nonzero entry in the $j$-th column for $j \in [\tilde{s}]$.
For $j = 1$ to $\tilde{s}$, we set
\begin{equation} \label{eq:def_S}
S_{i_j, j} = \frac{1}{\sqrt{s p_{i_j}}} .
\end{equation}
The expectation $\EB [\tilde{s}]$ equals to $s$,
and $\tilde{s} = \Theta (s)$ with high probability.
For the sake of simplicity and clarity, in the rest of this paper we will not distinguish ${\tilde{s}}$ and $s$.

\subsubsection{Random Projection}

Let $\G \in \RB^{n\times s}$ be a standard Gaussian matrix, namely each entry is sampled independently from $\NM (0, 1)$.
The matrix $\S = \frac{1}{\sqrt{s}} \G$ is a Gaussian projection matrix.
Gaussian projection is also well known as the Johnson-Lindenstrauss (JL) transform \citep{johnson1984extensions};
its theoretical property is  well established.
It takes $\OM (mns)$ time to apply $\S\in \RB^{n\times s}$ to any $m\times n$ dense matrix,
which makes Gaussian projection inefficient.

The subsampled randomized Hadamard transform (SRHT) is usually a more efficient alternative of Gaussian projection.
Let $\H_n \in \RB^{n\times n}$ be the Walsh-Hadamard matrix with $+1$ and $-1$ entries,
$\D \in \RB^{n\times n}$ be a diagonal matrix with diagonal entries sampled uniformly from $\{+1, -1\}$,
and $\PP \in \RB^{n\times s}$ be the uniform sampling matrix defined above.
The matrix $\S = \frac{1}{\sqrt{n}} \D \H_n \PP \in \RB^{n\times s}$ is an SRHT matrix,
and it can be applied to any $m\times n$ matrix in $\OM(mn\log s)$ time.

Count sketch stems from the data stream literature \citep{charikar2004finding,thorup2012tabulation} and
has been applied to speedup matrix computation.
The count sketch matrix $\S \in \RB^{n\times s}$ can be applied to any matrix $\A$ in $\OM (\nnz (\A))$ time
where $\nnz$ denotes the number of non-zero entries.
The readers can refer to \citep{woodruff2014sketching} for detailed descriptions of count sketch.

\subsubsection{Theories}

The following lemma shows important properties of the matrix sketching methods.
In the lemma, leverage score sampling means that the sampling probabilities are proportional to the row leverage scores of the column orthogonal matrix $\U \in \RB^{n\times k}$.
(Here $\U$ is different from the notation elsewhere in the paper.)
We prove the lemma in Appendix~\ref{sec:proof:lem:product}.

\begin{lemma} \label{lem:product}
	Let $\U \in \RB^{n\times k}$ be any fixed matrix with orthonormal columns and
	$\B \in \RB^{n\times d}$ be any fixed matrix.
	Let $\S \in \RB^{n\times s}$ be any sketching matrix considered in this section;
    the order of $s$ (with the $\OM$-notation omitted) is listed in Table~\ref{tab:product_average}.
	Then
	\begin{eqnarray*}
		\PB \Big\{ \big\|\U^T \S \S^T \U - \I_k \big\|_2 \geq \eta \Big\}
		\; \leq \; \delta_1
		& \qquad \textrm{(Property 1)}, \\
		\PB \Big\{\big\| \U^T \B - \U^T \S \S^T \B \big\|_F^2 \geq {\epsilon} \|\B\|_F^2 \Big\}
		\; \leq \; \delta_2
		& \qquad \textrm{(Property 2)} , \\
		\PB \Big\{\big\| \U^T \B - \U^T \S \S^T \B \big\|_2^2 \geq {\epsilon'} \|\B\|_2^2 + \frac{\epsilon'}{k} \|\B\|_F^2 \Big\}
		\; \leq \; \delta_3
		& \qquad \textrm{(Property 3)}.
	\end{eqnarray*}
\end{lemma}

\begin{table}[!h]\setlength{\tabcolsep}{0.3pt}
	\caption{The leverage score sampling is w.r.t.\ the row leverage scores of $\U$.
		For uniform sampling, the notation $\mu (\U ) \in [1, n]$ is the row coherence of $\U$.
	}
	\label{tab:product_average}
	\begin{center}
		\begin{tabular}{c c c c}
			\hline
			{\bf Sketching}             &~~~Property 1~~~&~~~Property 2~~~&~~~Property 3~~~\\
			\hline
			Leverage Sampling  &~~~$ \frac{k}{\eta^2} \log \frac{k}{\delta_1} $~~~
			&~~~$ \frac{k}{\epsilon \delta_2}  $~~~
			& --- \\
			Uniform Sampling &~~~$ \frac{\mu (\U) k }{ \eta^2} \log \frac{k}{\delta_1} $~~~
			&~~~$ \frac{\mu (\U) k }{\epsilon \delta_2} $~~~
			& --- \\
			Gaussian Projection\ &~~~$\frac{  k +  \log (1/\delta_1) }{\eta^2}$~~~
			& ~~~$\frac{k}{\epsilon \delta_2}$~~~
			& ~~~$\frac{1}{\epsilon' }  \big(k + \log \frac{d}{k\delta_3}  \big)$~~~    \\
			SRHT            &~~~$\frac{k + \log n}{\eta^2} \log \frac{k}{\delta_1}$~~~
			& ~~~$\frac{ k + \log n}{\epsilon  \delta_2  }$~~~
			& ~~~$  \frac{1}{ \epsilon'} \big( k + \log \frac{n {d}}{k \delta_1} \big) \log \frac{d }{\delta_3 }$~~~\\
			Count Sketch &~~~$\frac{k^2 }{\delta_1 \eta^2}$~~~
			& ~~~$\frac{ k}{\epsilon \delta_2}$~~~
			& ~~~---~~~    \\
			\hline
		\end{tabular}
	\end{center}
\end{table}

Property~1 is known as the subspace embedding property~\citep{woodruff2014sketching}.
It shows that all the singular values of $\S^T \U$ are close to one.
Properties~2 and 3 show that sketching preserves
the multiplication of a row orthogonal matrix and an arbitrary matrix.

For the SPSD/CUR matrix approximation problems,
the three properties are all we need to
capture the randomness in the sketching methods.
Leverage score sampling, uniform sampling, and count sketch
do not enjoy Property~3,
but it is fine---
Frobenius norm (Property~2) will be used as a loose upper bound on the spectral norm (Property~3).
Gaussian projection and SRHT satisfy all the three properties;
when applied to the SPSD/CUR problems,
their error bounds are stronger than the leverage score sampling, uniform sampling, and count sketch.

\subsection{SPSD Matrix Approximation Models} \label{sec:related_nystrom}

We first describe the prototype model and the Nystr\"om method, which are most relevant to this work.
We then introduce several other SPSD matrix approximation methods.

\subsubsection{Most Relevant Work} \label{sec:related_nystrom_most}

Given an $n\times n$ matrix $\K$ and an $n\times c$ sketching matrix $\PP$,
we let $\C = \K \PP$ and $\W = \PP^T \C = \PP^T \K \PP$.
{\it The prototype model} \citep{wang2013improving} is defined by
\begin{eqnarray}  \label{eq:def_modified}
\Kmod \;\triangleq \; \C \Umod \C^T \;=\; \C \C^\dag \K (\C^\dag)^T  \C^T ,
\end{eqnarray}%
and {\it the \nystrom method} is defined by
\begin{eqnarray} \label{eq:def_standard}
\Knys & \triangleq & \C \U^{\textrm{nys}} \C^T \;=\; \C \W^\dag \C^T \nonumber \\
& = & \C \big(\PP^T \C \big)^\dag \big(\PP^T \K \PP \big) \big(\C^T \PP \big)^\dag \C^T .
\end{eqnarray}%
The only difference between the two models is their $\U$ matrices,
and the difference leads to big difference in their approximation accuracies.
\citet{wang2013improving} provided a lower error bound of the \nystrom method,
which shows that no algorithm can select less than $\Omega (\sqrt{n k /\epsilon })$ columns of $\K$ to form $\C$ such that
\[
\| \K - \C \Unys \C^T \|_F^2 \leq (1+\epsilon) \|\K - \K_k\|_F^2 .
\]%
In contrast, the prototype model can attain the $1+\epsilon$ relative-error bound with $c = \OM(k / \epsilon)$
\citep{wang2014modified},
which is optimal up to a constant factor.

While we have mainly discussed the time complexity of kernel approximation in the previous sections,
the memory cost is often a more important issue in large scale
problems due to the limitation of computer memory.
The \nystrom method and the prototype model require $\OM (n c)$ memory
to hold $\C$ and $\U$ to approximately solve the aforementioned eigenvalue decomposition
or the linear system.\footnote{The memory costs of the prototype model
	is $\OM(n c + nd)$ rather than $\OM(n^2)$.
	This is because we can hold the $n\times d$ data matrix and the $c\times n$ matrix $\C^\dag$ in memory,
	compute a small block of $\K$ each time, and then compute $\C^\dag \K$ block by block.}
Therefore, we hope to make $c$ as small as possible while achieving a
low  approximation error.
There are two elements: (1) a good sketch $\C = \K \PP$, and (2)  a high-quality $\U$ matrix.
We focus on the latter in this paper.

\subsubsection{Less Relevant Work} \label{sec:related_nystrom_less}

We note that there are many other kernel approximation approaches in the literature.
However, these approaches do not directly address the issue we consider here,
so they are complementary to our work.
These studies are either less effective or inherently rely on the \nystrom method.

The Nystr\"om-like models such as MEKA~\citep{si2014memory} and the ensemble \nystrom method~\citep{kumar2012sampling}
are reported to significantly outperform the \nystrom method in terms
of approximation accuracy, but their key components are still the \nystrom method
and the component can be replaced by any other methods such as the
method studied in this work.
The spectral shifting \nystrom method~\citep{wang2014improving} also outperforms the \nystrom method in certain situations,
but the spectral shifting strategy can be used for any other kernel approximation models beyond the prototype model.
We do not compare with these methods in this paper because
MEKA, the ensemble \nystrom method, and the spectral shifting \nystrom method
can all be improved if we replace the underlying \nystrom method or the prototype model by the new method developed here.

The column-based low-rank approximation model \citep{kumar2009sampling}
is another SPSD matrix approximation approach different from the Nystr\"om-like methods.
Let $\PP \in \RB^{n\times c}$ be any sketching matrix and $\C = \K \PP$.
The column-based model approximates $\K$ by $\C (\C^T \C)^{-1/2} \C^T = (\C \C^T)^{1/2}$.
Equivalently, it approximates $\K^2$ by
\[
\K^T \K
\; \approx \; \C \C^T
\; = \; \K^T \PP \PP^T \K .
\]
From Lemma~\ref{lem:product} we can see that it is a typical sketch based approximation to the matrix multiplication.
Unfortunately, the approximate matrix multiplication is effective only when $\K$ has much more rows than columns,
which is not true for the kernel matrix.
The column-based model does not have good error bound
and is not empirically as good as the Nystr\"om method
\citep{kumar2009sampling}.

The random feature mapping \citep{rahimi2007random} is a family of kernel approximation methods.
Each random feature mapping method is applicable
to certain kernel rather than arbitrary SPSD matrix.
Furthermore, they are known to be noticeably less effective than the \nystrom method \citep{yang2012nystrom}.


\section{The Fast SPSD Matrix Approximation Model} \label{sec:gennystrom}

In Section~\ref{sec:motivation} we present the motivation behind the fast model.
In Section~\ref{sec:interpretation} we provide an alternative perspective on our fast model and the \nystrom method
by formulating them as approximate solutions to an optimization problem.
In Section~\ref{sec:theory} we analyze the error bound of the fast model.
Theorem~\ref{thm:faster_spsd} is the main theorem, which shows that
in terms of the Frobenius norm approximation,
the fast model is almost as good as the prototype model.
In Section~\ref{sec:algorithm} we describe the implementation of the fast model and analyze the time complexity.
In Section~\ref{sec:implementation} we give some implementation details that help to improve the approximation quality.
In Section~\ref{sec:lower_bound_sn} we show that our fast model exactly recovers $\K$ under certain conditions,
and we provide a lower error bound of the fast model.


\subsection{Motivation} \label{sec:motivation}

Let $\PP \in \RB^{n\times c}$ be sketching matrix and $\C = \K \PP \in \RB^{n\times c}$.
{The fast SPSD matrix approximation model} is defined by
\begin{eqnarray*}
	\Kgen & \triangleq &  \C \big(\S^T \C \big)^\dag \big(\S^T \K \S \big) \big( \C^T \S \big)^\dag \C^T,
\end{eqnarray*}%
where $\S$ is $n\times s$ sketching matrix.

From (\ref{eq:def_modified}) and (\ref{eq:def_standard})
we can see that the \nystrom method is a special case of the fast model
where $\S$ is defined as $\PP$
and that the prototype model is a special case where $\S$ is defined as  $\I_n$.

The fast model allows us to trade off the accuracy and the computational cost---larger $s$ leads
to higher accuracy and higher time cost, and vice versa.
Setting $s$ as small as $c$ sacrifices too much accuracy, whereas setting $s$ as large as $n$ is unnecessarily expensive.
Later on, we will show that $s = \OM(c \sqrt{n / \epsilon}) \ll n $ is a good choice.
The setting $s \ll n$  makes the fast model much cheaper to compute than the prototype model.
When applied to kernel methods,
the fast model avoids computing the entire kernel matrix.
We summarize the time complexities of the three matrix approximation methods in Table~\ref{tab:time};
the middle column lists the time cost for computing the $\U$ matrices given $\C$ and $\K$;
the right column lists the number of entry of $\K$ which much be observed.
We show a very intuitive comparison in Figure~\ref{fig:illustrate}.

\begin{table}[t]
	\caption{Summary of the time cost of the models for computing the $\U$ matrices and
		the number of entries of $\K$ required to be observed in order to compute the $\U$ matrices.
		As for the fast model, assume that $\S$ is column selection matrix.
		The notation is defined previously in Table~\ref{tab:notation}.}
	\label{tab:time}
	\begin{center}
		\begin{tabular}{l c c }
			\hline
			& Time & \#Entries \\
			\hline
			Nystr\"om   & $\OM (c^3)$     & $nc$ \\
			Prototype   & $\OM \big( \nnz (\K) c + n c^2 \big)$   & $n^2 $ \\
			Fast  & $ \OM (n c^2 + s^2 c) $ & $n c  + (s-c)^2  $ \\
			\hline
		\end{tabular}
	\end{center}
	\vspace{-4mm}
\end{table}

\begin{figure}[t]
	\begin{center}
		\centering
		\includegraphics[width=0.7\textwidth]{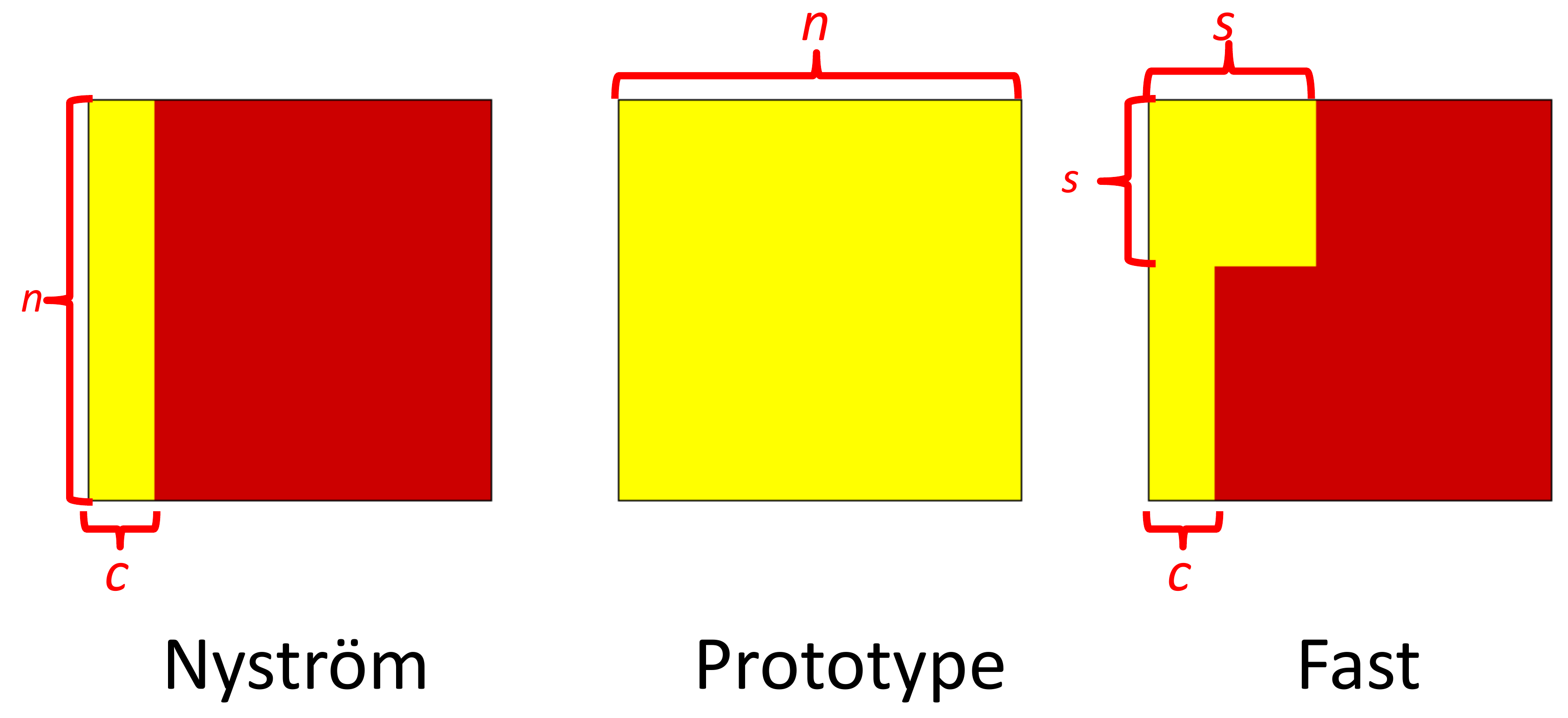}
		\vspace{-8mm}
	\end{center}
	\caption{The {\color{Goldenrod} yellow blocks} denote the submatrices of $\K$
		that must be seen by the kernel approximation models.
		The \nystrom method computes an $n\times c$ block of $\K$,
		provided that $\PP$ is column selection matrix;
		the prototype model computes the entire $n\times n$ matrix $\K$;
		the fast model computes an $n\times c$ block and an $(s-c)\times (s-c)$ block of $\K$ (due to the symmetry of $\K$),
		provided that $\PP$ and $\S$ are column selection matrices.}
	\label{fig:illustrate}
\end{figure}


\subsection{Optimization Perspective} \label{sec:interpretation}

With the sketch $\C = \K \PP \in \RB^{n\times c}$ at hand, we want to find the $\U$ matrix
such that $\C \U \C^T \approx \K$.
It is very intuitive to solve the following problem to make the approximation tight:
\begin{align}
\Umod  \; = \; \argmin_{\U}
\big\| \C \U \C^T - \K \big\|_F^2
\; = \; \C^\dag \K (\C^\dag)^T. \label{eq:def_x_opt}
\end{align}%
This is the prototype model.
Since solving this system is time expensive,
we propose to draw a sketching matrix $\S \in \RB^{n\times s}$ and solve the following problem instead:
\begin{eqnarray}
\Ugen
& = &  \argmin_{\U} \big\| \S^T (\C \U \C^T -  \K) \S \big\|_F^2 \nonumber \\
& = &  \argmin_{\U} \big\| (\S^T \C) \U (\S^T \C)^T - \S^T \K \S \big\|_F^2 \nonumber \\
& = & (\S^T \C)^\dag (\S^T \K \S) (\C^T \S)^\dag , \label{eq:def_x_approx}
\end{eqnarray}%
which results in the fast model.
Similar ideas have been exploited to efficiently solve the least squares regression problem
\citep{drineas2006sampling,drineas2011faster,clarkson2013low},
but their analysis can not be directly applied to the more complicated system \eqref{eq:def_x_approx}.

This approximate linear system interpretation offers a new perspective on the \nystrom method.
The $\U$ matrix of the \nystrom method is in fact an approximate solution to the problem
$\min_\U \|\C \U \C^T - \K \|_F^2$.
The \nystrom method uses $\S = \PP$ as the sketching matrix, which leads to the solution
\begin{eqnarray*}
	\Unys
	\; = \;  \argmin_{\U} \big\| \PP^T (\C \U \C^T -  \K) \PP \big\|_F^2
	\; =\; (\PP^T \K \PP)^\dag \; = \; \W^\dag .
\end{eqnarray*}%


\subsection{Error Analysis} \label{sec:theory}

Let $\Ugen$ correspond to the fast model \eqref{eq:def_x_approx}.
Any of the five sketching methods in Lemma~\ref{lem:product} can be used to compute $\Ugen$,
although column selection is more useful than random projection in this application.
In the following we show that $\Ugen$ is nearly as good as $\Umod$ in terms of the objective function value.
The proof is in Appendix~\ref{sec:proof_faster_spsd}.

\begin{theorem}[Main Result] \label{thm:faster_spsd}
	Let $\K$ be any $n\times n$ fixed symmetric matrix,
	$\C$ be any $n\times c$ fixed matrix, $k_c = \rk (\C)$,
	and $\Ugen$ be the $c\times c$ matrix defined in (\ref{eq:def_x_approx}).
	Let $\S \in \RB^{n\times s}$ be any of the five sketching matrices defined in Table~\ref{tab:faster_spsd}.
	Assume that $\epsilon^{-1} = o(n)$ or $\epsilon^{-1} = o(n/c)$.
	The inequality
	\begin{small}
		\begin{eqnarray}
		\big\| \K - \C \Ugen \C^T \big\|_F^2
		&  \leq & (1+\epsilon) \; \min_{\U} \big\| \K - \C \U \C^T \big\|_F^2 \label{eq:thm:faster_spsd}
		\end{eqnarray}%
	\end{small}%
	holds with probability at least $0.8$.
\end{theorem}

In the theorem, Gaussian projection and SRHT require smaller sketch size than the other three methods.
It is because Gaussian projection and SRHT enjoys all of Properties~1, 2, 3 in Lemma~\ref{lem:product},
whereas leverage score sampling, uniform sampling, and count sketch does not enjoy Property~3.

\begin{table}[!h]\setlength{\tabcolsep}{0.3pt}
	\caption{
		Leverage score sampling means sampling according to the row leverage scores of $\C$.
		For uniform sampling, the parameter $\mu (\C) \in [1, n]$ is the row coherence of $\C$.}
	\label{tab:faster_spsd}
	\begin{center}
		\begin{footnotesize}
			\begin{tabular}{c c c c c}
				\hline
				{\bf Sketching}            &Order of $s$ & ~~~Assumption~~~ & ~~~~~$T_{\textrm{sketch}}$~~~~~ & ~~~~~\#Entries~~~~~ \\
				\hline
				Leverage Score Sampling  &~~~~~$ c  \sqrt{n/\epsilon}$~~~
				& $\epsilon = o (n)$
				& $\OM (nc^2 + s^2) $
				& $n c + (s-c)^2 $    \\
				Uniform Sampling  &~~~~~$ \mu (\C) c \sqrt{n/\epsilon}$~~~
				& $\epsilon = o (n)$
				& $\OM (s^2) $
				& $n c  + (s-c)^2$     \\
				Gaussian Projection\
				&~~$ \sqrt{\frac{n}{c\epsilon}} \big( c  + \log \frac{n}{c} \big)$~~
				& $\epsilon = o (n / c)$
				& ~~$\OM \big( \nnz (\K) s \big) $~~
				& $n^2 $      \\
				SRHT
				&~~~$\sqrt{\frac{n}{c\epsilon}} (c + \log n) \log (n)$~~~
				& $\epsilon = o (n / c)$
				& $\OM (n^2 \, \log s) $
				& $n^2 $    \\
				Count Sketch    &~~$ c \sqrt{n/\epsilon}$~~
				& $\epsilon = o (n)$
				& $\OM \big( \nnz (\K)   \big) $
				& $n^2 $     \\
				\hline
			\end{tabular}
		\end{footnotesize}
	\end{center}
\end{table}

\begin{remark} \label{remark:optimal_nystrom}
	\citet{wang2014modified} showed that there exists an algorithm (though not linear-time algorithm) attaining the error bound
	\[
	\big\| \K - \C \C^\dag \K (\C^\dag)^T \C^T \big\|_F^2
	\; \leq \; (1+\epsilon) \big\| \K - \K_k \big\|_F^2
	\]%
	with high probability by sampling $c = \OM( k / \epsilon )$ columns of $\K$ to form $\C$.
	Let $\C \in \RB^{n\times c}$ be formed by this algorithm
	and $\S \in \RB^{n\times s}$ be the leverage score sampling matrix.
	With $c = \OM( k / \epsilon )$ and $s = \tilde\OM (n^{1/2} k \epsilon^{-3/2} )$,
	the fast model satisfies
	\[
	\big\| \K - \C \U^{\textrm{fast}} \C^T \big\|_F^2
	\; \leq \; (1+\epsilon) \big\| \K - \K_k \big\|_F^2
	\]
	with high probability.
\end{remark}


\subsection{Algorithm and Time Complexity} \label{sec:algorithm}

We describe the whole procedure of the fast model in Algorithm~\ref{alg:gennystrom},
where $\S \in \RB^{n\times s}$ can be one of the five sketching matrices described in Table~\ref{tab:faster_spsd}.
Given $\C$ and (the whole or a part of) $\K$, it takes time
\[
\OM \big( s^2 c \big) + T_{\textrm{sketch}}
\]
to compute $\Ugen$,
where $T_{\textrm{sketch}}$ is the time cost of forming the sketches $\S^T \C$ and $\S^T \K \S$ and is described in Table~\ref{tab:faster_spsd}.
In Table~\ref{tab:faster_spsd} we also show the number of entries of $\K$ that must be observed.
From Table~\ref{tab:faster_spsd} we can see that column selection is much more efficient than random projection,
and column selection does not require the full observation of $\K$.

We are particularly interested in the column selection matrix $\S$ corresponding to the row leverage scores of $\C$.
The leverage score sampling described in Algorithm~\ref{alg:leverage}
can be efficiently performed.
Using the leverage score sampling,
it takes time $\OM (n c^3/ \epsilon)$ (excluding the time of computing $\C = \K \PP$)
to compute $\Ugen$.
For the kernel approximation problem,
suppose
that we are given $n$ data points of dimension $d$ and that the kernel matrix $\K$ is unknown beforehand.
Then it takes $\OM (n c^2 d / \epsilon)$ additional time to evaluate the kernel function values.

\begin{algorithm}[tb]
	\caption{The Fast SPSD Matrix Approximation Model.}
	\label{alg:gennystrom}
	\begin{small}
		\begin{algorithmic}[1]
			\STATE {\bf Input:} an $n\times n$ symmetric matrix $\K$ and the number of selected columns or target dimension of projection $c$ $(< n)$.
			\STATE {Sketching:} $\C=\K \PP$ using an arbitrary $n\times c$
			sketching matrix $\PP$ (not studied in this work); \label{alg:gennystrom:sketch1}
			\STATE Optional: replace $\C$ by any orthonormal bases of the columns of $\C$;  \label{alg:gennystrom:optional}
			\STATE Compute another $n\times s$ sketching matrix $\S$, e.g.\ the leverage score sampling in Algorithm~\ref{alg:leverage};  \label{alg:gennystrom:formS}
			\STATE Compute the sketches $\S^T \C \in \RB^{s\times c}$ and $\S^T \K \S\in \RB^{s\times s}$; \label{alg:gennystrom:sketch2}
			\STATE Compute $\Ugen = (\S^T \C)^\dag (\S^T \K \S) (\C^T \S)^\dag \in \RB^{c\times c}$; \label{alg:gennystrom:intersection}
			\STATE {\bf Output:} $\C$ and $\Ugen$ such that $\K \approx \C \Ugen \C^T$.
		\end{algorithmic}
	\end{small}
\end{algorithm}

\begin{algorithm}[tb]
	\caption{The Leverage Score Sampling Algorithm.}
	\label{alg:leverage}
	\begin{small}
		\begin{algorithmic}[1]
			\STATE {\bf Input:} an $n\times c$ matrix $\C$, an integer $s$.
			\STATE Compute the condensed SVD of $\C$ (by discarding the zero singular values) to obtain the orthonormal bases $\U_\C \in \RB^{n\times \rho}$,
			where $\rho = \rk (\C) \leq c$;
			\STATE Compute the sampling probabilities $p_i = s \ell_i / \rho$, where $\ell_i = \|\e_i^T \U_\C \|_2^2$ is the $i$-th leverage score;
			\STATE Initialize $\S$ to be an matrices of size $n\times 0$;
			\FOR{$i = 1$ to $n$}
			\STATE With probability $p_i$, add $\sqrt{\frac{c}{s \ell_i }}  \e_i $ to be a new column of $\S$,
			where $\e_i$ is the $i$-th standard basis;
			\ENDFOR
			\STATE {\bf Output:} $\S$, whose expected number of columns is $s$.
		\end{algorithmic}
	\end{small}
\end{algorithm}



\subsection{Implementation Details} \label{sec:implementation}

In practice, the approximation accuracy and numerical stability can be significantly improved by the following techniques and tricks.

If $\PP$ and $\S$ are both random sampling matrices, then empirically speaking,
enforcing $\PM \subset \SM$ significantly improves the approximation accuracy.
Here $\PM$ and $\SM$ are the subsets of $[n]$ selected by $\PP$ and $\S$, respectively.
Instead of directly sampling $s$ indices from $[n]$ by Algorithm~\ref{alg:leverage},
it is better to sample $s$ indices from $[n] \setminus \PM$ to form $\SM'$
and let $\SM = \SM' \cup \PM$.
In this way, $s+c$ columns are sampled.
Whether the requirement $\PM \subset \SM $ improves the accuracy is unknown to us.

\begin{corollary} \label{cor:sketch_nystrom_restrict}
	Theorem~\ref{thm:faster_spsd} still holds when we restrict $\PM \subset \SM$.
\end{corollary}

\begin{proof}
	Let $p_1, \cdots , p_n$ be the original sampling probabilities without the restriction $\PM \subset \SM$.
	We define the modified sampling probabilities by
	\[
	\tilde{p}_i
	\; = \; \left\{
	\begin{array}{l l}
	1 & \textrm{if } i \in \PM ; \\
	p_i & \textrm{otherwise .} \\
	\end{array}
	\right.
	\]
	The column sampling with restriction $\PM \subset \SM$ amounts to sampling columns according to $\tilde{p}_1, \cdots , \tilde{p}_n$.
	Since $\tilde{p}_i \geq p_i$ for all $i \in [n]$,
	it follows from Remark~\ref{remark:column_sampling} that the error bound will not get worse if $p_i$ is replaced by $\tilde{p}_i$.
\end{proof}

If $\S$ is the leverage score sampling matrix, we find it better not to scale the entries of $\S$,
although the scaling is necessary for theoretical analysis.
According to our observation, the scaling sometimes makes the approximation numerically unstable.


\subsection{Additional Properties} \label{sec:lower_bound_sn}

When $\K$ is a low-rank matrix,
the \nystrom method and the prototype model are guaranteed to exactly recover $\K$ \citep{kumar2009sampling,talwalkar2010matrix,wang2014modified}.
We show in the following theorem that the fast model has the same property.
We prove the theorem in Appendix~\ref{sec:proof_exact}.

\begin{theorem}[Exact Recovery] \label{thm:exact}
	Let $\K$ be any $n\times n$ symmetric matrix,
	$\PP \in \RB^{n\times c}$ and $\S \in \RB^{n\times s}$ be any sketching matrices,
	$\C = \K \PP$, and $\W = \PP^T \C$.
	Assume that $\rk (\S^T \C) \geq \rk (\W) $.
	Then $\K = \C (\S^T \C)^\dag (\S^T \K \S) (\C^T \S)^\dag \C^T$
	if and only if $\rk (\K) = \rk (\C)$.
\end{theorem}

In the following we establish a lower error bound of the fast model,
which implies that to attain the $1+\epsilon$ Frobenius norm bound relative to the best rank $k$ approximation,
the fast model must satisfy
\[
c \geq \Omega \big(k / \epsilon \big)
\quad \textrm{ and } \quad
s \geq \Omega \big(\sqrt{n k / \epsilon} \big) .
\]
Notice that the theorem only holds for column selection matrices $\PP$ and $\S$.
We prove the theorem in Appendix~\ref{sec:proof_lower_bound}.

\begin{theorem} [Lower Bound] \label{thm:lower_bound}
	Let $\PP \in \RB^{n\times c}$ and $\S \in \RB^{n\times s}$ be any two column selection matrices
	such that $\PM \subset \SM \subset [n]$,
	where $\PM$ and $\SM$ are the index sets formed by $\PP$ and $\S$, respectively.
	There exists an $n\times n$ symmetric matrix $\K$ such that
	\begin{eqnarray} \label{eq:sn_lower_bound}
	\frac{\|\K - \Kgen \|_F^2}{\| \K - \K_k\|_F^2 }
	& \geq & \frac{n-c}{n-k} \Big( 1 + \frac{2 k }{c} \Big) + \frac{n-s}{n-k} \frac{k (n-s)}{s^2} ,
	\end{eqnarray}
	where $k$ is arbitrary positive integer smaller than $n$,
	$\C = \K \PP \in \RB^{n\times c}$, and
	\[
	\Kgen
	\; = \;
	\C (\S^T \C)^\dag (\S^T \K \S) (\C^T \S)^\dag \C^T
	\]
	is the fast model.
\end{theorem}

Interestingly, Theorem~\ref{thm:lower_bound} matches the lower bounds of the \nystrom method and the prototype model.
When $s=c$, the right-hand side of \eqref{eq:sn_lower_bound} becomes $\Omega (1+ k n / c^2 )$,
which is the lower error bound of the \nystrom method given by \citet{wang2013improving}.
When $s=n$, the right-hand side of \eqref{eq:sn_lower_bound} becomes $\Omega (1+ k/c )$,
which is the lower error bound of the prototype model given by \citet{wang2014modified}.


%


\section{Extension to CUR Matrix Decomposition} \label{sec:cur}

In Section~\ref{sec:cur_classic} we describe the CUR matrix decomposition and establish an improved error bound of CUR in Theorem~\ref{thm:optimal_cur}.
In Section~\ref{sec:faster_cur} we use sketching to more efficiently compute the $\U$ matrix of CUR.
Theorem~\ref{thm:optimal_cur} and Theorem~\ref{thm:cur} together show that our fast CUR method satisfies $1+\epsilon$ error bound relative to the best rank $k$ approximation.
In Section~\ref{sec:cur_experiments} we provide empirical results to intuitively illustrate the effectiveness of our fast CUR.
In Section~\ref{sec:cur_discussions} we discuss the application of our results beyond the CUR decomposition.

\subsection{The CUR Matrix Decomposition} \label{sec:cur_classic}

Given any $m\times n$ matrix $\A$,
the CUR matrix decomposition is computed by
selecting $c$ columns of $\A$ to form $\C \in \RB^{m\times c}$
and $r$ rows of $\A$ to form $\R \in \RB^{r\times n}$
and computing the $\U$ matrix such that $\|\A - \C \U \R \|_F^2$ is small.
CUR preserves the sparsity and non-negativity properties of $\A$;
it is thus more attractive than SVD in certain applications~\citep{mahoney2009matrix}.
In addition, with the CUR of $\A$ at hand, the truncated SVD of $\A$ can be very efficiently computed.

A standard way to finding the $\U$ matrix is by minimizing $\|\A - \C \U \R \|_F^2$ to obtain
the optimal $\U$ matrix
\begin{align} \label{def:cur_u_opt}
\U^\star \; = \;
\argmin_{\U} \|\A - \C \U \R \|_F^2
\; = \; \C^\dag \A \R^\dag ,
\end{align}
which has been used by \citet{stewart1999four,wang2013improving,boutsidis2014optimal}.
This approach costs time $\OM (m c^2 + n r^2)$ to compute the Moore-Penrose inverse
and $\OM (m n \cdot \min\{ c, r\})$ to compute the matrix product.
Therefore, even if $\C$ and $\R$ are uniformly sampled from $\A$,
the time cost of CUR is $\OM (m n \cdot \min\{ c, r\})$.

At present the strongest theoretical guarantee is by \citet{boutsidis2014optimal}.
They use the adaptive sampling algorithm to select $c = \OM (k/\epsilon)$ column and $r = \OM (k /\epsilon)$ rows to form $\C$ and $\R$, respectively,
and form $\U^\star = \C^\dag \A \R^\dag$.
The approximation error is bounded by
\[
\| \A - \C \U^\star \R \|_F^2
\; \leq \;
(1+\epsilon ) \| \A - \A_k \|_F^2.
\]
This result matches the theoretical lower bound up to a constant factor.
Therefore this CUR algorithm is near optimal.
We establish in Theorem~\ref{thm:optimal_cur} an improved error bound of the
adaptive sampling based CUR algorithm,
and the constants in the theorem are better than the those in \citep{boutsidis2014optimal}.
Theorem~\ref{thm:optimal_cur} is obtained by following the idea of \citet{boutsidis2014optimal}
and slightly changing the proof of \citet{wang2013improving}.
The proof is in Appendix~\ref{sec:optimal_cur_proof}.

\begin{theorem} \label{thm:optimal_cur}
	Let $\A$ be any given $m\times n$ matrix, $k$ be any positive integer less than $m$ and $n$,
	and $\epsilon \in (0, 1)$ be an arbitrary error parameter.
	Let $\C \in \RB^{m\times c}$ and $\R \in \RB^{r\times n}$ be
	columns and rows of $\A$ selected by the near-optimal column selection algorithm of \cite{boutsidis2011NOCshort}.
	When $c$ and $r$ are both greater than $ 4 k \epsilon^{-1} \big( 1+o(1) \big)$,
	the following inequality holds:
	\[
	\EB \big\| \A - \C \C^\dag \A \R^\dag \R \|_F^2
	\; \leq \;
	(1+\epsilon ) \| \A - \A_k \|_F^2 ,
	\]
	where the expectation is taken w.r.t.\ the random column and row selection.
\end{theorem}


\subsection{Fast CUR Decomposition} \label{sec:faster_cur}

Analogous to the fast SPSD matrix approximation model,
the CUR decomposition can be sped up while preserving its accuracy.
Let $\S_C \in \RB^{m\times s_c}$ and $\S_R \in \RB^{n\times s_r}$ be any sketching matrices satisfying
the approximate matrix multiplication properties.
We propose to compute $\U$ more efficiently by
\begin{align} \label{def:cur_u_sn}
& \tilde\U \; = \;
\argmin_{\U} \|\S_{C}^T \A \S_{R} - (\S_{C}^T\C) \U (\R\S_{R}) \|_F^2 \nonumber \\
& \qquad = \; \underbrace{(\S_{C}^T\C)^\dag}_{c\times s_c} \underbrace{(\S_{C}^T \A \S_{R} )}_{s_c \times s_r} \underbrace{(\R\S_{R})^\dag}_{s_r\times r} ,
\end{align}
which costs time
\[
\OM (s_r r^2 + s_c c^2 + s_c s_r \cdot \min\{c , r\}) + T_{\textrm{sketch}} ,
\]
where $T_{\textrm{sketch}}$ denotes the time for forming the sketches $\S_{C}^T \A \S_{R}$, $\S_{C}^T\C$, and $\R\S_{R}$.
As for Gaussian projection, SRHT, and count sketch,
$T_{\textrm{sketch}}$ are respectively
$\OM \big(\nnz (\A) \min\{s_c, s_r\}\big)$,
$\OM \big(mn \log ( \min\{s_c, s_r\}) \big)$,
and $\OM \big( \nnz (\A)\big)$.
As for leverage score sampling and uniform sampling,
$T_{\textrm{sketch}}$ are respectively $\OM (m c^2 + n r^2 + s_c s_r)$ and $\OM(s_c s_r)$.
Forming the sketches by column selection is more efficient than by random projection.

The following theorem shows that when $s_c$ and $s_r$ are sufficiently large,
$\tilde\U$ is nearly as good as the best possible $\U$ matrix.
In the theorem, leverage score sampling means that $\S_C$ and $\S_R$ sample columns according to the row leverage scores of $\C$ and $\R^T$, respectively.
The proof is in Appendix~\ref{sec:proof_cur}.

\begin{theorem} \label{thm:cur}
	Let $\A \in \RB^{m\times n}$, $\C \in \RB^{m\times c}$, $\R \in \RB^{r\times n}$ be any fixed matrices with $c \ll n$ and $r \ll m$.
	Let $q = \min\{m, n\}$ and $\tilde{q} = \min\{m/c, n/r \}$.
	The sketching matrices $\S_C \in \RB^{m\times s_c}$ and $\S_R \in \RB^{n\times s_r}$ are described in Table~\ref{tab:curtype}.
	Assume that $\epsilon^{-1} = o(q)$ or $\epsilon^{-1} = o ( \tilde{q} )$, as  shown in the table.
	The matrix $\tilde\U$ is defined in (\ref{def:cur_u_sn}).
	Then the inequality
	\begin{eqnarray*}
		\|\A - \C \tilde\U \R\|_F^2
		& \leq & (1+\epsilon) \min_\U \|\A - \C \U \R\|_F^2
	\end{eqnarray*}
	holds with probability at least $0.7$.
\end{theorem}

\begin{table}[!h]\setlength{\tabcolsep}{0.3pt}
	\caption{
		Leverage score sampling means sampling according to the row leverage scores of $\C$ and
		the column leverage scores of $\R$, respectively.
		For uniform sampling, the parameter $\mu (\C) $ is the row coherence of $\C$
		and $\nu (\R)  $ is the column coherence of $\R$.
	}
	\label{tab:curtype}
	\begin{center}
		\begin{footnotesize}
			\begin{tabular}{c c c c}
				\hline
				{\bf Sketching}            &Order of $s_c$& Order of $s_r $ & ~~Assumption~~ \\
				\hline
				Leverage Score Sampling  &~~~~~$ c \sqrt{q/\epsilon}$~~~
				& ~~~$ r \sqrt{q/\epsilon}$~~~~
				& $\epsilon^{-1} = o ( q )$  \\
				Uniform Sampling  &~~~~~$ \mu (\C) c \sqrt{q/\epsilon}$~~~
				& ~~~$ \nu (\R)r \sqrt{q/\epsilon}$~~~
				& $\epsilon^{-1} = o (q )$  \\
				Gaussian Projection\ &~~$  \sqrt{\frac{m}{ c \epsilon} } \big( c + \log \frac{n}{c} \big)$~~
				&~~$  \sqrt{\frac{n }{r \epsilon} } \big( r + \log \frac{m}{r} \big) $~~
				& $\epsilon^{-1} = o ( \tilde{q} )$ \\
				SRHT            &~~~$  \sqrt{\frac{m}{ c \epsilon} } \big( c + \log \frac{mn}{c} \big)\log ( m)$~~~
				& ~~~$  \sqrt{\frac{n }{r \epsilon} } \big( r + \log \frac{mn}{r} \big) \log (n) $~~~
				& $\epsilon^{-1} = o ( \tilde{q} )$ \\
				Count Sketch    &~~$ c \sqrt{q/\epsilon}$~~
				& ~$ r \sqrt{q/\epsilon}$~~~
				& $\epsilon^{-1} = o (q )$\\
				\hline
			\end{tabular}
		\end{footnotesize}
	\end{center}
\end{table}

As for leverage score sampling, uniform sampling, and count sketch,
the sketch sizes $s_c = \OM (c\sqrt{q/\epsilon})$ and $s_r = \OM (r \sqrt{q/\epsilon})$ suffice,
where $q = \min\{m, n\}$.
As for Gaussian projection and SRHT, much smaller sketch sizes are required:
$s_c = \tilde{\OM} (\sqrt{m c / \epsilon})$ and $s_r = \tilde{\OM} (\sqrt{n r / \epsilon})$ suffice.
However, these random projection methods are inefficient choices in this application and only have theoretical interest.
Only column sampling methods have linear time complexities.
If $\S_C$ and $\S_R$ are leverage score sampling matrices (according to the row leverage scores of $\C$ and $\R^T$, respectively),
it follows from Theorem~\ref{thm:cur} that
$\tilde\U$ with $1+\epsilon$ bound can be computed in time
\[
\OM \big(s_r r^2 + s_c c^2 + s_c s_r \cdot \min\{c , r\} \big) + T_{\textrm{sketch}}
\;\,= \;\,
\OM \big( c r \epsilon^{-1} \cdot \min \{ m ,n \} \cdot \min \{ c , r\} \big) ,
\]
which is linear in $\OM (\min\{m, n\})$.


\subsection{Empirical Comparisons} \label{sec:cur_experiments}

To intuitively demonstrate the effectiveness of our method,
we conduct a simple experiment on a $1920\times 1168$ natural image obtained from the internet.
We first uniformly sample $c=100$ columns to form $\C$ and $r=100$ rows to form $\R$,
and then compute the $\U$ matrix by varying $s_c$ and $s_r$.
We show the image $\tilde\A = \C \U \R$ in Figure~\ref{fig:cur}.

Figure~\ref{fig:cur}(b) is obtained by computing the $\U$ matrix according to \eqref{def:cur_u_opt},
which is the best possible result when $\C$ and $\R$ are fixed.
The $\U$ matrix of Figure~\ref{fig:cur}(c) is computed according to \cite{drineas2008cur}:
\[
\U \; = \; (\PP_R^T \A \PP_C)^\dag ,
\]
where $\PP_C$ and $\PP_R$ are column selection matrices such that $\C = \A \PP_C$ and $\R = \PP_R^T \A$.
This is equivalently to \eqref{def:cur_u_sn} by setting $\S_C = \PP_R$ and $\S_R = \PP_C$.
Obviously, this setting leads to very poor quality.
In Figures~\ref{fig:cur}(c) and (d) the sketching matrices $\S_C$ and $\S_R$ are uniform sampling matrices.
The figures show that when $s_c$ and $s_r$ are moderately greater than $r$ and $c$, respectively,
the approximation quality is significantly improved.
Especially, when $s_c = 4r$ and $s_r = 4c$, the approximation quality is nearly as good as using the optimal $\U$ matrix defined in \eqref{def:cur_u_opt}.

%

\begin{figure}[!ht]
	\begin{center}
		\centering
		\includegraphics[width=0.97\textwidth]{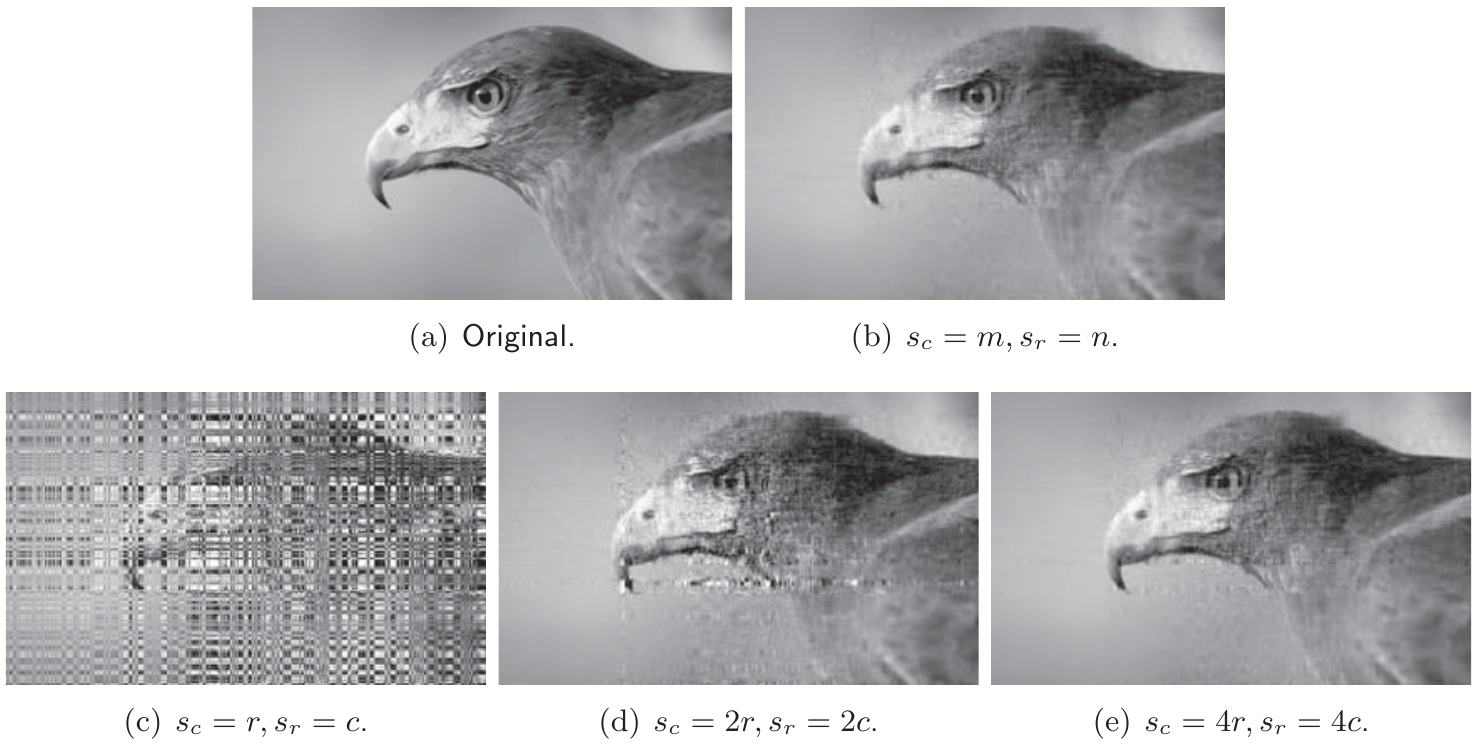}
	\end{center}
	\caption{(a): the original $1920\times 1168$ image. (b) to (e): CUR decomposition with $c=r=100$ and different settings of $s_c$ and $s_r$.}
	\label{fig:cur}
\end{figure}


\subsection{Discussions} \label{sec:cur_discussions}

We note that we are not the first to use row and column sampling to solve the CUR problem more efficiently,
though we are the first to provide rigorous error analysis.
Previous work has exploited similar ideas as heuristics to speed up computation and to avoid visiting every entry of $\A$.
For example, the MEKA method \citep{si2014memory} partitions the kernel matrix $\K$ into $b^2$ blocks
$\K^{(i,j)}$ ($i = 1, \cdots, b$ and $j=1, \cdots , b$),
and requires solving
\[
\LL^{(i,j)} \; = \; \argmin_{\LL} \big\| \W^{(i)} \LL {\W^{(j)}}^T - \K^{(i,j)}\big\|_F^2
\]
for all $i \in [b]$, $j \in [b]$, and $i \neq j$.
Since $\W^{(i)}$ and $\W^{(j)}$ have much more rows than columns,
\citet{si2014memory} proposed to approximately solve the linear system by uniformly sampling
rows from $\W^{(i)}$ and $\K^{(i,j)}$ and columns from $(\W^{(j)})^T$ and $\K^{(i,j)}$,
and they noted that this heuristic works pretty well.
The basic ideas of our fast CUR and their MEKA are the same;
their experiments demonstrate the effectiveness and efficiency of this approach,
and our analysis answers why this approach is correct.
This also implies that our algorithms and analysis may have broad applications and impacts beyond the CUR decomposition and SPSD matrix approximation.

\begin{table}[t]\setlength{\tabcolsep}{0.3pt}
	\caption{A summary of the datasets for kernel approximation.}
	\label{tab:datasets}
	\begin{center}
		\begin{small}
			\begin{tabular}{l c c c c c c  }
				\hline
				{\bf Dataset}  &~~Letters~~&~~PenDigit~~&~~Cpusmall~~&~~Mushrooms~~&~~WineQuality~~\\
				\hline
				{\bf \#Instance}  &$15,000$ & 10,992  & $8,192$ &  $8,124$  &  $4,898$    \\
				{\bf \#Attribute} & $16$    &  16     & $12$    &   $112$   &  $12$       \\
				$\sigma$ (when $\eta = 0.90)$ & $0.400$ & $0.101$ & $0.075$ & $1.141$ & $0.314$ \\
				$\sigma$ (when $\eta = 0.99)$ & $0.590$ & $0.178$ & $0.180$ & $1.960$ & $0.486$ \\
				\hline
			\end{tabular}
		\end{small}
	\end{center}
\end{table}

\begin{figure}[!ht]
	\begin{center}
		\centering
		\includegraphics[width=0.99\textwidth]{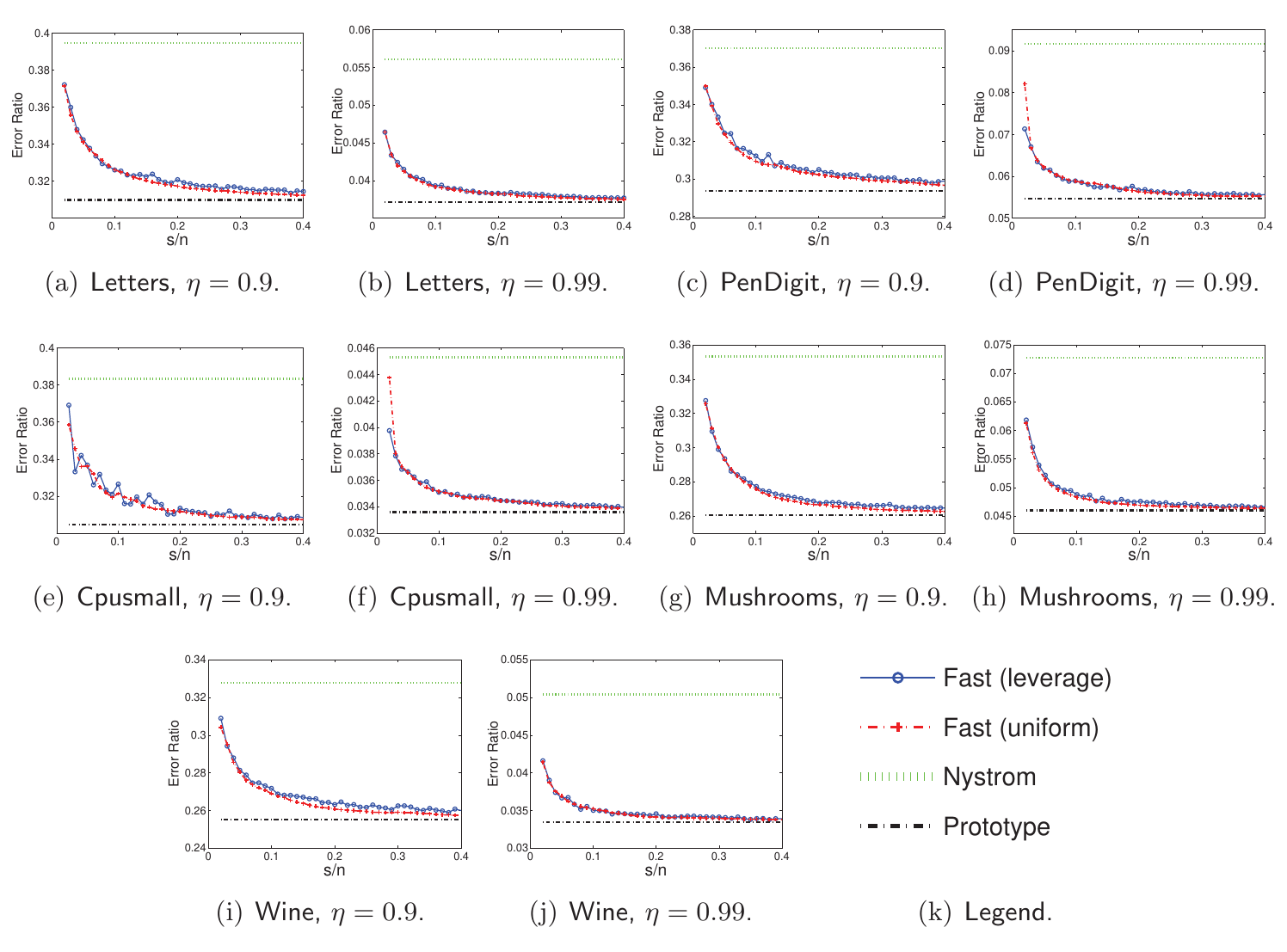}
	\end{center}
	\caption{The plot of $\frac{s}{n}$ against the approximation error $\|\K - \C \U \C^T\|_F^2 / \|\K\|_F^2$,
		where $\C$ contains $c = \lceil n / 100 \rceil$ column of $\K \in \RB^{n\times n}$ selected by uniform sampling.}
	\label{fig:kernel_approx_uniform}
\end{figure}

\begin{figure}[!ht]
	\begin{center}
		\centering
		\includegraphics[width=0.99\textwidth]{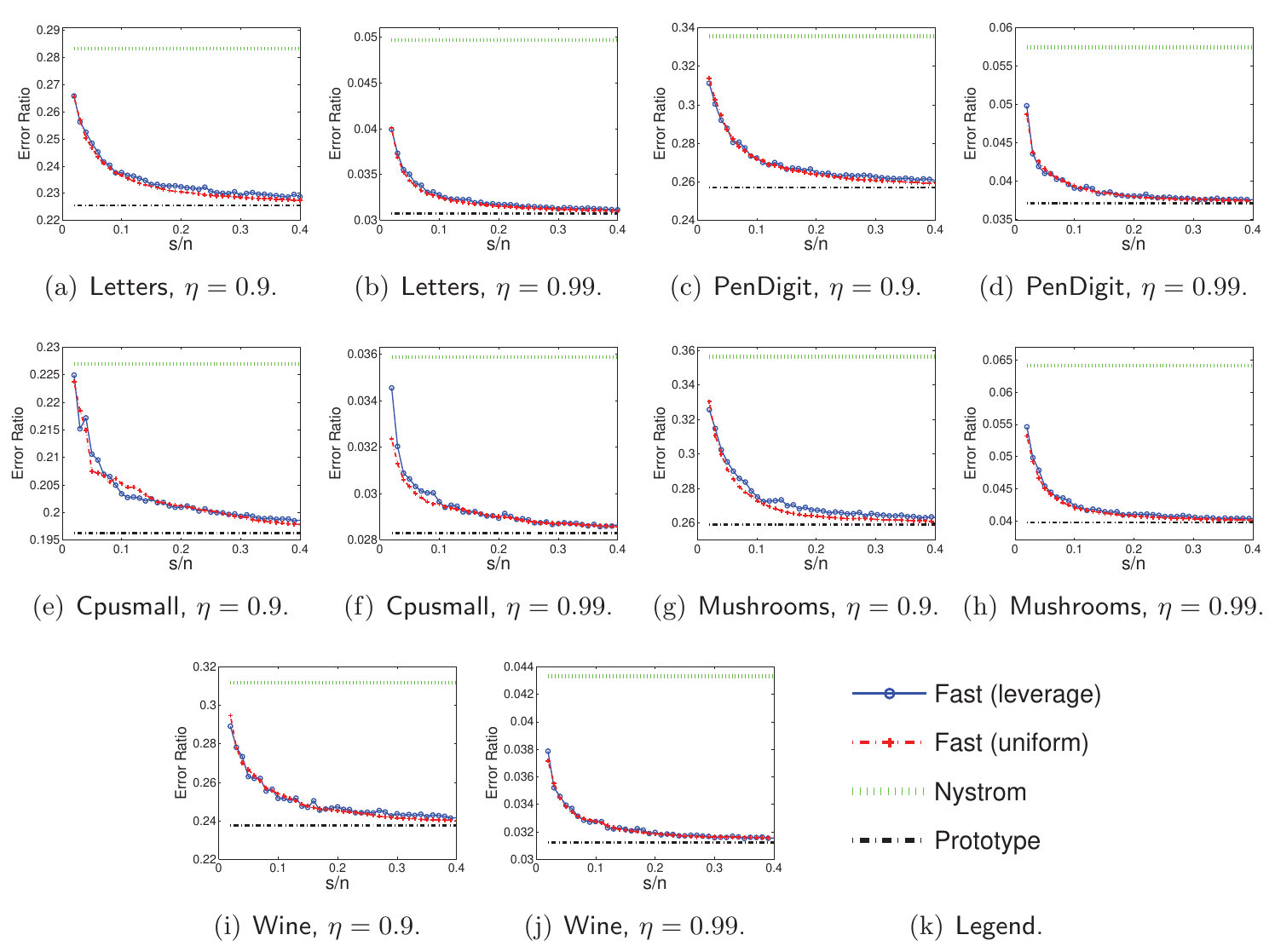}
	\end{center}
	\caption{The plot of $\frac{s}{n}$ against the approximation error $\|\K - \C \U \C^T\|_F^2 / \|\K\|_F^2$,
		where $\C$ contains $c = \lceil n / 100 \rceil$ column of $\K \in \RB^{n\times n}$ selected by the uniform+adaptive$^2$ sampling algorithm~\citep{wang2014modified}.}
	\label{fig:kernel_approx_adaptive}
\end{figure}

\section{Experiments} \label{sec:experiments}

In this section we conduct several sets of illustrative experiments to show the effect of the $\U$ matrix.
We compare the three methods with different settings of $c$ and $s$.
We do not compare with other kernel approximation methods for the
reasons stated in Section~\ref{sec:related_nystrom_less}.


\subsection{Setup}

Let $\X=[\x_1, \ldots, \x_n]$ be the $d\times n$ data matrix, and
$\K$ be the RBF kernel matrix with each entry computed by $K_{i j} = \exp \big(- \frac{\|\x_i - \x_j\|_2^2 }{2\sigma^2} \big)$ where $\sigma$ is the scaling parameter.

When comparing the kernel approximation error $\| \K - \C \U \C^T \|_F^2$,
we set the scaling parameter $\sigma$ in the following way.
We let $k = \lceil n/100 \rceil$ and define
\begin{small}
	\[
	\eta \; = \; \frac{\|\K_k\|_F^2}{\|\K\|_F^2}
	\; = \; \frac{\sum_{i=1}^{k}\sigma_i^2 (\K)}{\sum_{i=1}^{n}\sigma_i^2 (\K)} ,
	\]%
\end{small}%
which indicate the importance of the top one percent singular values of $\K$.
In general $\eta$ grows with $\sigma$.
We set $\sigma$ such that $\eta = 0.9$ or $0.99$.

All the methods are implemented in MATLAB and run on a laptop with Intel i5 2.5GHz CUP and 8GB RAM.
To compare the running time, we set MATLAB in the single thread mode.


\subsection{Kernel Approximation Accuracy}

We conduct experiments on several datasets available at the LIBSVM site.
The datasets are summarized in Table~\ref{tab:datasets}.
In this set of experiments, we study the effect of the $\U$ matrices.
We use two methods to form $\C \in \RB^{n\times c}$: uniform sampling and the uniform+adaptive$^2$ sampling \citep{wang2014modified};
we fix $c = \lceil n/100 \rceil$.
For our fast model, we use two kinds of sketching matrices $\S \in \RB^{n\times s}$: uniform sampling and leverage score sampling;
we vary $s$ from $2c$ to $40c$.
We plot $\frac{s}{n}$ against the approximation error $\|\K - \C \U \C^T\|_F^2 / \|\K\|_F^2$ in Figures~\ref{fig:kernel_approx_uniform} and \ref{fig:kernel_approx_adaptive}.
The \nystrom method and the prototype model are included for comparison.

Figures~\ref{fig:kernel_approx_uniform} and \ref{fig:kernel_approx_adaptive} show that
the fast SPSD matrix approximation model is significantly better than the \nystrom method
when $s$ is slightly larger than $c$, e.g., $s = 2c$.
Recall that the prototype model is a special case of the fast model where $s = n$.
We can see that the fast model is nearly as accurate as the prototype model
when $s$ is far smaller than $n$, e.g., $s = 0.2 n$.

The results also show that using uniform sampling and leverage score sampling to generate $\S$ does not make much difference.
Thus, in practice, one can simply compute $\S$ by uniform sampling.

By comparing the results in Figures~\ref{fig:kernel_approx_uniform} and \ref{fig:kernel_approx_adaptive},
we can see that computing $\C$ by uniform+adaptive$^2$ sampling is substantially better than uniform sampling.
However, adaptive sampling requires the full observation of $\K$;
thus with uniform+adaptive$^2$ sampling, our fast model does not have much advantage over the prototype model in terms of time efficiency.
Our  main focus of this work is the $\U$ matrix,
so in the rest of the experiments we simply use uniform sampling to compute $\C$.

\begin{figure}[!ht]
	\begin{center}
		\centering
		\includegraphics[width=0.99\textwidth]{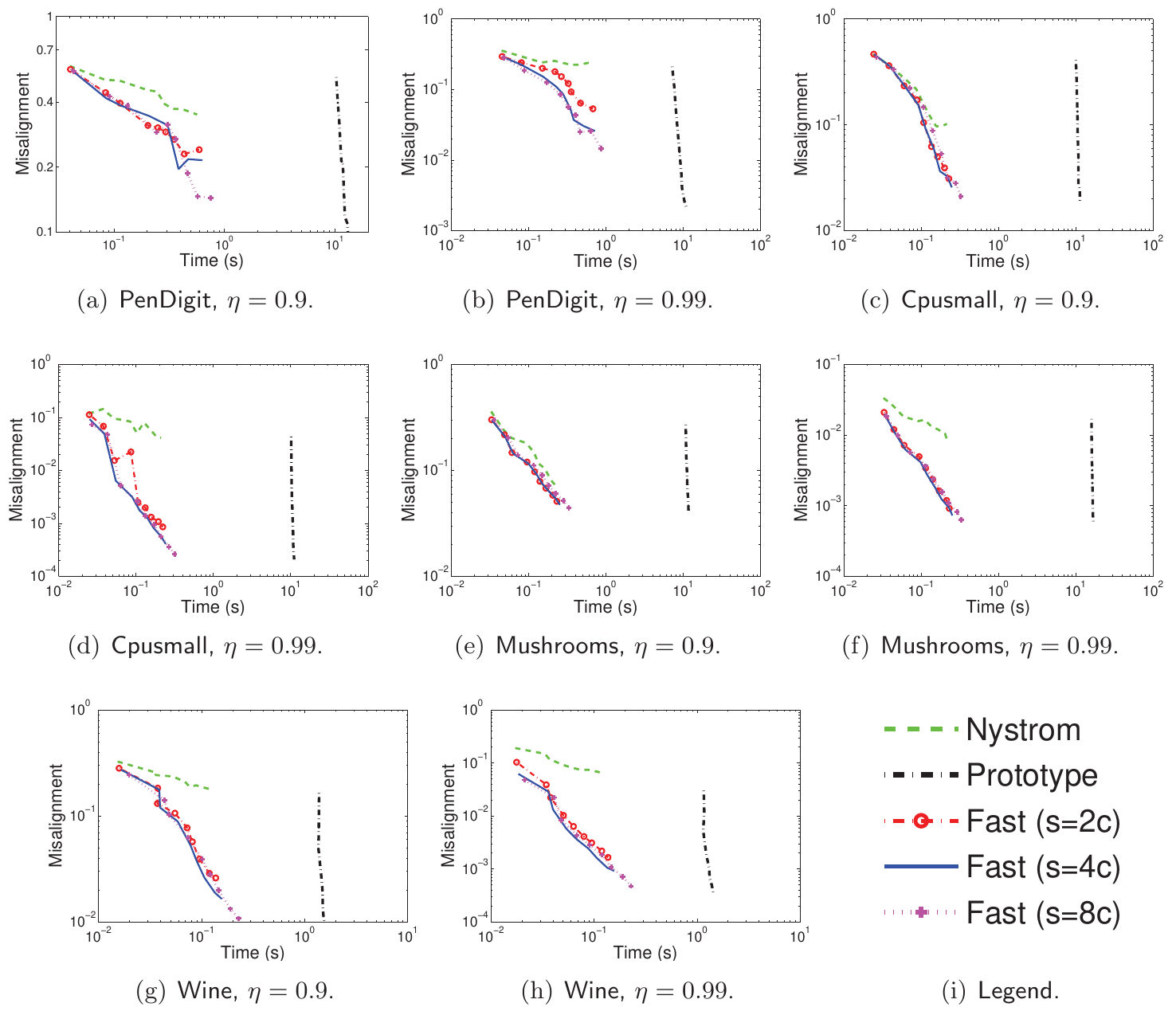}
	\end{center}
	\vspace{-3mm}
	\caption{The plot of (log-scale) elapsed time against the (log-scale) misalignment defined in (\ref{eq:misalignment}).}
	\label{fig:time_vs_misalignment_uniform}
\end{figure}

\begin{figure}[!ht]
	\begin{center}
		\centering
		\includegraphics[width=0.99\textwidth]{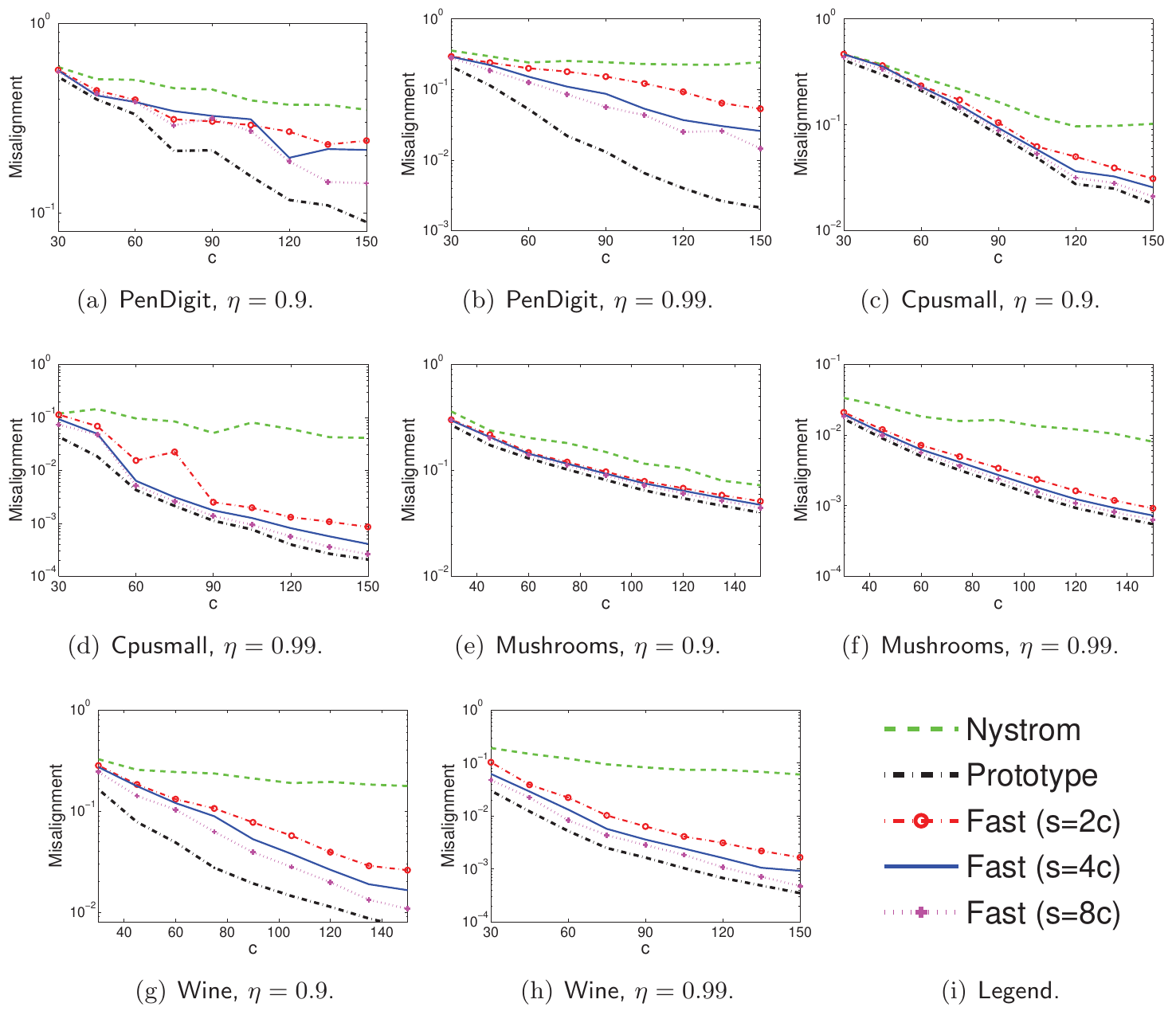}
	\end{center}
	\vspace{-3mm}
	\caption{The plot of $c$ against the (log-scale) misalignment defined in (\ref{eq:misalignment}).}
	\label{fig:memory_vs_misalignment_uniform}
\end{figure}


\subsection{Approximate Kernel Principal Component Analysis}

We apply the three methods to approximately compute kernel principal component analysis (KPCA),
and contrast with the exact solution.
The experiment setting follows \citet{zhang2010clustered}.
We fix $k$ and vary $c$.
For our fast model, we set $s=2c$, $4c$, or $8c$.
Since computing $\S$ by uniform sampling or leverage score sampling yields the same empirical performance,
we use only uniform sampling.
Let $\C \U \C^T$ be the low-rank approximation formed by the three methods.
Let ${\tilde{\V}} {\tilde{\Lam}}{\tilde{\V}}^T$ be the $k$-eigenvalue decomposition of $\C \U \C^T$.

\subsubsection{Quality of the Approximate Eigenvectors}

Let $\U_{\K, k} \in \RB^{n\times k}$ contain the top $k$ eigenvectors of $\K$.
In the first set of experiments,
we measure the distance between $\U_{\K, k}$ and the approximate eigenvectors $\tilde{\V}$ by
\begin{eqnarray} \label{eq:misalignment}
\textrm{Misalignment}
& = &
\frac{1}{k} \big\| \U_{\K, k} - \tilde{\V}\tilde{\V}^T \U_{\K, k} \big\|_F^2
\; \; (\in \; [0, 1]).
\end{eqnarray}
Small misalignment indicates high approximation quality.
We fix $k = 3$.

We conduct experiments on the datasets summarized in Table~\ref{tab:datasets}.
We record the elapsed time of the entire procedure---computing (part of) the kernel matrix,
computing $\C$ and $\U$ by the kernel approximation methods,
computing the $k$-eigenvalue decomposition of $\C \U \C^T$.
We plot the elapsed time against the misalignment defined in Figure~\ref{fig:time_vs_misalignment_uniform}.
Results on the Letters dataset are not reported because the exact $k$-eigenvalue decomposition on MATLAB ran out of memory,
making it impossible to calculate the misalignment.

At the end of Section~\ref{sec:related_nystrom_most} we have mentioned the
importance of memory cost of the kernel approximation methods
and that all three compared methods  cost $\OM (n c + n d)$ memory.
Since $n$ and $d$ are fixed,
we plot $c$ against the misalignment in Figure~\ref{fig:memory_vs_misalignment_uniform} to show the memory efficiency.

The results show that using the same amount of time or memory,
the misalignment incurred by the \nystrom method is usually tens of times higher than our fast model.
The experiment also shows that with fixed $c$, the fast model is nearly as accuracy as the prototype model when $s=8c \ll n$.

\begin{table}[t]
	\caption{A summary of the datasets for clustering and classification.}
	\label{tab:dataset_classification}
	\begin{center}
		\begin{footnotesize}
			\begin{tabular}{l c c c c c  c}
				\hline
				{\bf Dataset}               & ~~MNIST~~~&~~Pendigit~& ~~USPS~~&~Mushrooms~&~Gisette~&~~DNA~\\
				\hline
				{\bf \#Instance}            &  $60,000$ &  $10,992$ & $9,298$ &  $8,124$  & $7,000$ & $2,000$   \\
				{\bf \#Attribute}           &   $780$   &   $16$    & $256$   &   $112$   & $5,000$ &  $180$    \\
				{\bf \#Class}               &   $10$    &   $10$    & $10$    &    $2$    & $2$     &   $3$     \\
				{\bf Scaling Parameter} $\sigma$&    $10$   &    $0.7$  & $15$    &    $3$    &  $50$   &   $4$     \\
				\hline
			\end{tabular}
		\end{footnotesize}
	\end{center}
\end{table}

\begin{figure}[!ht]
	\begin{center}
		\centering
		\includegraphics[width=0.99\textwidth]{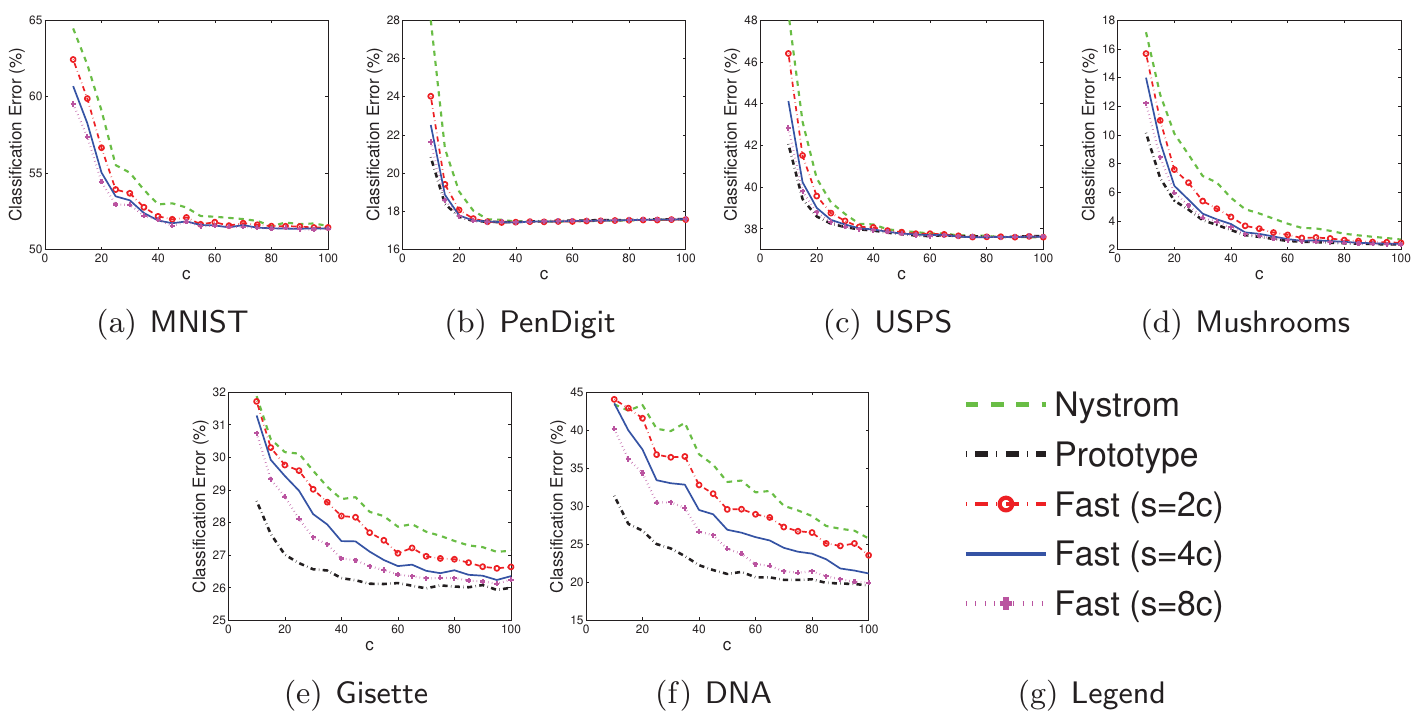}
	\end{center}
	\caption{The plot of $c$ against the classification error. Here $k=3$.}
	\label{fig:knn_memory_k3}
\end{figure}

\begin{figure}[!ht]
	\begin{center}
		\centering
		\includegraphics[width=0.99\textwidth]{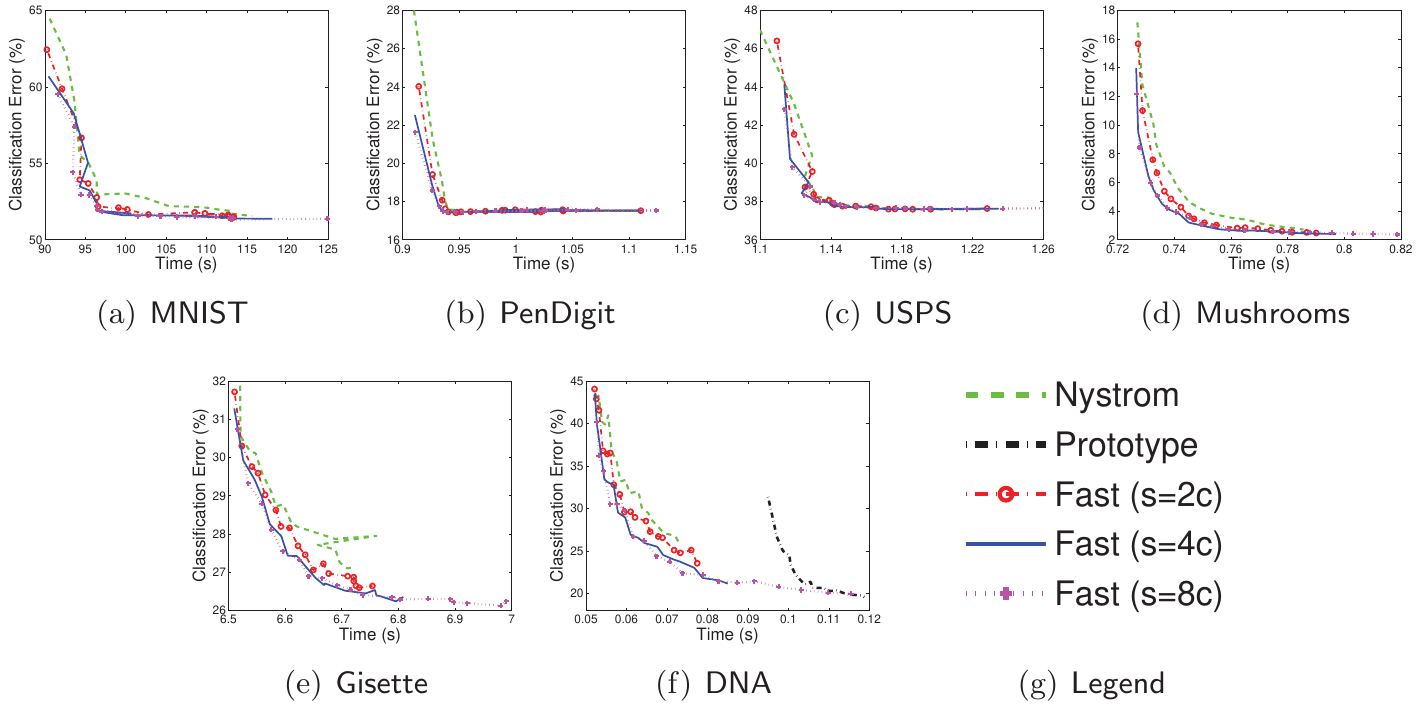}
	\end{center}
	\caption{The plot of elapsed time against the classification error. Here $k=3$.}
	\label{fig:knn_time_k3}
\end{figure}

\begin{figure}[!ht]
	\begin{center}
		\centering
		\includegraphics[width=0.99\textwidth]{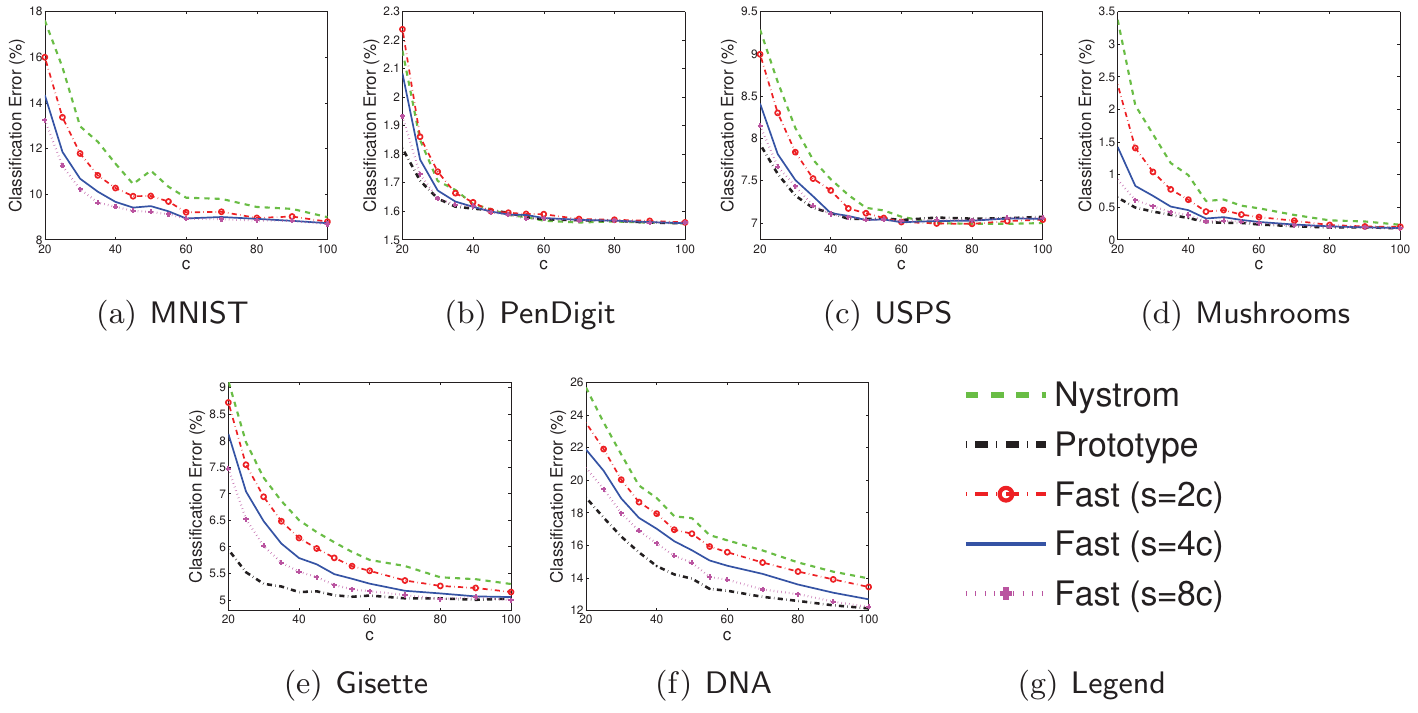}
	\end{center}
	\caption{The plot of $c$ against the classification error. Here $k=10$.}
	\label{fig:knn_memory_k10}
\end{figure}

\begin{figure}[!ht]
	\begin{center}
		\centering
		\includegraphics[width=0.99\textwidth]{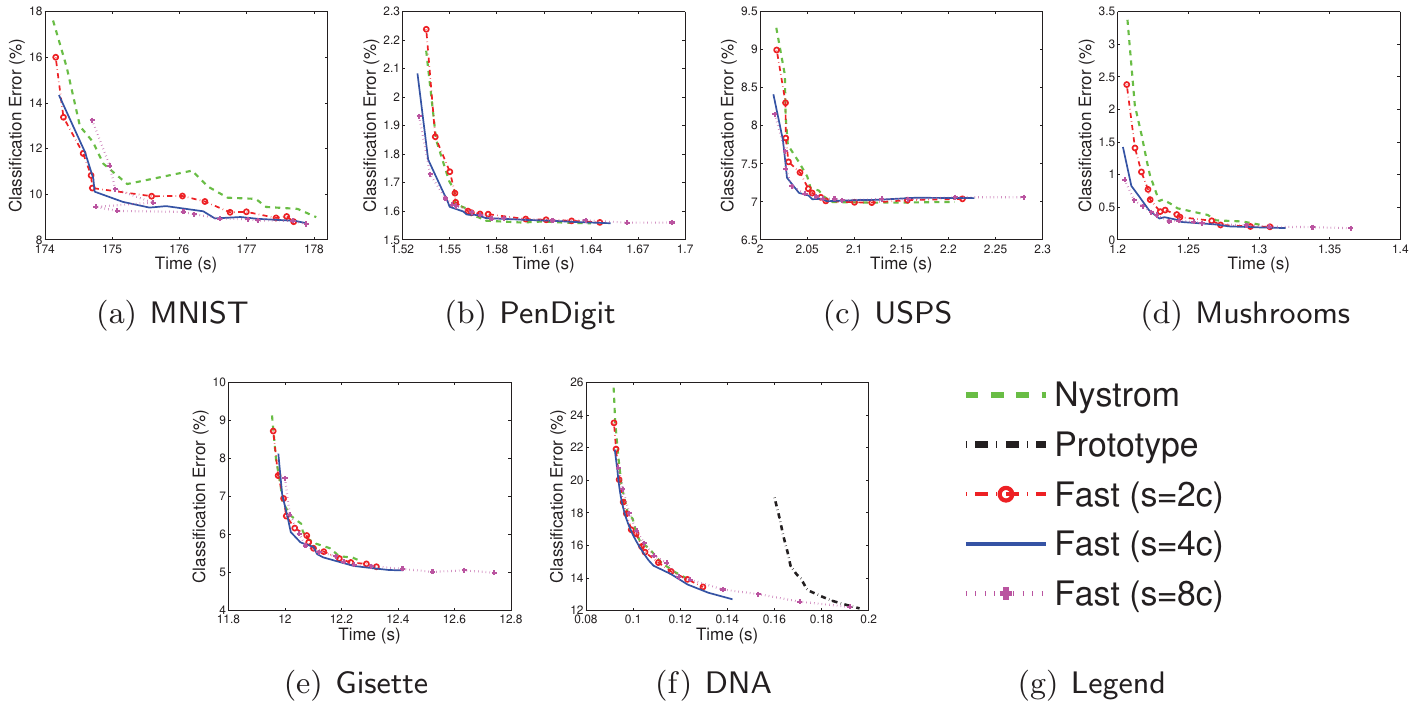}
	\end{center}
	\caption{The plot of elapsed time against the classification error. Here $k=10$.}
	\label{fig:knn_time_k10}
\end{figure}

\subsubsection{Quality of the Generalization}

In the second set of experiments,
we test the generalization performance of the kernel approximation methods on classification tasks.
The classification datasets are described in Table~\ref{tab:dataset_classification}.
For each dataset, we randomly sample $n_1 = 50\% n$ data points for training and the rest $50\% n$ for test.
In this set of experiments, we set $k = 3$ and $k = 10$.

We let $\K\in \RB^{n_1 \times n_1}$ be the RBF kernel matrix of the training data
and $\k (\x) \in \RB^{n_1}$ be defined by $[\k (\x)]_i = \exp \big(-\frac{\|\x - \x_i \|_2^2}{2\sigma^2} \big)$,
where $\x_i$ is the $i$-th training data point.
In the training step, we approximately compute the top $k$ eigenvalues and eigenvectors,
and denote $\tilde\Lam \in \RB^{k\times k}$ and $\tilde\V \in \RB^{n_1\times k}$.
The feature vector (extracted by KPCA) of the $i$-th training data point is the $i$-th column of ${\tilde\Lam}^{0.5} {\tilde\V}^T $.
In the test step, the feature vector of test data $\x$ is ${\tilde\Lam}^{-0.5} {\tilde\V}^T \k (\x)$.
Then we put the training labels and training and test features into the MATLAB K-nearest-neighbor classifier
\texttt{knnclassify} to classify the test data.
We fix the number of nearest neighbors to be $10$.
The scaling parameters of each dataset are listed in Table~\ref{tab:dataset_classification}.
Since the kernel approximation methods are randomized,
we repeat the training and test procedure $20$ times and record the average elapsed time and average classification error.

We plot $c$ against the classification error in Figures~\ref{fig:knn_memory_k3} and \ref{fig:knn_memory_k10},
and plot the elapsed time (excluding the time cost of KNN) against the classification error in Figures~\ref{fig:knn_time_k3} and \ref{fig:knn_time_k10}.
Using the same amount of memory,
the fast model is significantly better than the \nystrom method,
especially when $c$ is small.
Using the same amount of time,
the fast model outperforms the \nystrom method by one to two percent of classification error in many cases,
and it is at least as good as the \nystrom method in the rest cases.
This set of experiments also indicate that the fast model with $s = 4c$ or $8c$ has the best empirical performance.

\begin{figure}[!ht]
	\begin{center}
		\centering
		\includegraphics[width=0.99\textwidth]{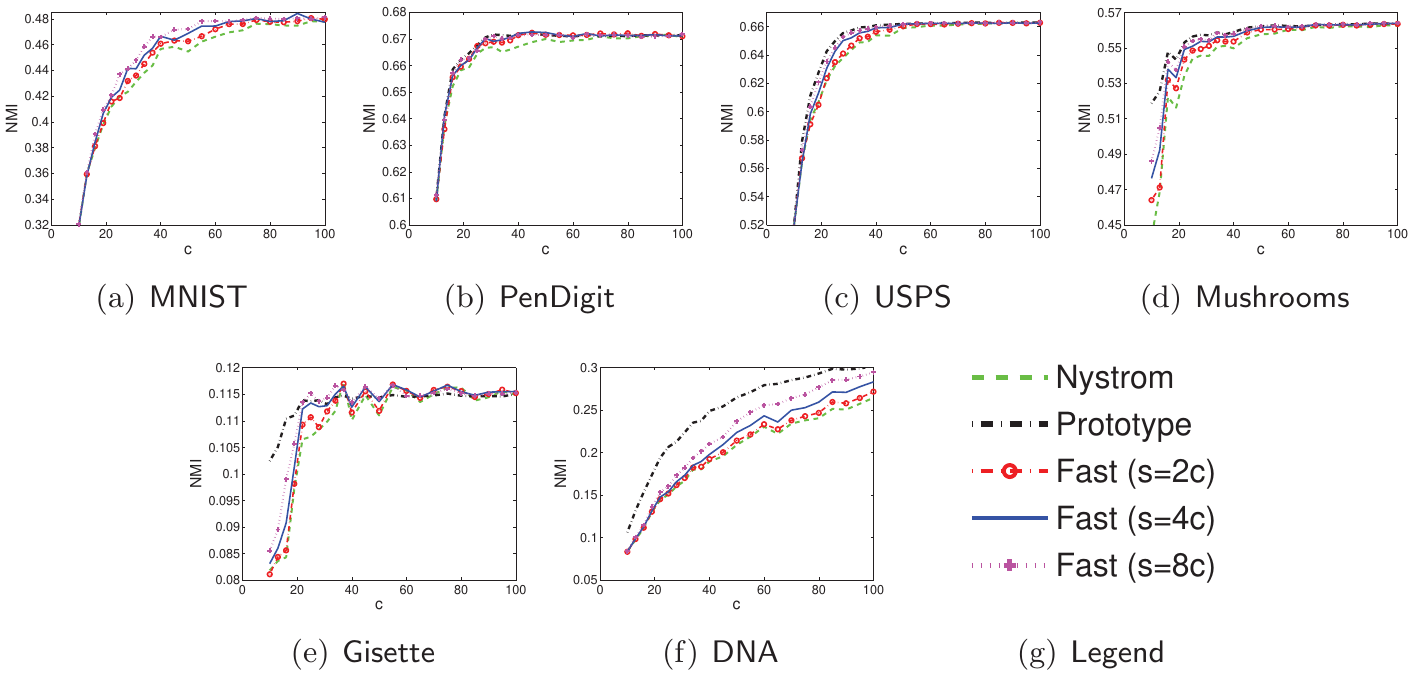}
	\end{center}
	\caption{The plot of $c$ against NMI.}
	\label{fig:sc_memory}
\end{figure}

\begin{figure}[!ht]
	\begin{center}
		\centering
		\includegraphics[width=0.99\textwidth]{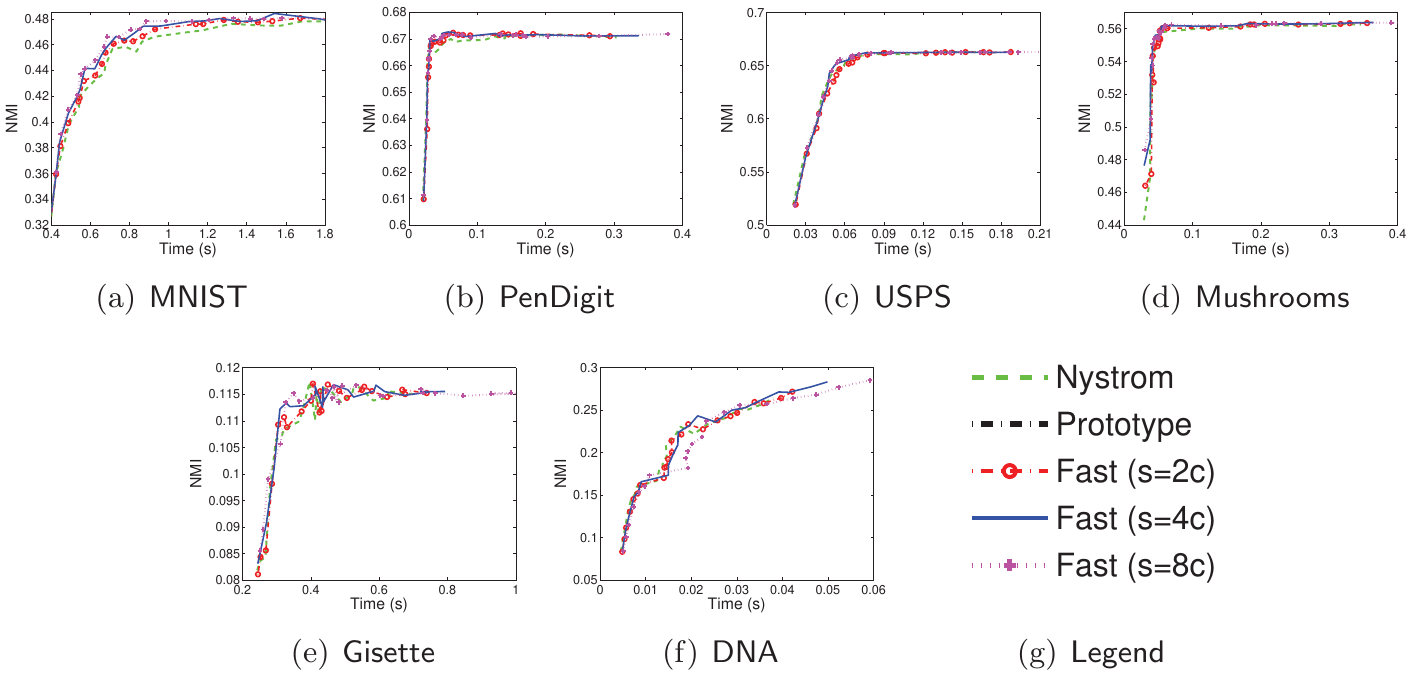}
	\end{center}
	\caption{The plot of elapsed time against NMI.}
	\label{fig:sc_time}
\end{figure}


\subsection{Approximate Spectral Clustering}

Following the work of \cite{fowlkes2004spectral},
we evaluate the performance of the kernel approximation methods on the spectral clustering task.
We conduct experiments on the datasets summarized in Table~\ref{tab:dataset_classification}.

We describe the approximate spectral clustering in the following.
The target is to cluster $n$ data points into $k$ classes.
We use the RBF kernel matrix $\K$ as the weigh matrix and let
$\C \U \C^T \approx \K$ be the low-rank approximation.
The degree matrix $\D = \diag (\d)$ is a diagonal matrix with $\d = \C \U \C^T \1_n$,
and the normalized graph Laplacian is $\LL = \I_n - \D^{-1/2} (\C \U \C^T ) \D^{-1/2}$.
The bottom $k$ eigenvectors of $\LL$ are the top $k$ eigenvectors of
\[
\underbrace{(\D^{-1/2}\C)}_{n\times c} \underbrace{\U}_{c\times c} \underbrace{(\D^{-1/2}\C)^T}_{c\times n} ,
\]
which can be efficiently computed according to Appendix~\ref{sec:approx_eigen_linear}.
We denote the top $k$ eigenvectors by $\tilde\V \in \RB^{n\times k}$.
We normalize the rows of $\tilde\V$ and
take the normalized rows of $\tilde\V$ as the input of the $k$-means clustering.
Since the matrix approximation methods are randomized, we repeat this procedure $20$ times
and record the average elapsed time and the average normalized mutual information (NMI)\footnote{NMI
	is a standard metric of clustering.
	NMI is between 0 and 1.
	Big NMI indicates good clustering performance.}
of clustering.

We plot $c$ against NMI in Figure~\ref{fig:sc_memory}
and the elapsed time (excluding the time cost of $k$-means) against NMI in Figure~\ref{fig:sc_time}.
Figure~\ref{fig:sc_memory} shows that using the same amount of memory,
the performance of the fast model is better than the \nystrom method.
Using the same amount of time,
the fast model and the \nystrom method have almost the same performance,
and they are both better than the prototype model.


\section{Concluding Remarks}

In this paper we have studied  the fast SPSD matrix approximation model for approximating large-scale SPSD matrix.
We have shown that our fast model potentially costs time linear in $n$,
while it is nearly as accurate as the best possible approximation.
The fast model is theoretically better than the \nystrom method and the prototype model because the latter two methods cost time quadratic in $n$ to attain the same theoretical guarantee.
Experiments show that our fast model is nearly as accurate as the prototype model and nearly as efficient as the \nystrom method.

The technique of the fast model can be straightforwardly applied to speed up the CUR matrix decomposition,
and theoretical analysis shows that the accuracy is almost unaffected.
In this way, for any $m\times n$ large-scale matrix,
the time cost of computing the $\U$ matrix drops from $\OM (m n)$ to $\OM (\min\{m, n\})$.



\acks{%
	We thank the anonymous reviewer for their helpful feedbacks.
	Shusen Wang acknowledges the support of Cray Inc.,
	the Defense Advanced Research Projects Agency, the National Science Foundation, and the Baidu Scholarship.
	Zhihua Zhang acknowledges the support of National Natural Science Foundation of China (No.\ 61572017)
	and MSRA Collaborative Research Grant awards.
        Tong Zhang acknowledges NSF IIS-1250985, NSF IIS-1407939,
        and NIH R01AI116744.
}


\appendix

\section{Approximately Solving the Eigenvalue Decomposition and Matrix Inversion} \label{sec:approx_eigen_linear}

In this section we show how to use the SPSD matrix approximation methods to speed up eigenvalue decomposition and linear system.
The two lemmas are well known results.
We show them here for the sake of self-containing.

\begin{lemma}[Approximate Eigenvalue Decomposition]
	Given $\C \in \RB^{n\times c}$ and $\U \in \RB^{c\times c}$.
	Then the eigenvalue decomposition of $\tilde\K = \C \U \C^T$ can be computed in time $\OM (n c^2)$.
\end{lemma}

\begin{proof}
	It cost $\OM (n c^2)$ time to compute the SVD
	\[
	\C = \underbrace{\U_\C}_{n\times c} \underbrace{\Si_\C}_{c\times c} \underbrace{\V_\C^T}_{c\times c}
	\]
	and $\OM (c^3)$ time to compute $\Z = (\Si_\C \V_\C^T) \U (\Si_\C \V_\C^T)^T \in \RB^{c \times c}$.
	It costs $\OM (c^3)$ time to compute the eigenvalue decomposition $\Z = \V_\Z \Lam_\Z \V_\Z^T$.
	Combining the results above, we obtain
	\begin{eqnarray*}
		\C \U \C^T
		& = & (\U_\C \Si_\C \V_\C^T) \U (\U_\C \Si_\C \V_\C^T)^T \\
		& = & \U_\C \Z \U_\C^T
		\; = \; (\U_\C \V_\Z) \Lam_\Z (\U_\C \V_\Z)^T .
	\end{eqnarray*}
	It then cost time $\OM (n c^2)$ to compute the matrix product $\U_\C \V_\Z$.
	Since $(\U_\C \V_\Z)$ has orthonormal columns and $\Lam_\Z$ is diagonal matrix,
	the eigenvalue decomposition of $\C \U \C^T $ is solved.
	The total time cost is $\OM(n c^2) + \OM(c^3) = \OM(n c^2)$.
\end{proof}

\begin{lemma}[Approximately Solving Matrix Inversion]
	Given $\C \in \RB^{n\times c}$, SPDS matrix $\U \in \RB^{c\times c}$, vector $\y \in \RB^n$,
	and arbitrary positive real number $\alpha$.
	Then it costs time $\OM (n c^2)$ to solve the $n\times n$ linear system
	$(\C \U \C^T + \alpha \I_n) \w = \y$ to obtain $\w\in \RB^n$.
	
	In addition, if the SVD of $\C$ is given, then it takes only $\OM (c^3 + nc)$ time to solve the linear system.
\end{lemma}


\begin{proof}
	Since the matrix $(\C \U \C^T + \alpha \I_n)$ is nonsingular when $\alpha > 0$ and $\U$ is SPSD,
	the solution is $\w^\star = (\C \U \C^T + \alpha \I_n)^{-1} \y$.
	Instead of directly computing the matrix inversion,
	we can expand the matrix inversion by the Sherman-Morrison-Woodbury matrix identity and obtain
		\begin{eqnarray*}
			(\C \U \C^T + \alpha \I_n)^{-1}
			= \alpha^{-1} \I_n -  \alpha^{-1} \C (\alpha \U^{-1} + \C^T \C)^{-1} \C^T .
		\end{eqnarray*}
	Thus the solution to the linear system is
	\begin{small}
		\begin{eqnarray*}
			\w^\star
			= \alpha^{-1} \y -  \alpha^{-1} \underbrace{\C}_{n\times c} \underbrace{(\alpha \U^{-1} + \C^T \C)^{-1}}_{c\times c} \underbrace{\C^T}_{c\times n} \y .
		\end{eqnarray*}
	\end{small}%

	Suppose we are given only $\C$ and $\U$.
	The matrix multiplication $\C^T \C$ costs time $\OM(n c^2)$,
	the matrix inversions cost time $\OM (c^3)$,
	and multiplying matrix with vector costs time $\OM (n c)$.
	Thus the total time cost is $\OM (n c^2) + \OM (c^3) + \OM (n c) = \OM (n c^2)$.
	
	Suppose we are given $\U$ and the SVD $\C = \U_\C \Si_\C \V_\C^T$.
	The matrix product
	\[
	\C^T \C = \V_\C \Si_\C \U_\C^T \U_\C \Si_\C \V_\C = \V_\C \Si_\C^2 \V_\C
	\]%
	can be computed in time $\OM (c^3)$.
	Thus the total time cost is merely $\OM (c^3 + nc)$.
\end{proof}


\section{Proof of Theorem~\ref{thm:linear}} \label{sec:thm_linear}

The prototype model trivially satisfies requirement R1 with $\epsilon = 0$.
However, it violates requirement R2 because
computing the $\U$ matrix by solving $\min_{\U} \|\K - \C \U \C^T\|_F^2$ costs time $\OM(n^2 c)$.

For the Nystr\"om method,
we provide such an adversarial case that assumptions A1 and A2 can both be satisfied and that requirements R1 and R2 cannot hold simultaneously.
The adversarial case is the block diagonal matrix
\[
\K = \diag (\underbrace{\B, \cdots , \B}_{k \textrm{ blocks}}),
\]
where
\[
\B = (1-a) \I_p + a \1_p \1_p^T, \qquad a < 1 , \qquad \textrm{and } p = \frac{n}{k},
\]
and let $a \rightarrow 1$.
\citet{wang2014modified} showed that sampling $c = 3 k \gamma^{-1} \big( 1+o(1) \big)$ columns of $\K$ to form $\C$ makes
assumptions A1 and A2 in Question~\ref{question:linear} be satisfied.
This indicates that $\C$ is a good sketch of $\K$.
The problem is caused by the way the $\Unys$ matrix is computed.
\citet[Theorem 12]{wang2013improving} showed that to make requirement R1 in Question~\ref{question:linear} satisfied,
$c$ must be greater than $\Omega (\sqrt{n k /(\epsilon + \gamma)})$.
Thus it takes time $\OM (n c^2)  = \Omega (n^2 k / (\epsilon + \gamma) )$ to compute
the rank-$k$ eigenvalue decomposition of $\C \Unys \C^T$
or the linear system $ (\C \Unys \C^T + \alpha \I_n ) \w = \y$.
Thus, requirement R2 is violated.


\section{Proof of Lemma~\ref{lem:product}} \label{sec:proof:lem:product}

Lemma~\ref{lem:product} is a simplified version of Lemma~\ref{lem:product2}.
We prove Lemma~\ref{lem:product2} in the subsequent subsections.
In the lemma, leverage score sampling means that the sampling probabilities are proportional to the row leverage scores of $\U \in \RB^{n\times k}$.
For uniform sampling, $\mu (\U )$ is the row coherence of $\U$.

\begin{lemma} \label{lem:product2}
	Let $\U \in \RB^{n\times k}$ be any fixed matrix with orthonormal columns and
	$\B \in \RB^{n\times d}$ be any fixed matrix.
	Let $\S \in \RB^{n\times s}$ be any sketching matrix described in Table~\ref{tab:product2}.
	Then
	\begin{eqnarray*}
		\PB \Big\{ \big\|\U^T \S \S^T \U - \I_k \big\|_2 \geq \eta \Big\}
		\; \leq \; \delta_1
		& \qquad \textrm{(Property 1)}, \\
		\PB \Big\{\big\| \U^T \B - \U^T \S \S^T \B \big\|_F^2 \geq {\epsilon} \|\B\|_F^2 \Big\}
		\; \leq \; \delta_2
		& \qquad \textrm{(Property 2)} , \\
		\PB \Big\{\big\| \U^T \B - \U^T \S \S^T \B \big\|_2^2 \geq {\epsilon'} \|\B\|_2^2 + \frac{\epsilon'}{k} \|\B\|_F^2 \Big\}
		\; \leq \; \delta_3
		& \qquad \textrm{(Property 3)}.
	\end{eqnarray*}
\end{lemma}

\begin{table}[!h]\setlength{\tabcolsep}{0.3pt}
	\caption{The sketch size $s$ for satisfying the three properties.
		For SRHT, we define $\lambda = \big(1+ \sqrt{8 k^{-1} \log (100n)} \big)^2$ and
		$\lambda' = \Big( 1 + \sqrt{4k^{-1} \log \frac{n {d}}{k \delta_1}} \Big)^2$.
	}
	\label{tab:product2}
	\begin{center}
		\begin{small}
			\begin{tabular}{c c c c}
				\hline
				{\bf Sketching}             &~~~Property 1~~~&~~~Property 2~~~&~~~Property 3~~~\\
				\hline
				Leverage Sampling  &~~~$ k \frac{6+2\eta}{3\eta^2} \log \frac{k}{\delta_1} $~~~
				&~~~$ \frac{k}{\epsilon \delta_2}  $~~~
				& --- \\
				Uniform Sampling  &~~~$ \mu (\U) k \frac{6+2\eta}{3\eta^2} \log \frac{k}{\delta_1} $~~~
				&~~~$ \frac{\mu (\U) k }{\epsilon \delta_2} $~~~
				& --- \\
				SRHT            &~~~$\lambda k \frac{6+2\eta}{3\eta^2} \log \frac{k}{\delta_1-0.01}$~~~
				& ~~~$\frac{\lambda k}{\epsilon (\delta_2-0.01)}$~~~
				& ~~~$ \lambda' k \frac{24 + 4 \sqrt{2\epsilon'}}{3 \epsilon'} \log \frac{2 {d} }{\delta_3-0.01 }$~~~\\
				Gaussian Projection\ &~~~$\frac{9 \big(\sqrt{k} + \sqrt{2 \log (2/\delta_1)} \big)^2}{\eta^2}$~~~
				& ~~~$\frac{18k}{\epsilon \delta_2}$~~~
				& ~~~$\frac{36  k}{\epsilon' }  \Big(1 + \sqrt{k^{-1} \log \frac{2{d}}{k\delta_3}} \Big)^2$~~~    \\
				Count Sketch &~~~$\frac{k^2 + k}{\delta_1 \eta^2}$~~~
				& ~~~$\frac{2 k}{\epsilon \delta_2}$~~~
				& ~~~---~~~    \\
				\hline
			\end{tabular}
		\end{small}
	\end{center}
\end{table}


\subsection{Column Selection}

In this subsection we prove Property~1 and Property~2 of
leverage score sampling and uniform sampling.
We cite the following lemma from \citep{wang2014modified};
the lemma was firstly proved by the work \citet{drineas2008cur,gittens2011spectral,woodruff2014sketching}.

\begin{lemma}\label{lem:sampling_property}
	Let $\U\in \RB^{n\times k}$ be any fixed matrix with orthonormal columns.
	The column selection matrix $\S \in \RB^{n\times s}$ samples $s$
	columns according to arbitrary probabilities $p_1 , p_2 , \cdots , p_n$.
	Assume $\alpha \geq k$ and
	\begin{equation*}
	\max_{i\in [n]} \frac{\| \u_{i:} \|_2^2}{p_{i}}
	\; \leq \; \alpha .
	\end{equation*}
	If
	$s \, \geq \, \alpha \frac{6 + 2\eta}{3 \eta^2} \log (k/\delta_1)$,
	it holds that
	\[
	\PB \Big\{ \big\| \I_k - \U^T \S \S^T \U \big\|_2 \; \geq \; \eta \Big\}
	\; \leq \; \delta_1 .
	\]
	If $s \, \geq \, \frac{\alpha }{\epsilon \delta_2}$,
	it holds that
	\[
	\PB \Big\{ \big\| \U \B - \U^T \S \S^T \B \big\|_F^2 \; \geq \; \epsilon \|\B\|_F^2 \Big\}
	\; \leq \; \delta_2 .
	\]
\end{lemma}

Leverage score sampling satisfies
$\max_{i\in [n]} \frac{\| \u_{i:} \|_2^2}{p_{i}} \leq k$.
Uniform sampling satisfies
$\max_{i\in [n]} \frac{\| \u_{i:} \|_2^2}{p_{i}} \leq \mu (\U) k$,
where $\mu (\U)$ is the row coherence of $\U$.
Then Property~1 and Property~2 of the two column sampling methods
follow from Lemma~\ref{lem:sampling_property}.

\begin{remark} \label{remark:column_sampling}
	Let $p_1 , \cdots , p_n$ be the sampling probabilities corresponding to the leverage score sampling or uniform sampling,
	and let $\tilde{p}_i \in [p_i, 1]$ for all $i\in [n]$ be arbitrary.
	For all $i\in [n]$, if the $i$-th column is sampled with probability $s \tilde{p}_i$
	and scaled by $\frac{1}{\sqrt{ s \tilde{p}_i}}$ if it gets sampled,
	then Lemma~\ref{lem:product} still holds.
	This can be easily seen from the proof of the above lemma (in \citep{wang2014modified}).
	Intuitively, it indicates that if we increase the sampling probabilities, the resulting error bound will not get worse.
\end{remark}


\subsection{Count Sketch}

Count sketch stems from the data stream literature \citep{charikar2004finding,thorup2012tabulation}.
Theoretical guarantees were first shown by
\citet{weinberger2009feature,pham2013fast,clarkson2013low}.
\citet{meng2013low,nelson2013osnap} strengthened and simplified the proofs.
Because the proof is involved,
we will not show the proof here.
The readers can refer to \citep{meng2013low,nelson2013osnap,woodruff2014sketching} for the proof.


\subsection{Property 1 and Property 2 of SRHT}

The properties of SRHT were established in the previous work \citep{drineas2011faster,lu2013faster,tropp2011improved}.
Following \citep{tropp2011improved}, we show a simple proof
of the properties of SRHT.
Our analysis is based on the following two key observations.
\begin{itemize}
	\item
	The scaled Walsh-Hadamard matrix $\frac{1}{\sqrt{n}} \H_n$ and the diagonal matrix $\D$ are both orthogonal,
	so $\frac{1}{\sqrt{n}} \D \H_n$ is also orthogonal.
	If $\U$ has orthonormal columns,
	the matrix $\frac{1}{\sqrt{n}} (\D \H_n)^T \U$ has orthonormal columns.
	\item
	For any fixed matrix $\U \in \RB^{n\times k}$ ($k\ll n$) with orthonormal columns,
	the matrix $\frac{1}{\sqrt{n}} (\D \H_n)^T \U \in \RB^{n\times k}$
	has low row coherence with high probability.
	\citet{tropp2011improved} showed that the row coherence of $\frac{1}{\sqrt{n}} (\D \H_n)^T \U$
	satisfies
	\[
	\mu
	\; \triangleq \; \frac{n}{k} \, \max_{i \in [n]}
	\Big\| \big( \frac{1}{\sqrt{n}} (\D \H_n )^T \U \big)_{i:} \Big\|_2^2
	\; \leq \; \bigg( 1 + \sqrt{\frac{8 \log (n / \delta)}{k}} \bigg)^2
	\]
	with probability at least $1 - \delta$.
	In other words, the randomized Hadamard transform flats out the leverage scores.
	Consequently uniform sampling can be safely applied to form a sketch.
\end{itemize}

In the following, we use the properties of uniform sampling and
the bound on the coherence $\mu$ to analyze SRHT.
Let $\V \triangleq \frac{1}{\sqrt{n}} (\D \H_n)^T \U \in \RB^{n\times k}$,
$\bar\B \triangleq \frac{1}{\sqrt{n}} (\D \H_n)^T \B \in \RB^{n\times d}$,
and $\mu$ be the row coherence of $\V$.
It holds that
\begin{align*}
& \V^T \V = \U^T \U = \I_k, \qquad \V^T \PP \PP^T \V = \U^T \S \S^T \U , \\
& \V^T \bar\B = \U^T \B, \quad \V^T \PP \PP^T \bar{\B} = \U^T \S \S^T \B ,
\quad \| \bar{\B} \|_F = \|\B \|_F , \\
& \PB \Big\{ \mu \: > \:  \big( 1 + \sqrt{8 k^{-1} \log (100 n ) } \big)^2 \Big\}
\; \leq \; 0.01.
\end{align*}
Therefore it suffices to prove that
\begin{align*}
& \PB \Big\{ \big\| \I_k - \V^T \PP \PP^T \V \big\|_2 \; \geq \; \eta \Big\}
\; \leq \; \delta_1 - 0.01 , \\
& \PB \Big\{ \big\| \V \bar\B - \V^T \PP \PP^T \bar\B \big\|_F^2
\; \geq \; \epsilon \|\bar\B\|_F^2 \Big\}
\; \leq \; \delta_2 - 0.01 .
\end{align*}
The above inequalities follows from the two properties
of uniform sampling.


\subsection{Property 1 and Property 2 of Gaussian Projection}

The two properties of Gaussian projection can be found in \citep{woodruff2014sketching}.
In the following we prove Property~1 in a much simpler way than \citep{woodruff2014sketching}.

The concentration of the singular values of standard Gaussian matrix is very well known.
Let $\G$ be an $n \times s$ ($n> s$) standard Gaussian matrix.
For any fixed matrix $\U \in \RB^{n\times k}$ with orthonormal columns,
the matrix $\N = \G^T \U \in \RB^{s\times k}$ is also standard Gaussian matrix.
\citet{vershynin2010introduction} showed that
for every $t \geq 0$, the following holds with probability at least $1-2 e^{-t^2/2}$:
\[
\sqrt{s} - \sqrt{k} - t
\leq \sigma_{k} (\N)
\leq \sigma_{1} (\N)
\leq \sqrt{s} + \sqrt{k} + t .
\]
Therefore, for any $\eta \in (0, 1)$,
if $s = 9 \eta^{-2} \big(\sqrt{k} + \sqrt{2 \log (2/\delta_1)} \big)^2$,
then
\[
\sigma_i (\U^T \S \S^T \U)
\; = \; \sigma_i^2 (\S^T\U)
\; \in \; 1 \pm \eta
\qquad \textrm{ for all } i \in [n]
\]
hold simultaneously with probability at least $1-\delta_1$.
Hence
\begin{align*}
& \PB \Big\{ \big\| \I_k - \U^T \S \S^T \U \big\|_2 \; \geq \; \eta \Big\}
\; \leq \; \delta_1  .
\end{align*}
This concludes Property~1 of Gaussian projection.


\subsection{Property~3 of SRHT and Gaussian Projection}

The following lemma is the main result of \citep{cohen2015optimal}.
If a sketching method satisfies Property~1 for arbitrary column orthogonal matrix $\U$,
then it satisfies Property~3 due to the following lemma.
Notice that the lemma does not apply to the leverage score and uniform sampling
because they depends on the leverage scores or matrix coherence of specific
column orthogonal matrix $\U$.
The lemma is inappropriate for count sketch
because Property~1 of count sketch holds with constant probability rather than arbitrary high probability.

\begin{lemma} \label{lem:product_spectral}
	Let $\A \in \RB^{n\times k}$ and $\B \in \RB^{n\times d}$ be any fixed matrices and $r$ be any fixed integer.
	Let $\tilde{k}\geq k$ and $\tilde{d} \geq d$ be the least integer divisible by $r$.
	Let $\S \in \RB^{n\times s}$ be a certain data-independent sketching matrix
	satisfying
	\[
	\PB \Big\{ \big\| \U^T \S \S^T \U - \I \big\|_2^2
	\; \geq \; \eta \Big\}
	\; \leq \; \frac{r^2 \delta_3 }{\tilde{k} \tilde{d}}
	\]
	for any fixed matrix $\U\in \RB^{n\times 2r}$ with orthonormal columns. Then
	\[
	\big\| \A^T \S \S^T \B -  \A^T \B \big\|_2^2
	\; \leq \; \eta \bigg( \|\A\|_2^2 + \frac{\| \A \|_F^2 - \| \A \|_2^2}{r} \bigg)  \bigg( \|\B\|_2^2 + \frac{\| \B \|_F^2 - \| \B \|_2^2}{r} \bigg)
	\]
	holds with probability at least $1-\delta_3$.
\end{lemma}

SRHT and Gaussian projection enjoys Property 1 with high probability for arbitrary
column orthogonal matrix $\U$.
Thus Property~3 can be immediately obtained by applying the above lemma
with the setting $r=k$.


\section{Proof of Theorem~\ref{thm:faster_spsd}} \label{sec:proof_faster_spsd}

Let $\K \in \RB^{n\times n}$ be any fixed SPSD matrix, $\C \in \RB^{n\times c}$ be any fixed matrix,
$\S \in \RB^{n\times s}$ be a sketching matrix, and
\begin{eqnarray*}
	\U^\star & = & \argmin_\U \big\| \K - \C \U \C^T \big\|_F^2
	\; = \; \C^\dag \K (\C^T)^\dag,\\
	\tilde\U & = & \argmin_\U \big\| \S^T ( \K - \C \U \C^T ) \S \big\|_F^2
	\; = \; (\S^T \C)^\dag (\S^T \K \S)  (\C^T \S)^\dag .
\end{eqnarray*}
Lemma~\ref{lem:spsd_det} is a direct consequence of Lemma~\ref{lem:curtype_det}.

\begin{lemma} \label{lem:spsd_det}
	Let $\K \in \RB^{n\times n}$ be any fixed SPSD matrix, $\C \in \RB^{n\times c}$ be any fixed matrix,
	and $\C = \U_\C \Si_\C \V_\C^T$ be the SVD.
	Assume that $\S^T \U_\C$ has full column rank.
	Let $\U^\star$ and $\tilde\U$ be defined in the above.
	Then the following inequality holds:
	\begin{eqnarray*}
		\|\K - \C \tilde\U \C^T \|_F^2
		& \leq & \| \A - \C \U^\star \C^T \|_F^2 + \Big( 2 f \sqrt{ h }  + f^2 \sqrt{g_2 g_F} \Big)^2 ,
	\end{eqnarray*}
	where $\alpha \in [0, 1]$ is arbitrary and
	\begin{align*}
	&f = \sigma_{\min}^{-1} (\U_\C^T \S \S^T \U_\C),
	\qquad h = \big\|\U_\C^T \S \S^T (\K - \U_\C \U_\C^T \K) \big\|_F^2, \\
	&g_2 = \big\| \U_\C^T \S  \S^T (\I_n - \U_\C \U_\C^T) \K^{\alpha} \big\|_2^2 ,
	\qquad g_F = \big\| \U_\C^T \S  \S^T (\I_n - \U_\C \U_\C^T) \K^{1-\alpha} \big\|_F^2.
	\end{align*}
\end{lemma}

The following lemma shows that $\tilde\X$ is nearly as good as $\X^\star$ in terms of objective function value
if $\S$ satisfies Assumption~\ref{assumption:faster_spsd1}.

\begin{assumption} \label{assumption:faster_spsd1}
	Let $\B$ be any fixed matrix.
	Let $\C \in \RB^{m\times c}$ and $\C = \U_\C \Si_\C \V_\C^T$ be the SVD.
	Assume that the sketching matrix $\S \in \RB^{m\times s}$ satisfies
	\begin{eqnarray*}
		\PB \Big\{ \big\| \U_\C \S \S^T \U_\C - \I  \big\|_2
		\; \geq \; \frac{1}{10} \Big\}
		& \leq & \delta_{1} \\
		\PB \Big\{ \big\| \U_\C^T \S \S^T \B - \U_\C^T \B \big\|_F^2
		\; \geq \; \epsilon \|\B\|_F^2 \Big\}
		& \leq & \delta_{2}
	\end{eqnarray*}
	for any $\delta_1, \delta_2 \in (0, 1/3)$.
\end{assumption}

\begin{lemma} \label{lemma:spsd_det1}
	Let $\K \in \RB^{n\times n}$ be any fixed SPSD matrix, $\C \in \RB^{n\times c}$ be any fixed matrix,
	and $\C = \U_\C \Si_\C \V_\C^T$ be the SVD.
	Let $\U^\star$ and $\tilde\U$ be defined in the above, respectively.
	Let $\S \in \RB^{n\times s}$ be certain sketching matrix satisfying Assumption~\ref{assumption:faster_spsd1}.
	Assume that $\epsilon^{-1} = o(n)$.
	Then
	\begin{align*}
	& \big\|\K - \C \tilde\U \C^T \big\|_F^2 - \big\| \A - \C \U^\star \C^T \big\|_F^2 \\
	& \leq \;  \Big( \frac{20\sqrt{\epsilon}}{9} \big\| \A - \C \U^\star \C^T \big\|_F
	+ \frac{100\epsilon}{81} \big\| (\I_n - \U_\C \U_\C^T ) \K  \big\|_* \Big)^2 \\
	& \leq \; 4 \epsilon^2 n    \big\| \A - \C \U^\star \C^T \big\|_F^2 .
	\end{align*}
	holds with probability at least $1-\delta_1 - 2\delta_2$.
\end{lemma}

\begin{proof}
	Let $f$, $h$, $g_2$, $g_F$, $\alpha$ be defined in Lemma~\ref{lem:spsd_det} and fix $\alpha = 1/2$.
	Under Assumption~\ref{assumption:faster_spsd1} it holds simultaneously with probability at least $1-\delta_1 -2 \delta_2$ that
	\begin{align*}
	f \leq \frac{10}{9},
	\qquad
	h \leq \epsilon \big\| (\I_n - \U_\C \U_\C^T ) \K \big\|_F^2 ,
	\qquad
	g_2 \leq g_F \leq \epsilon \big\| (\I_n - \U_\C \U_\C^T ) \K^{1/2} \big\|_F^2 .
	\end{align*}
	It follows that
	\begin{eqnarray*}
		g_2 \leq g_F
		& \leq & \epsilon \cdot \tr \Big( (\I_n - \U_\C \U_\C^T ) \K^{1/2} \K^{1/2} (\I_n - \U_\C \U_\C^T ) \Big) \\
		& \leq & \epsilon \cdot \tr \Big( (\I_n - \U_\C \U_\C^T ) \K (\I_n - \U_\C \U_\C^T ) \Big) \\
		& = & \epsilon \big\| (\I_n - \U_\C \U_\C^T ) \K (\I_n - \U_\C \U_\C^T ) \big\|_* \\
		& \leq & \epsilon \big\| (\I_n - \U_\C \U_\C^T ) \K  \big\|_* .
	\end{eqnarray*}
	It follows from Lemma~\ref{lem:spsd_det} and the assumption $\epsilon^{-1} = o (n)$ that
	\begin{align*}
	& \big\|\K - \C \tilde\U \C^T \big\|_F^2 - \big\| \A - \C \U^\star \C^T \big\|_F^2 \\
	& \leq \;  \Big( \frac{20\sqrt{\epsilon}}{9} \big\| \A - \C \U^\star \C^T \big\|_F
	+ \frac{10^2 \epsilon}{9^2} \big\| (\I_n - \U_\C \U_\C^T ) \K  \big\|_* \Big)^2 \\
	& \leq \;  \Big( \frac{20\sqrt{\epsilon}}{9} \big\| \A - \C \U^\star \C^T \big\|_F
	+ \frac{10^2 \epsilon \sqrt{n}}{9^2} \big\| (\I_n - \U_\C \U_\C^T ) \K  \big\|_F \Big)^2 \\
	& = \; \frac{10^4 \epsilon^2 n}{9^4} \big( 1+ o(1) \big)  \big\| \A - \C \U^\star \C^T \big\|_F^2 ,
	\end{align*}
	by which the lemma follows.
\end{proof}

Under both Assumption~\ref{assumption:faster_spsd1} and Assumption~\ref{assumption:faster_spsd2},
the error bound can be further improved.
We show the improved bound in Lemma~\ref{lemma:spsd_det2}.

\begin{assumption} \label{assumption:faster_spsd2}
	Let $\B$ be any fixed matrix.
	Let $\C \in \RB^{m\times c}$, $k_c=\rk (\C)$, and $\C = \U_\C \Si_\C \V_\C^T$ be the SVD.
	Assume that the sketching matrix $\S \in \RB^{m\times s}$ satisfies
	\begin{eqnarray*}
		\PB \Big\{ \big\| \U_\C^T \S \S^T \B - \U_\C^T \B \big\|_2^2
		\; \geq \; \epsilon \|\B\|_2^2 + \frac{\epsilon}{k_c} \|\B\|_F^2 \Big\}
		& \leq & \delta_{3}
	\end{eqnarray*}
	for any $\delta_3\in (0, 1/3)$.
\end{assumption}

\begin{lemma} \label{lemma:spsd_det2}
	Let $\K \in \RB^{n\times n}$ be any fixed SPSD matrix, $\C \in \RB^{n\times c}$ be any fixed matrix,
	$k_c =\rk (\C)$, and $\C = \U_\C \Si_\C \V_\C^T$ be the SVD.
	Let $\U^\star$ and $\tilde\U$ be defined in the beginning of this section.
	Let $\S \in \RB^{n\times s}$ be certain sketching matrix satisfying both Assumption~\ref{assumption:faster_spsd1} and Assumption~\ref{assumption:faster_spsd2}.
	Assume that $\epsilon = o (n/k_c)$.
	Then
	\begin{align*}
	\big\|\K - \C \tilde\U \C^T \big\|_F^2
	\; \leq \; \big\| \A - \C \U^\star \C^T \big\|_F^2 +  4 \epsilon^2 n / k_c \,  \big\| \A - \C \U^\star \C^T \big\|_F^2
	\end{align*}
	holds with probability at least $1-\delta_1 - \delta_2 -  \delta_3$.
\end{lemma}

\begin{proof}
	Let $f$, $h$, $g_2$, $g_F$, $\alpha$ be defined in Lemma~\ref{lem:spsd_det} and fix $\alpha = 0$.
	Under Assumption~\ref{assumption:faster_spsd1} it holds simultaneously with probability at least $1-\delta_1 - \delta_2$ that
	\begin{align*}
	f \leq \frac{10}{9},
	\qquad
	h = g_F \leq \epsilon \big\| (\I_n - \U_\C \U_\C^T ) \K \big\|_F^2 .
	\end{align*}
	Under Assumption~\ref{assumption:faster_spsd2}, it holds with probability at least $1-\delta_3$ that
	\begin{eqnarray*}
		g_2
		& = & \big\| \U_\C^T \S  \S^T (\I_n - \U_\C \U_\C^T)   + \underbrace{\U_\C^T  (\I_m - \U_\C \U_\C^T) }_{=\0} \big\|_2^2 \\
		& \leq & \epsilon \big\| \I_n - \U_\C \U_\C^T \big\|_2^2 + \frac{\epsilon}{k_c} \big\| \I_n - \U_\C \U_\C^T \big\|_F^2
		\; \leq \; \epsilon + \frac{\epsilon}{k_c} (n-k_c )
		\; = \; \frac{\epsilon n}{k_c }.
	\end{eqnarray*}
	It follows from Lemma~\ref{lem:spsd_det} and the assumption $\epsilon^{-1} = o (n / k_c)$ that
	\begin{align*}
	& \big\|\K - \C \tilde\U \C^T \big\|_F^2 - \big\| \A - \C \U^\star \C^T \big\|_F^2 \\
	& \leq \;  \Big( \frac{20\sqrt{\epsilon}}{9} \big\| \A - \C \U^\star \C^T \big\|_F
	+ \frac{10^2 \epsilon}{9^2} \sqrt{n / k_c } \, \big\| \A - \C \U^\star \C^T \big\|_F \Big)^2 \\
	& \leq \; 4 \epsilon^2 {n/k_c } \, \big\| \A - \C \U^\star \C^T \big\|_F^2 ,
	\end{align*}
	by which the lemma follows.
\end{proof}

Finally, we prove Theorem~\ref{thm:faster_spsd} using Lemma~\ref{lemma:spsd_det1} and Lemma~\ref{lemma:spsd_det2}.
Leverage score sampling, uniform sampling, and count sketch satisfy Assumption~\ref{assumption:faster_spsd1},
and the bounds follow by setting $\epsilon = 0.5 \sqrt{\epsilon' / n}$ and applying Lemma~\ref{lemma:spsd_det1}.
For the three sketching methods, we set $\delta_{1} = 0.01$ and $\delta_{2} = 0.095$.

Gaussian projection and SRHT satisfy  Assumption~\ref{assumption:faster_spsd1} and Assumption~\ref{assumption:faster_spsd2},
and their bounds follow by setting $\epsilon = 0.5 \sqrt{ \epsilon' k_c / n}$ and applying Lemma~\ref{lemma:spsd_det2}.
For Gaussian projection, we set $\delta_{1}  = 0.01$, $\delta_{2} = 0.09$, and $\delta_{3}  = 0.1$.
For SRHT, we set $\delta_{1} = 0.02$, $\delta_{2}   = 0.08$, and $\delta_{3}  = 0.1$.


\section{Proof of Theorem~\ref{thm:exact}} \label{sec:proof_exact}

Since $\C = \K \PP \in \RB^{n\times c}$, $\W = \PP^T \C \in \RB^{c \times c}$,
and $\rk (\S^T \C)  \geq \rk (\W)$, we have that
\begin{equation} \label{eq:thm:connection:4}
\rk(\K) \;\geq\; \rk(\C) \;\geq \; \rk (\S^T \C) \; \geq \; \rk(\W) \textrm{.}
\end{equation}
If $\rk (\C) = \rk (\K)$, there exists a matrix $\X$ such that
$\K = \C \X$.
By left multiplying both sides by $\PP^T$, it follows that
\[
\C^T \; = \; \PP^T \K
\; = \; \PP^T \C \X
\; = \; \W \X ,
\]
and thus $\rk (\W) = \rk (\S^T \C) = \rk (\C) =  \rk (\K)$.
It follows from $\K = \C \X$ and $\C = \X^T \W$ that
\begin{eqnarray*}
\K
\; = \; \X^T \W \X .
\end{eqnarray*}
We let $\Ph = \X \S$,
and it holds that
\begin{eqnarray*}
\Kgen
& = & \C (\S^T \C)^\dag (\S^T \K \S) (\C^T \S)^\dag \C^T \\
& = & \X^T \W (\S^T \X^T \W)^\dag (\S^T \X^T \W \X \S) (\W \X \S)^\dag \W \X \\
& = & \X^T \W (\Ph^T \W)^\dag (\Ph^T \W \Ph) (\W \Ph)^\dag \W \X .
\end{eqnarray*}
Let $\rk (\W) = \rk (\C) = \rk (\S^T \C) = \rk (\K) = \rho$.
Since $\W$ is symmetric, we denote the rank-$\rho$ eigenvalue decomposition of $\W$ by
\[
\W \; = \; \underbrace{\U_\W}_{c\times \rho} \underbrace{\Lam_\W }_{\rho \times \rho} \underbrace{\U_\W^T}_{\rho \times c} .
\]
Since $\S^T \C = \Ph^T \W$ and $\rk (\S^T \C) = \rk (\W) = \rho $,
we have that $\rk (\Ph^T \W) = \rk (\W) = \rho$.
The $n\times \rho$ matrix $\Ph^T \U_\W$ must have full column rank,
otherwise $\rk (\Ph^T \W) < \rho$.
Thus we have
\[
(\Ph^T \W)^\dag
\; = \; (\Ph^T \U_\W \Lam_\W \U_\W^T)^\dag
\; = \; (\Lam_\W \U_\W^T)^\dag (\Ph^T \U_\W)^\dag .
\]
It follows that
\begin{eqnarray*}
\Kgen
& = & \X^T \W \underbrace{(\Lam_\W \U_\W^T)^\dag}_{c\times \rho} \underbrace{(\Ph^T \U_\W)^\dag}_{\rho\times n} \underbrace{(\Ph^T \U_\W) }_{n\times \rho}
    \Lam_\W  (\U_\W^T \Ph) (\U_\W^T \Ph)^\dag  (\U_\W \Lam_\W)^\dag    \W \X \\
& = & \X^T \U_\W \Lam_\W \U_\W \X
\; = \; \X^T \W \X
\; = \; \K .
\end{eqnarray*}
This shows that the fast model is exact.
To this end, we have shown that if $\rk (\C) = \rk (\K)$, then the fast model is exact.

Conversely, if the fast model is exact, that is,
$\K = \C (\S^T \C)^\dag (\S^T \K \S) (\C^T \S)^\dag \C^T$,
we have that $\rk(\K) \leq \rk(\C)$.
It follows from \eqref{eq:thm:connection:4} that $\rk (\K) = \rk (\C)$.


\section{Proof of Theorem~\ref{thm:lower_bound}} \label{sec:proof_lower_bound}

We prove Theorem~\ref{thm:lower_bound} by constructing an adversarial case.
Theorem~\ref{thm:lower_bound} is a direct consequence of the following theorem.

\begin{theorem} \label{thm:lower_bound_sketch_A}
Let $\A$ be the $n\times n$ symmetric matrix defined in Lemma~\ref{lem:nystrom_residual_new} with $\alpha \rightarrow 1$
and $k$ be any positive integer smaller than $n$.
Let $\PM$ be any subset of $[n]$ with cardinality $c$
and $\C\in \RB^{n\times c}$ contain $c$ columns of $\A$ indexed by $\PM$.
Let $\S$ be any $n\times s$ column selection matrix satisfying $\PM \subset \SM$, where $\SM \subset [n]$ is the index set formed by $\S$.
Then the following inequality holds:
\begin{eqnarray*}
\frac{\|\A - \C (\S^T \C)^\dag (\S^T \A \S) (\C^T \S)^\dag \C^T \|_F^2}{\| \A - \A_k\|_F^2 }
& \geq & \frac{n-c}{n-k} \Big( 1 + \frac{2 k }{c} \Big) + \frac{n-s}{n-k} \frac{k (n-s)}{s^2} .
\end{eqnarray*}
\end{theorem}

\begin{proof}
Let $\A$ and $\B$ be defined in Lemma~\ref{lem:nystrom_residual_new}.
We prove the theorem using Lemma~\ref{lem:nystrom_residual_new} and Lemma~\ref{lem:lower_bound_sketch_B}.
Let $n = p k$.
Let $\C$ consist of $c$ column sampled from $\A$
and $\hat{\C}_i$ consist of $c_i$ columns sampled from the $i$-th diagonal block of $\A$.
Thus $\C = \diag (\hat{\C}_1 , \cdots , \hat{\C}_k)$.
Without loss of generality, we assume $\hat{\C}_i$ consists of the first $c_i$ columns of $\B$.
Let $\hat\S = \diag \big(\hat\S_1, \cdots , \hat\S_k \big)$ be an $n\times s$ column selection matrix,
where $\hat\S_i$ is a $p\times s_i$ column selection matrix and $s_1 + \cdots s_k = s$.
Then the $\U$ matrix is computed by
\begin{eqnarray}
\U
& = & \big(\S^T \C \big)^\dag \big(\S^T \A \S \big) \big(\C^T \S \big)^\dag \nonumber\\
& = &
\big[ \diag\big( \hat{\S}_1^T \hat{\C}_1 , \cdots , \hat{\S}_k^T \hat{\C}_k \big) \big]^\dag
        \diag\big( \hat{\S}_1^T \B \hat\S_1 , \cdots , \hat{\S}_k^T \B \hat\S_k \big)
        \big[ \diag\big( \hat{\C}_1^T \hat{\S}_1 , \cdots , \hat{\C}_k^T \hat{\S}_k \big) \big]^\dag
         \nonumber\\
& = & \diag\Big( \big( \hat{\S}_1^T \hat{\C}_1\big)^\dag \big( \hat{\S}_1^T \B \hat\S_1 \big) \big( \hat{\C}_1^T \hat{\S}_1\big)^\dag ,
    \cdots , \big( \hat{\S}_k^T \hat{\C}_k\big)^\dag \big( \hat{\S}_k^T \B \hat\S_k \big) \big( \hat{\C}_k^T \hat{\S}_k \big)^\dag \Big)  \textrm{.} \nonumber
\end{eqnarray}
The approximation formed by the fast model is the block-diagonal matrix whose the $i$-th ($i\in [k]$) diagonal block is the $p\times p$ matrix
\begin{eqnarray*}
\big[\tilde{\A}^{\textrm{fast}}_{c,s} \big]_{i i}
& = &  \hat\C_i \big( \hat{\S}_i^T \hat{\C}_i \big)^\dag
    \big( \hat{\S}_i^T \B \hat\S_i \big)
    \big( \hat{\C}_i^T \hat{\S}_i \big)^\dag \hat\C_i^T . \nonumber
\end{eqnarray*}
It follows from Lemma~\ref{lem:lower_bound_sketch_B} that for any $i\in [k]$,
\begin{align}
\lim_{\alpha \rightarrow 1} \frac{\|\B - \big[\tilde{\A}^{\textrm{fast}}_{c,s} \big]_{i i} \|_F^2}{(1-\alpha)^2 }
\; = \; (p-c_i) \Big(1+ \frac{2}{c_i} \Big) + \frac{(p-s_i)^2}{ s_i^2} . \nonumber
\end{align}
Thus
\begin{eqnarray*}
\lim_{\alpha \rightarrow 1} \frac{\|\A - \tilde{\A}^{\textrm{fast}}_{c,s} \|_F^2}{(1-\alpha)^2 }
& = & \lim_{\alpha \rightarrow 1} \sum_{i=1}^k  \frac{\|\B - \big[\tilde{\A}^{\textrm{fast}}_{c,s} \big]_{i i} \|_F^2}{(1-\alpha)^2 } \\
& = & \sum_{i=1}^k (p-c_i) \Big(1+ \frac{2}{c_i} \Big) + \frac{(p-s_i)^2}{ s_i^2} \\
& = & \bigg( \sum_{i=1}^k p - c_i - 2 \bigg) +  \bigg( 2p \sum_{i=1}^k \frac{1}{c_i} \bigg)
    + \bigg( p^2 \sum_{i=1}^k \frac{1}{s^2_i} \bigg)  - \bigg( 2 p \sum_{i=1}^k \frac{1}{s_i} \bigg) + k \\
& \geq & n - c - 2k + \frac{2 n k}{c} + \frac{k n^2}{s^2} - \frac{2 n k}{s} + k \\
& = & (n-c) \Big( 1 + \frac{2 k }{c} \Big) + \frac{k (n-s)^2}{s^2} .
\end{eqnarray*}
Here the inequality follows by minimizing over $c_1 , \cdots , c_k$ and $s_1 , \cdots , s_k$
with constraints $\sum_i c_i = c$ and $\sum_i s_i = s$.
Finally, it follows from Lemma~\ref{lem:nystrom_residual_new} that
\begin{eqnarray*}
\lim_{\alpha \rightarrow 1} \frac{\|\A - \tilde{\A}^{\textrm{fast}}_{c,s} \|_F^2}{\| \A - \A_k\|_F^2 }
& \geq & \frac{n-c}{n-k} \Big( 1 + \frac{2 k }{c} \Big) + \frac{n-s}{n-k} \frac{k (n-s)}{s^2} .
\end{eqnarray*}
\end{proof}

\subsection{Key Lemmas} \label{sec:proof:lower_bound:lemma}

Lemma \ref{lem:pinv_partitioned_matrix} provides a useful tool for expanding the Moore-Penrose inverse of partitioned matrices.

\begin{lemma}[Page~179 of \cite{adi2003inverse}] \label{lem:pinv_partitioned_matrix}
Given a matrix $\X\in \RBmn$ of rank $c$ which has a nonsingular $c\times c$ submatrix $\X_{1 1}$.
By rearrangement of columns and rows by permutation matrices $\PP$ and $\Q$,
the submatrix $\X_{1 1}$ can be bought to the top left corner of $\X$, that is,
\[
\PP \X \Q \; =\; \left[
             \begin{array}{cc}
               \X_{1 1} & \X_{1 2} \\
               \X_{2 1} & \X_{2 2} \\
             \end{array}
           \right].
\]
Then the Moore-Penrose inverse of $\X$ is
\begin{align}
\X^\dag \; =\;
\Q \left[
     \begin{array}{c}
       \I_c \\
       \T^T \\
     \end{array}
   \right]
\big( \I_c + \T \T^T \big)^{-1} \X_{1 1}^{-1}   \big( \I_c + \H^T \H \big)^{-1}
\left[
  \begin{array}{cc}
    \I_c & \H^T  \\
  \end{array}
\right] \PP ,\nonumber
\end{align}
where $\T = \X_{1 1}^{-1} \X_{1 2}$ and
$\H = \X_{2 1} \X_{1 1}^{-1}$.
\end{lemma}

Lemmas~\ref{lem:nystrom_residual_new} and \ref{lem:lower_bound_sketch_B} will be used to prove Theorem~\ref{thm:lower_bound_sketch_A}.

\begin{lemma}[{Lemma~19 of \cite{wang2013improving}}] \label{lem:nystrom_residual_new}
Given $n$ and $k$, we let $\B$ be an $\frac{n}{k} \times \frac{n}{k}$ matrix
whose diagonal entries equal to one and off-diagonal entries equal to $\alpha \in [0, 1)$.
We let $\A$ be an $n\times n$ block-diagonal matrix
\begin{eqnarray}\label{eq:construction_bad_nystrom_blk}
\A \; =\;  \diag(\underbrace{\B, \cdots , \B}_{k \textrm{ blocks}}) .
\end{eqnarray}
Let $\A_k$ be the best rank-$k$ approximation to the matrix $\A$,
then we have that
\begin{equation}
\| \A - \A_k \|_F^2 \; =\; (1-\alpha)^2 (n-k)  \textrm{.} \nonumber
\end{equation}
\end{lemma}

\begin{lemma} \label{lem:inverse_expand}
The following equality holds for any nonzero real number $a$:
\begin{eqnarray*}
\big( a \I_c + b \1_c \1_c^T \big)^{-1}
& = & a^{-1} \I_c - \frac{b}{a (a + b c)} \1_c \1_c^T .
\end{eqnarray*}
\end{lemma}

\begin{proof}
The lemma directly follows from the Sherman-Morrison-Woodbury matrix identity
\[
(\X + \Y \Z \R)^{-1}  = \X^{-1} - \X^{-1} \Y (\Z^{-1} + \R \X^{-1} \Y)^{-1} \R \X^{-1} .
\]
\end{proof}

\begin{lemma} \label{lem:lower_bound_sketch_B}
Let $\B$ be any $n\times n$ matrix with diagonal entries equal to one and off-diagonal entries equal to $\alpha$.
Let $\C = \B \PP \in \RB^{n\times c}$;
let $\tilde{\B} = \C (\S^T \C)^\dag (\S^T \K \S) (\C^T \S)^\dag \C^T$
be the fast SPSD matrix approximation model of $\B$.
Let $\PM$ and $\SM$ be the index sets formed by $\PP$ and $\S$, respectively.
If $\PM \subset \SM$,
the error incurred by the fast model satisfies
\begin{align}
\lim_{\alpha \rightarrow 1} \frac{\|\B - \tilde\B \|_F^2}{(1-\alpha)^2 }
\; \geq \; (n-c) \Big(1+ \frac{2}{c} \Big) + \frac{(n-s)^2}{ s^2} \textrm{.} \nonumber
\end{align}
\end{lemma}

\begin{proof}
Let $\B_{1} = \S^T \B \S\in \RB^{s\times s}$
and $\C_1 = \S^T \C = \S^T \B \PP \in \RB^{s\times c}$.
Without loss of generality, we assume that $\PP$ selects the first $c$ columns and $\S$ selects the first $s$ columns.
We partition $\B$ and $\C$ by:
\[
\B = \left[
       \begin{array}{cc}
         \B_{1} & \B_{3}^T \\
         \B_{3} & \B_{2} \\
       \end{array}
     \right]
\qquad
\textrm{ and }
\qquad
\C = \left[
       \begin{array}{c}
         \C_1  \\
         \C_{2} \\
       \end{array}
     \right]
      = \left[
       \begin{array}{c}
         \W  \\
         \C_{1 2} \\
         \C_{2} \\
       \end{array}
     \right]
     \textrm{.}
\]
We further partition $\B_1\in \RB^{s\times s}$ by
\[
\B_1 = \left[
       \begin{array}{cc}
         \W & \C_{1 2}^T \\
         \C_{1 2} & \B_{1 2} \\
       \end{array}
     \right] ,
\]
where
\begin{eqnarray*}
\C_{1 2} = \alpha \1_{s-c} \1_c^T
\quad \textrm{ and } \quad
\B_{1 2} = (1-\alpha) \I_{s-c} + \alpha \1_{s-c} \1_{s-c}^T .
\end{eqnarray*}
The $\U$ matrix is computed by
\[
\U = (\S^T \C)^\dag (\S^T \B \S) (\C^T \S)^\dag
= \C_1^\dag \B_{1} (\C_1^\dag)^T .
\]
It is not hard to see that $\C_1$ contains the first $c$ rows of $\B_{1}$.

We expand the Moore-Penrose inverse of $\C_1$ by Lemma~\ref{lem:pinv_partitioned_matrix} and obtain
\[
\C_1^\dag = \W^{-1} \big( \I_c + \H^T \H \big)^{-1}
\left[
  \begin{array}{cc}
    \I_c & \H^T \\
  \end{array}
\right],
\]
where
\begin{equation*}
\W^{-1}
\; = \; \Big( (1-\alpha) \I_c + \alpha \1_c \1_c^T \Big)^{-1}
\; = \; \frac{1}{1-\alpha} \I_c - \frac{\alpha}{(1-\alpha)(1-\alpha+c\alpha)} \1_{c} \1_{c}^T
\end{equation*}
and
\[
\H = \C_{1 2} \W^{-1} = \frac{\alpha}{1-\alpha + c\alpha} \1_{s-c} \1_c^T.
\]
It is easily verified that $\H^T \H = \big( \frac{\alpha}{1-\alpha + c\alpha} \big)^2 (s-c) \1_c \1_c^T$.
It follows from Lemma~\ref{lem:inverse_expand} that
\begin{eqnarray*}
(\I_c + \H^T \H)^{-1}
& = & \I_c - \frac{(s-c) \alpha^2}{c (s-c) \alpha^2 + (1-\alpha + c\alpha)^2} \1_c \1_c^T.
\end{eqnarray*}
Then we obtain
\begin{eqnarray} \label{eq:c1_pinv}
\C_1^\dag
& = & \W^{-1} \big( \I_c + \H^T \H \big)^{-1}
\left[
  \begin{array}{cc}
    \I_c & \H^T \\
  \end{array}
\right] \nonumber \\
& = & \Big( \frac{1}{1-\alpha} \I_c + \gamma_1 \1_c \1_c^T \Big)
\left[
  \begin{array}{cc}
    \I_c & \H^T \\
  \end{array}
\right] ,
\end{eqnarray}
where
\begin{eqnarray*}
\gamma_1
& = & c \gamma_2 \gamma_3 - \gamma_2 - \frac{\gamma_3}{1-\alpha} ,\\
\gamma_2
& = & \frac{\alpha}{(1-\alpha)(1-\alpha+c\alpha)} ,\\
\gamma_3
& = & \frac{(s-c) \alpha^2}{c (s-c) \alpha^2 + (1-\alpha + c\alpha)^2}.
\end{eqnarray*}

Then
\begin{eqnarray} \label{eq:expand_middle}
[\I_c , \H^T ] \B_1 [\I_c , \H^T]^T
& = & \W + \B_{1 3}^T \H + \H^T \B_{1 3} + \H^T \B_{1 2} \H \nonumber \\
& = & (1-\alpha) \I_c + \gamma_4 \1_c \1_c^T ,
\end{eqnarray}
where
\begin{small}
\[
\gamma_4 = \frac{\alpha (3 \alpha s - \alpha c - 2 \alpha + \alpha^2 c - 3 \alpha^2 s + \alpha^2 + \alpha^2 s^2 + 1)}{(\alpha c - \alpha + 1)^2}.
\]
\end{small}
It follows from \eqref{eq:c1_pinv} \eqref{eq:expand_middle} that
\begin{eqnarray}
\U = \C_1^\dag \B_1 (\C_1^\dag)^T
& = & \Big( \frac{1}{1-\alpha} \I_c + \gamma_1 \1_c \1_c^T \Big)
\Big( (1-\alpha) \I_c + \gamma_4 \1_c \1_c^T \Big)
\Big( \frac{1}{1-\alpha} \I_c + \gamma_1 \1_c \1_c^T \Big) \nonumber \\
& = & \frac{1}{1-\alpha} \I_c + \gamma_5 \1_c \1_c^T ,\nonumber
\end{eqnarray}
where
\begin{eqnarray*}
\gamma_5 & = &
\gamma_1 + \Big( c\gamma_1 + \frac{1}{1-\alpha} \Big)
\Big( c \gamma_1 \gamma_4 + \gamma_1 (1-\alpha) + \frac{\gamma_4}{1-\alpha} \Big) .
\end{eqnarray*}
Then we have
\begin{eqnarray}
\W \U & = & \I_c + \gamma_6 \1_c \1_c^T, \nonumber \\
\gamma_6 & = & (1-\alpha + \alpha c) \gamma_5 + \frac{\alpha}{1-\alpha} . \nonumber
\end{eqnarray}

We partition the fast SPSD matrix approximation model by
\begin{eqnarray*}
\tilde\B = \left[
       \begin{array}{cc}
         \tilde\W & \tilde\B_{2 1}^T \\
         \tilde\B_{2 1} & \tilde\B_{2 2} \\
       \end{array}
     \right],
\end{eqnarray*}
where
\begin{eqnarray*}
\tilde\B_{1 1}
& = & \W \U \W
\; = \; (1-\alpha)\I_c + \big( \alpha + (1-\alpha + c\alpha) \gamma_6 \big) \1_c \1_c^T , \\
\tilde\B_{2 1} & = & \W \U \big( \alpha \1_c \1_{n-c}^T \big)
\; = \; \alpha (1+c\gamma_6) \1_c \1_{n-c}^T , \\
\tilde\B_{2 2} & = & \big( \alpha \1_{n-c} \1_{c}^T \big) \U \big( \alpha \1_c \1_{n-c}^T \big)
\; = \; \alpha^2 c \Big( \frac{1}{1-\alpha} + \gamma_5 c \Big) \1_c \1_{n-c}^T \\
\end{eqnarray*}

The approximate error is
\begin{eqnarray*}
\big\|\B - \tilde\B \big\|_F^2
& = & \big\| \W - \tilde\W \big\|_F^2 + 2 \big\| \B_{2 1} - \tilde\B_{2 1} \big\|_F^2 + \big\| \B_{2 2} - \tilde\B_{2 2} \big\|_F^2 ,
\end{eqnarray*}
where
\begin{small}
\begin{eqnarray*}
\big\| \W - \tilde\W \big\|_F^2
& = & \big\| (1-\alpha + c\alpha) \gamma_6  \1_c \1_c^T \big\|_F^2
\; = \; c^2 (1-\alpha + c\alpha)^2 \gamma_6^2  ,\\
\big\| \B_{2 1} - \tilde\B_{2 1} \big\|_F^2
& = & \big\| \alpha c \gamma_6 \1_c \1_{n-c}^T \big\|_F^2
\; = \; \alpha^2 c^3 (n-c) \gamma_6^2  ,\\
\big\| \B_{2 2} - \tilde\B_{2 2} \big\|_F^2
& = & \underbrace{(n-c) (n-c-1) \alpha^2 \Big( \frac{\alpha c}{1-\alpha} + \alpha c^2 \gamma_5 - 1 \Big)^2}_{\textrm{off-diagonal}}
    + \underbrace{(n-c) \Big( \frac{\alpha^2 c}{1-\alpha} + \alpha^2 c^2 \gamma_5 - 1 \Big)^2}_{\textrm{diagonal}}.
\end{eqnarray*}
\end{small}
We let
\begin{eqnarray*}
\eta \triangleq \frac{\|\B - \tilde\B \|_F^2}{(1-\alpha)^2 } ,
\end{eqnarray*}
which is a symbolic expression of $\alpha$, $n$, $s$, and $c$.
We then simplify the expression using MATLAB and substitute the $\alpha$ in $\eta$ by $1$, and we obtain
\begin{eqnarray*}
\lim_{\alpha \rightarrow 1} \eta = (n-c) (1+2/c) + (n-s)^2 / s^2,
\end{eqnarray*}
by which the lemma follows.
\end{proof}


\section{Proof of Theorem~\ref{thm:optimal_cur}} \label{sec:optimal_cur_proof}

We define the projection operation $\PM_{\C , k} (\A) = \C \X$ where $\X$ is defined by
\[
\X
\; = \; \argmin_{\rk (\X) \leq k} \big\| \A - \C \X \big\|_F^2 .
\]

By sampling $c = 2 k \epsilon^{-1} \big( 1+o(1) \big)$ columns of $\A$ by the near-optimal algorithm of \cite{boutsidis2011NOCshort}
to form $\C \in \RB^{m\times c_1}$,
we have that
\begin{equation*}
\EB \big\| \A - \PM_{\C, k} (\A) \big\|_F^2
\; \leq \; (1+\epsilon) \big\| \A - \A_k \big\|_F^2 .
\end{equation*}
Applying Lemma 3.11 of \cite{boutsidis2014optimal},
there exists a much smaller column orthogonal matrix $\Z\in \RB^{m\times k}$
such that $\range (\Z) \subset \range (\C)$ and
\begin{equation*}
\EB \big\| \A - \C \C^\dag \A \big\|_F^2
\; \leq \;
\EB \big\| \A - \Z \Z^T \A \big\|_F^2
\; \leq \;
\big\| \A - \PM_{\C, k} (\A) \big\|_F^2.
\end{equation*}
Notice that the algorithm does not compute $\Z$.

Let $\R_1^T \in \RB^{n \times r_1}$ be columns of $\A^T$ selected by the randomized dual-set sparsification algorithm of \cite{boutsidis2011NOCshort}.
When $r_1 = \OM (k)$, it holds that
\[
\EB \big\| \A - \R_1 \R_1^T \A \big\|_F^2
\; \leq \; 2 (1+o(1)) \| \A - \A_k\|_F^2 .
\]
Let $\R_2^T \in \RB^{n\times r_2}$ be columns of $\A^T$ selected by adaptive sampling according to the residual $\A^T - \R_1^T (\R_1^T)^\dag \A^T$.
Set $r_2 = 2 k \epsilon^{-1} \big( 1+o(1) \big)$.
Let $\R^T = [\R_1^T , \R_2^T]$.
By the adaptive sampling theorem of \cite{wang2013improving}, we obtain
\begin{eqnarray} \label{eq:opt_cur_eq1}
\EB \big\| \A - \Z \Z^T \A \R^\dag \R \big\|_F^2
& \leq & \EB  \big\| \A - \Z \Z^T  \A \big\|_F^2 + \frac{k}{r_2}\EB  \big\| \A - \A \R_1^\dag \R_1^T \big\|_F^2 \nonumber \\
& \leq & (1+ \epsilon) \big\| \K - \K_k \big\|_F^2 + \epsilon \big\| \K - \K_k \big\|_F^2 \nonumber \\
& \leq & (1+ 2 \epsilon) \big\| \K - \K_k \big\|_F^2.
\end{eqnarray}
Obviously $\R^T$ contains
\[
r = r_1 + r_2= 2 k \epsilon^{-1} \big( 1+o(1) \big)
\]
columns of $\A^T$.

It remains to show $\| \A - \C \C^\dag \A \R^\dag \R \|_F^2 \leq \| \A - \Z \Z^T \A \R^\dag \R \|_F^2 $.
Since the columns of $\Z$ are contained in the column space of $\C$,
for any matrix $\Y$ the inequality $\| (\I_m - \C \C^\dag ) \Y \|_F^2 \leq (\I_m - \Z \Z^T ) \Y \|_F^2$ holds.
Then we obtain
\begin{eqnarray} \label{eq:opt_cur_eq2}
\| \A - \C \C^\dag \A \R^\dag \R \|_F^2
& = & \| \A - \A \R^\dag \R + \A \R^\dag \R  - \C \C^\dag \A \R^\dag \R \|_F^2 \nonumber \\
& = & \| \A (\I_n -  \R^\dag \R ) \|_F^2 + \|  (\I_m - \C \C^\dag ) \A \R^\dag \R \|_F^2 \nonumber \\
& \leq & \| \A (\I_n -  \R^\dag \R ) \|_F^2 + \|  (\I_m - \Z \Z^T ) \A \R^\dag \R \|_F^2 \nonumber  \\
& = & \| \A (\I_n -  \R^\dag \R ) +  (\I_m - \Z \Z^T ) \A \R^\dag \R \|_F^2  \nonumber \\
& = & \| \A - \Z \Z^T \A \R^\dag \R \|_F^2.
\end{eqnarray}
The theorem follows from \eqref{eq:opt_cur_eq1} and \eqref{eq:opt_cur_eq2} and by setting $\epsilon' = 2 \epsilon$.


\section{Proof of Theorem~\ref{thm:cur}} \label{sec:proof_cur}

In Section~\ref{sec:proof_cur_lemma} we establish a key lemma to decompose the error incurred by the approximation.
In Section~\ref{sec:proof_cur_main} we prove Theorem~\ref{thm:cur} using the key lemma.


\subsection{Key Lemma}  \label{sec:proof_cur_lemma}

We establish the following lemma for decomposing the error of the approximate solution.

\begin{lemma} \label{lem:curtype_det}
Let $\A \in \RB^{m\times n}$, $\C \in \RB^{m\times c}$, and $\R \in \RB^{r\times n}$ be any fixed matrices,
and $\A = \U_\A \Si_\A \V_\A^T$, $\C = \U_\C \Si_\C \V_\C^T$, $\R = \U_\R \Si_\R \V_\R^T$ be the SVD.
Assume that $\S_C^T \U_\C$ and $\S_R^T \V_\R$ have full column rank.
Let $\U^\star$ and $\tilde\U$ be defined in \eqref{def:cur_u_opt} and \eqref{def:cur_u_sn}, respectively.
Then the following inequalities hold:
\begin{eqnarray*}
\|\A - \C \tilde\U \R\|_F^2
& \leq & \| \A - \C \U^\star \R \|_F^2 + \Big( f_R \sqrt{ h_R } + f_C \sqrt{h_C} + f_C f_R \sqrt{ g_C' g_R}  \Big)^2 ,\\
\|\A - \C \tilde\U \R\|_F^2
& \leq & \| \A - \C \U^\star \R \|_F^2 + \Big( f_R \sqrt{ h_R } + f_C \sqrt{h_C} + f_C f_R \sqrt{ g_C g_R'}  \Big)^2 ,
\end{eqnarray*}
where $\alpha \in [0, 1]$ is arbitrary, and
\begin{eqnarray*}
f_{C} \; = \; \sigma_{\min}^{-1} (\U_\C^T \S_C \S_C^T \U_\C),
& \quad &  f_{R} \; = \; \sigma_{\min}^{-1} (\V_\R^T \S_R \S_R^T \V_\R),\\
h_{C} = \big\|\U_\C^T \S_C \S_C^T (\A - \U_\C \U_\C^T \A) \big\|_F^2,
& \quad &   h_{R} = \big\|(\A - \A \V_\R \V_\R^T) \S_C \S_C^T \V_\R \big\|_F^2 ,\\
 g_{C} = \big\| \U_\C^T \S_C  \S_C^T (\I_m - \U_\C \U_\C^T) \U_\A \Si_\A^\alpha \big\|_F^2 ,
& \quad &  g_{R} = \big\| \Si_\A^{1-\alpha} \V_\A (\I_n -\V_\R  \V_\R^T) \S_R \S_R^T \V_\R\big\|_F^2 ,\\
g_{C}' = \big\| \U_\C^T \S_C  \S_C^T (\I_m - \U_\C \U_\C^T) \U_\A \Si_\A^\alpha \big\|_2^2 ,
& \quad &  g_{R}' = \big\| \Si_\A^{1-\alpha} \V_\A (\I_n -\V_\R  \V_\R^T) \S_R \S_R^T \V_\R\big\|_2^2 .
\end{eqnarray*}
\end{lemma}

\begin{proof}
Let $k_c = \rk (\C) \leq c$ and $k_r = \rk (\R) \leq r$.
Let $\U_\C \in \RB^{m\times k_c}$ be the left singular vectors of $\C$ and $\V_\R \in \RB^{n\times k_r}$ be the right singular vectors of $\R$.
Define $\Z^\star, \tilde\Z \in \RB^{k_c\times k_r}$ by
\begin{eqnarray*}
\Z^\star
=  \U_\C^T \A \V_\R, \qquad
\tilde\Z
=  (\S_C^T \U_\C)^\dag (\S_C^T \A \S_R) (\V_\R^T \S_R)^\dag .
\end{eqnarray*}
We have that $\C \U^\star \R = \C \C^\dag \A \R^\dag \R = \U_\C \U_\C^T \A \V_\R \V_\R^T = \U_\C \Z^\star \V_\R^T$.
By definition, it holds that that
\begin{eqnarray*}
\tilde\U
& = & (\S_C^T \C)^\dag (\S_C^T \A \S_R) (\R \S_R)^\dag \\
& = & (\S_C^T \U_\C \Si_\C \V_\C^T)^\dag (\S_C^T \A \S_R) (\U_\R \Si_\R \V_\R^T \S_R)^\dag \\
& = & (\Si_\C \V_\C^T)^\dag (\S_C^T \U_\C )^\dag (\S_C^T \A \S_R) ( \V_\R^T \S_R)^\dag (\U_\R \Si_\R )^\dag\\
& = & (\Si_\C \V_\C^T)^\dag \tilde\Z (\U_\R \Si_\R )^\dag ,
\end{eqnarray*}
where the third equality follows from that $\S_C^T \U_\C$ and $\S_R^T \V_\R$ have full column rank
and that $\Si_\C \V_\C^T$ and $\V_\R^T \S_R$ have full row rank.
It follows that
\begin{eqnarray*}
\C \tilde\U \R
& = & \U_\C \Si_\C \V_\C^T (\Si_\C \V_\C^T)^\dag \tilde{\Z} (\U_\R \Si_\R )^\dag  \U_\R \Si_\R \V_\R^T
\; = \; \U_\C \tilde\Z \V_\R^T .
\end{eqnarray*}
Since $\C \U^\star \R = \U_\C \Z^\star \V_\R^T$ and $\C \tilde\U \R = \U_\C \tilde\Z \V_\R^T$,
it suffices to prove the two inequalities:
\begin{eqnarray} \label{eq:curtype:0}
\|\A -\U_\C \tilde\Z \V_\R^T \|_F^2
& \leq & \| \A - \U_\C \Z^\star \V_\R^T \|_F^2 + \Big( f_R \sqrt{ h_R } + f_C \sqrt{h_C} + f_C f_R \sqrt{ g_C g_R'}  \Big)^2 , \nonumber \\
\|\A -\U_\C \tilde\Z \V_\R^T \|_F^2
& \leq & \| \A - \U_\C \Z^\star \V_\R^T \|_F^2 + \Big( f_R \sqrt{ h_R } + f_C \sqrt{h_C} + f_C f_R \sqrt{ g_C' g_R}  \Big)^2 .
\end{eqnarray}
The left-hand side can be expressed as
\begin{align*}
& \big\|\A -\U_\C \tilde\Z \V_\R^T \big\|_F^2
\; = \; \big\| (\A - \U_\C \Z^\star \V_\R^T) +\U_\C ( \Z^\star  - \tilde\Z ) \V_\R^T \big\|_F^2 \\
& = \; \big\| (\I_m - \U_\C \U_\C^T)\A + \U_\C \U_\C^T \A (\I_n -  \V_\R \V_\R^T) +\U_\C ( \Z^\star  - \tilde\Z ) \V_\R^T \big\|_F^2 \\
& = \; \big\| (\I_m - \U_\C \U_\C^T)\A \big\|_F^2 +\big\| \U_\C \U_\C^T \A (\I_n -  \V_\R \V_\R^T) +\U_\C ( \Z^\star  - \tilde\Z ) \V_\R^T \big\|_F^2 \\
& = \; \big\| (\I_m - \U_\C \U_\C^T)\A \big\|_F^2 +\big\| \U_\C \U_\C^T \A (\I_n -  \V_\R \V_\R^T)\big\|_F^2  + \big\| \U_\C ( \Z^\star  - \tilde\Z ) \V_\R^T \big\|_F^2 \\
& = \; \big\| (\I_m - \U_\C \U_\C^T)\A + \U_\C \U_\C^T \A (\I_n -  \V_\R \V_\R^T)\big\|_F^2  + \big\| \U_\C ( \Z^\star  - \tilde\Z ) \V_\R^T \big\|_F^2 \\
& = \; \big\| \A - \U_\C \U_\C^T \A  \V_\R \V_\R^T\big\|_F^2  + \big\| \U_\C ( \Z^\star  - \tilde\Z ) \V_\R^T \big\|_F^2 .
\end{align*}
From \eqref{eq:curtype:0} we can see that it suffices to prove the two inequalities:
\begin{eqnarray} \label{eq:curtype:2}
\big\|\Z^\star  - \tilde\Z  \big\|_F
& \leq & f_R \sqrt{ h_R } + f_C \sqrt{h_C} + f_C f_R \sqrt{ g_C g_R'} , \nonumber \\
\big\|\Z^\star  - \tilde\Z  \big\|_F
& \leq & f_R \sqrt{ h_R } + f_C \sqrt{h_C} + f_C f_R \sqrt{ g_C' g_R} .
\end{eqnarray}

We left multiply both sides of $\tilde\Z =  (\S_C^T \U_\C)^\dag (\S_C^T \A \S_R) (\V_\R^T \S_R)^\dag$ by $(\S_C^T \U_\C)^T (\S_C^T \U_\C)$
and right multiply by $(\V_\R^T \S_R) (\V_\R^T \S_R)^T$.
We obtain
\begin{align*}
& (\U_\C^T \S_C \S_C^T \U_\C) \tilde\Z (\V_\R^T \S_R \S_R^T \V_\R)\\
&=\;  (\S_C^T \U_\C)^T (\S_C^T \U_\C) (\S_C^T \U_\C)^\dag (\S_C^T \A \S_R) (\V_\R^T \S_R)^\dag (\V_\R^T \S_R) (\V_\R^T \S_R)^T\\
&=\;  (\S_C^T \U_\C)^T  (\S_C^T \A \S_R) (\V_\R^T \S_R)^T \\
&=\;  \U_\C^T \S_C  \S_C^T (\A^\perp + \U_\C \Z^\star \V_\R^T) \S_R \S_R^T \V_\R .
\end{align*}
Here the second equality follows from that $\Y^T \Y \Y^\dag = \Y^T$ and $\Y^\dag \Y \Y^T=\Y^T$ for any $\Y$,
and the last equality follows by defining $\A^\perp = \A - \U_\C \Z^\star \V_\R^T$.
It follows that
\begin{align*}
& (\U_\C^T \S_C \S_C^T \U_\C) ( \tilde\Z - \Z^\star) (\V_\R^T \S_R \S_R^T \V_\R)
\;=\;  \U_\C^T \S_C  \S_C^T \A^\perp  \S_R \S_R^T \V_\R .
\end{align*}
We decompose $\A^\perp$ by
\begin{align*}
& \A^\perp
\; = \; \A - \U_\C \U_\C^T \A + \U_\C \U_\C^T \A - \U_\C \U_\C^T \A \V_\R \V_\R^T \\
& = \; \U_\C \U_\C^T \A (\I_n - \V_\R \V_\R^T) + (\I_m - \U_\C \U_\C^T) \A \V_\R \V_\R^T +  (\I_m - \U_\C \U_\C^T) \A (\I_n - \V_\R \V_\R^T) .
\end{align*}
It follows that
\begin{align*}
& (\U_\C^T \S_C \S_C^T \U_\C) ( \tilde\Z - \Z^\star) (\V_\R^T \S_R \S_R^T \V_\R)\\
& = \; \U_\C^T \S_C  \S_C^T \U_\C \U_\C^T \A (\I_n - \V_\R \V_\R^T)  \S_R \S_R^T \V_\R\\
& \quad  + \U_\C^T \S_C  \S_C^T (\I_m - \U_\C \U_\C^T) \A \V_\R  \V_\R^T \S_R \S_R^T \V_\R  \\
& \quad + \U_\C^T \S_C  \S_C^T (\I_m - \U_\C \U_\C^T) \A (\I_n -\V_\R  \V_\R^T) \S_R \S_R^T \V_\R ,
\end{align*}
and thus
\begin{align*}
&\tilde\Z - \Z^\star
 = \U_\C^T \A (\I_n - \V_\R \V_\R^T)  \S_R \S_R^T \V_\R (\V_\R^T \S_R \S_R^T \V_\R)^{-1} \\
&\;  + (\U_\C^T \S_C \S_C^T \U_\C)^{-1} \U_\C^T \S_C  \S_C^T (\I_m - \U_\C \U_\C^T) \A \V_\R \\
&\;  + (\U_\C^T \S_C \S_C^T \U_\C)^{-1} \U_\C^T \S_C  \S_C^T (\I_m - \U_\C \U_\C^T) \A (\I -\V_\R  \V_\R^T) \S_R \S_R^T \V_\R (\V_\R^T \S_R \S_R^T \V_\R)^{-1}.
\end{align*}
It follows that
\begin{small}
\begin{align*}
&\| \tilde\Z - \Z^\star \|_F
\; \leq \; \sigma_{\min}^{-1} (\V_\R^T \S_R \S_R^T \V_\R) \big\|\A (\I_n - \V_\R \V_\R^T)  \S_R \S_R^T \V_\R \big\|_F \\
&  +  \sigma_{\min}^{-1} (\U_\C^T \S_C \S_C^T \U_\C) \big\| \U_\C^T \S_C  \S_C^T (\I_m - \U_\C \U_\C^T) \A \V_\R \big\|_F \\
& + \sigma_{\min}^{-1} (\U_\C^T \S_C \S_C^T \U_\C) \sigma_{\min}^{-1} (\V_\R^T \S_R \S_R^T \V_\R)
        \big\| \U_\C^T \S_C  \S_C^T (\I_m - \U_\C \U_\C^T) \A (\I -\V_\R  \V_\R^T) \S_R \S_R^T \V_\R\big\|_F .
\end{align*}
\end{small}%
This proves \eqref{eq:curtype:2} and thereby concludes the proof.
\end{proof}


\subsection{Proof of the Theorem}  \label{sec:proof_cur_main}

Assumption~\ref{assumption:faster_cur1} assumes that the sketching matrices $\S_C$ and $\S_R$ satisfy the first two approximate matrix multiplication properties.
Under the assumption, we obtain Lemma~\ref{lem:faster_cur_det},
which shows that $\tilde\U$ is nearly as good as $\U^\star$ in terms of objective function value.

\begin{assumption} \label{assumption:faster_cur1}
Let $\B$ be any fixed matrix.
Let $\C \in \RB^{m\times c}$ and $\C = \U_\C \Si_\C \V_\C^T$ be the SVD.
Assume that a certain sketching matrix $\S_C \in \RB^{m\times s_c}$ satisfies
\begin{eqnarray*}
\PB \Big\{ \big\| \U_\C \S_C \S_C^T \U_\C - \I  \big\|_2
\; \geq \; \frac{1}{10} \Big\}
& \leq & \delta_{1} \\
\PB \Big\{ \big\| \U_\C^T \S_C \S_C^T \B - \U_\C^T \B \big\|_F^2
\; \geq \; \epsilon \|\B\|_F^2 \Big\}
& \leq & \delta_{2}
\end{eqnarray*}
for any $\delta_{1}, \delta_{2} \in (0, 0.2)$.
Let $\R \in \RB^{r\times n}$ and $\R = \U_\R \Si_\R \V_\R^T$ be the SVD.
Similarly, assume $\S_R \in \RB^{n\times s_r}$ satisfies
\begin{eqnarray*}
\PB \Big\{ \big\| \V_\R^T \S_R \S_R^T \V_\R - \I  \big\|_2
\; \geq \; \frac{1}{10} \Big\}
& \leq & \delta_{1} \\
\PB \Big\{ \big\| \V_\R^T \S_R \S_R^T \B - \V_\R^T \B \big\|_F^2
\; \geq \; \epsilon \|\B \|_F^2 \Big\}
& \leq & \delta_{2} .
\end{eqnarray*}
\end{assumption}

\begin{lemma} \label{lem:faster_cur_det}
Let $\A \in \RB^{m\times n}$, $\C \in \RB^{m\times c}$, and $\R \in \RB^{r\times n}$ be any fixed matrices.
Let $\U^\star$ and $\tilde\U$ be defined in \eqref{def:cur_u_opt} and \eqref{def:cur_u_sn}, respectively.
Let $k_c = \rk (\C)$, $k_r = \rk (\R)$, $q= \min\{m,n\}$, and $\epsilon \in (0,1)$ be the error parameter.
Assume that the sketching matrices $\S_C $ and $\S_R$ satisfy Assumption~\ref{assumption:faster_cur1} and that $\epsilon^{-1} = o (q)$.
Then
\begin{eqnarray*}
\|\A - \C \tilde\U \R\|_F^2
& \leq & (1+ 4 \epsilon^2 q )\: \| \A - \C \U^\star \R \|_F^2
\end{eqnarray*}
holds with probability at least $1-2\delta_{1}  - 3 \delta_{2} $.
\end{lemma}

\vspace{1mm}

\begin{proof}
Let $f_C$, $f_R$, $h_C$, $h_R$, $g_C$, $g_R$, $g_C'$, $g_R'$ be defined Lemma~\ref{lem:curtype_det}.
Under Assumption~\ref{assumption:faster_cur1}, we have that
\begin{align*}
f_C  \leq  \frac{10}{9} ,
&\qquad h_C \leq \epsilon \|\A - \U_\C \U_\C^T \A \|_F^2 \leq \epsilon \|\A - \C \U^\star \R^T \|_F^2 ,\\
f_R \leq \frac{10}{9} ,
&\qquad h_R \leq \epsilon \|\A - \A \V_\R \V_\R^T \|_F^2 \leq \epsilon \|\A - \C \U^\star \R^T \|_F^2 ,
\end{align*}
hold simultaneously with probability at least $1- 2 \delta_{ 1}  - 2 \delta_{ 2}$.

We fix $\alpha = 1$, then $g_C = h_C$, and $g_{R}' \leq \big\|(\I_n -\V_\R  \V_\R^T) \S_R \S_R^T \V_\R\big\|_2^2$.
Under Assumption~\ref{assumption:faster_cur1}, we have that
\begin{eqnarray*}
\sqrt{g_R'}
&  \leq & \big\| (\I_n -\V_\R  \V_\R^T) \S_R \S_R^T \V_\R - (\I_n -\V_\R  \V_\R^T) \V_\R \big\|_F  \\
& \leq & \sqrt{\epsilon} \big\| (\I_n -\V_\R  \V_\R^T)  \big\|_F
\; \leq \; \sqrt{\epsilon n }
\end{eqnarray*}
holds with probability at least $1- \delta_{ 2}$.
It follows from Lemma~\ref{lem:curtype_det} that
\begin{align*}
& \|\A - \C \tilde\U \R\|_F^2 - \| \A - \C \U^\star \R \|_F^2 \\
& \leq \;  \Big( f_R \sqrt{ h_R } + f_C \sqrt{h_C} + f_C f_R \sqrt{ g_C g_R'}  \Big)^2 \\
& \leq \;  \Big( \frac{20}{9} \sqrt{\epsilon} \|\A - \C \U^\star \R^T \|_F + \frac{10^2}{9^2} \epsilon \sqrt{n} \|\A - \C \U^\star \R^T \|_F \Big)^2 \\
& = \;\frac{10^4}{9^4} \epsilon^2 n \big( 1+ o(1) \big)  \|\A - \C \U^\star \R^T \|_F^2
\; \leq \; 4 \epsilon^2 n  \|\A - \C \U^\star \R^T \|_F^2
\end{align*}
holds with probability at least $1-2\delta_{ 1} -  3\delta_{ 2}$.
Here the equality follows from that $\epsilon^{-1} = o(n)$.

Alternatively, if we fix $\alpha = 0$, we will obtain that
\begin{align*}
\|\A - \C \tilde\U \R\|_F^2
\; \leq \;  \| \A - \C \U^\star \R \|_F^2 + 4 \epsilon^2 m  \|\A - \C \U^\star \R^T \|_F^2
\end{align*}
with probability $1-2\delta_{ 1} -  3\delta_{ 2}$.
Therefore, if $n\leq m$, we fix $\alpha = 1$; otherwise we fix $\alpha = 0$. This concludes the proof.
\end{proof}

In the following we further assume that the sketching matrices $\S_C$ and $\S_R$ satisfy the third approximate matrix multiplication property.
Under Assumption~\ref{assumption:faster_cur1} and Assumption~\ref{assumption:faster_cur2},
we obtain Lemma~\ref{lem:faster_cur_det2} which is stronger than Lemma~\ref{lem:faster_cur_det}.

\begin{assumption} \label{assumption:faster_cur2}
Let $\B$ be any fixed matrix.
Let $\C \in \RB^{m\times c}$, $k_c = \rk (\C)$, and $\C = \U_\C \Si_\C \V_\C^T$ be the SVD.
Assume that a certain sketching matrix $\S_C \in \RB^{n\times s_c}$ satisfies
\begin{align*}
\PB \Big\{ \big\| \U_\C^T \S_C \S_C^T \B - \U_\C^T \B \big\|_2^2
\; \geq \; \epsilon  \|\B\|_2^2 + \frac{\epsilon}{k_c} \|\B\|_F^2 \Big\}
\; \leq \; \delta_{3}
\end{align*}
for any $\epsilon \in (0, 1)$ and $ \delta_{3} \in (0, 0.2)$.
Let $\R \in \RB^{r\times n}$, $k_r = \rk (\R)$, and $\R = \U_\R \Si_\R \V_\R^T$ be the SVD.
Similarly, assume that $\S_R \in \RB^{n\times s_r}$ satisfies
\begin{align*}
\PB \Big\{ \big\| \V_\R^T \S_R \S_R^T \B - \V_\R^T \B \big\|_2^2
\; \geq \; \epsilon  \|\B\|_2^2 + \frac{\epsilon}{k_r} \|\B\|_F^2 \Big\}
\; \leq \; \delta_{3} .
\end{align*}
\end{assumption}

\begin{lemma} \label{lem:faster_cur_det2}
Let $\A \in \RB^{m\times n}$, $\C $, $\R $, $\U^\star$, $\tilde\U$, $k_c$, $k_r$ be defined in Lemma~\ref{lem:faster_cur_det}.
Let $q= \min\{m,n\}$ and $\tilde{q} = \min\{m/k_c, n/k_r \}$.
Assume that the sketching matrices $\S_C $ and $\S_R$ satisfy Assumption~\ref{assumption:faster_cur1} and Assumption~\ref{assumption:faster_cur2} and
that $\epsilon^{-1} = o \big( \tilde{q} \big)$.
Then
\begin{eqnarray*}
\|\A - \C \tilde\X \R\|_F^2
& \leq & (1+ 4 \epsilon^2 \tilde{q} )\: \| \A - \C \X^\star \R \|_F^2
\end{eqnarray*}
holds with probability at least $1-2\delta_1 - 2 \delta_{2} - \delta_{3}$.
\end{lemma}

\begin{proof}
Let $f_C$, $f_R$, $h_C$, $h_R$, $g_C$, $g_R$, $g_C'$, $g_R'$ be defined Lemma~\ref{lem:curtype_det}.
Under Assumption~\ref{assumption:faster_cur1}, we have shown in the proof of Lemma~\ref{lem:faster_cur_det} that
\begin{align*}
f_C  \leq  \frac{10}{9} ,
&\qquad h_C \leq \epsilon \|\A - \C \U^\star \R^T \|_F^2 ,\\
f_R \leq \frac{10}{9} ,
&\qquad h_R \leq \epsilon \|\A - \C \U^\star \R^T \|_F^2 ,
\end{align*}
hold simultaneously with probability at least $1-2\delta_{1} - 2\delta_{2}$.

We fix $\alpha = 1$, then $g_C = h_C$, and $g_{R}' \leq \big\|(\I_n -\V_\R  \V_\R^T) \S_R \S_R^T \V_\R\big\|_2^2$.
Under Assumption~\ref{assumption:faster_cur2}, we have that
\begin{eqnarray*}
{g_R'}
& \leq & \big\| (\I_n -\V_\R  \V_\R^T) \S_R \S_R^T \V_\R - \underbrace{(\I_n -\V_\R  \V_\R^T) \V_\R}_{=\0} \big\|_2^2  \\
& \leq & \epsilon \big\| \I_n -\V_\R  \V_\R^T \big\|_2^2 + \frac{\epsilon}{k_r} \big\|\I_n -\V_\R  \V_\R^T  \big\|_F^2
\; \leq \; \epsilon + \frac{\epsilon (n-k_r)}{k_r}
\; = \; \frac{\epsilon n}{k_r}
\end{eqnarray*}
holds with probability at least $1-\delta_3$.
It follows from Lemma~\ref{lem:curtype_det} that
\begin{align*}
& \|\A - \C \tilde\U \R\|_F^2 - \| \A - \C \U^\star \R \|_F^2 \\
& \leq \;  \Big( f_R \sqrt{ h_R } + f_C \sqrt{h_C} + f_C f_R \sqrt{ g_C g_R'}  \Big)^2 \\
& \leq \;  \Big( \frac{20}{9} \sqrt{\epsilon} \|\A - \C \U^\star \R^T \|_F + \frac{10^2}{9^2} \epsilon \sqrt{n / k_r} \|\A - \C \U^\star \R^T \|_F \Big)^2 \\
& = \;\frac{10^4}{9^4} \epsilon^2 n k_r^{-1} \big( 1+ o(1) \big)  \|\A - \C \U^\star \R^T \|_F^2
\; \leq \; 4 \epsilon^2 n k_r^{-1} \|\A - \C \U^\star \R^T \|_F^2
\end{align*}
holds with probability at least $1- 2 \delta_{1} - 2 \delta_{2} - \delta_{3}$.
Here the equality follows from that $\epsilon^{-1} = o(n / k_r)$.

Analogously, by fixing $\alpha = 0$ and assuming $\epsilon^{-1} = o(m / k_c)$, we can show that
\[
\|\A - \C \tilde\U \R\|_F^2 - \| \A - \C \U^\star \R \|_F^2
\; \leq \; 4 \epsilon^2 m k_c^{-1} \|\A - \C \U^\star \R^T \|_F^2
\]
holds with probability at least $1- 2 \delta_{ 1} - 2 \delta_{ 2} - \delta_{ 3}$.
This concludes the proof.
\end{proof}

Finally, we prove Theorem~\ref{thm:cur} using Lemma~\ref{lem:faster_cur_det} and Lemma~\ref{lem:faster_cur_det2}.

For leverage score sampling, uniform sampling, and count sketch,
Assumption~\ref{assumption:faster_cur1} is satisfied.
Then the bound follows by setting $\epsilon = 0.5 \sqrt{\epsilon' / q}$ and
applying Lemma~\ref{lem:faster_cur_det}.
Here $q= \min\{m,n\}$.
For the three sketching methods, we set $\delta_{1} = 0.01$ and $\delta_{2} = 0.093$.

For Gaussian projection and SRHT,
Assumption~\ref{assumption:faster_cur1} and Assumption~\ref{assumption:faster_cur2} are satisfied.
Then the bound follows by setting $\epsilon = 0.5 \sqrt{ \epsilon' / \tilde{q}}$ and applying Lemma~\ref{lem:faster_cur_det2}.
Here $\tilde{q} = \min\{m/k_c, n/k_r \}$.
For Gaussian projection, we set $\delta_{1}  = 0.01$, $\delta_{2} = 0.09$, and $\delta_{3}  = 0.1$.
For SRHT, we set $\delta_{1} = 0.02$, $\delta_{2}   = 0.08$, and $\delta_{3}  = 0.1$.


\bibliography{Efficient}


\end{document}